\documentclass[moor]{informs1}              

\usepackage{color}
\usepackage{algorithm,algorithmic}
\usepackage{subfigure}
\usepackage{amsfonts}
\usepackage{caption}
\usepackage{scrextend}

\usepackage{natbib}
 \NatBibNumeric
 \def\bibfont{\small}%
 \bibpunct[, ]{[}{]}{,}{n}{}{,}%
\definecolor{applegreen}{rgb}{0.55, 0.71, 0.0}
\definecolor{darkpastelgreen}{rgb}{0.01, 0.75, 0.24}
\usepackage[colorlinks=true,breaklinks=true,bookmarks=true,urlcolor=blue,
     citecolor=blue,linkcolor=blue,bookmarksopen=false,draft=false]{hyperref}

\def\EMAIL#1{\href{mailto:#1}{#1}}

\TheoremsNumberedThrough     

\EquationsNumberedThrough    

\begin{document}


\RUNAUTHOR{Guo, Hu, Xu and Zhang}

\RUNTITLE{A General Framework for Learning Mean-Field Games}

\TITLE{A General Framework for Learning Mean-Field Games}

\ARTICLEAUTHORS{%
\AUTHOR{Xin Guo*}
\AFF{University of California, Berkeley, IEOR, \EMAIL{xinguo@berkeley.edu}\\
Amazon.com, \EMAIL{xnguo@amazon.com}}
\AUTHOR{Anran Hu}
\AFF{University of California, Berkeley, IEOR, \EMAIL{anran\_hu@berkeley.edu}}
\AUTHOR{Renyuan Xu}
\AFF{University of Southern California, Industrial Systems and Engineering, \EMAIL{renyuanx@usc.edu}\\
University of Oxford, Mathematical Institute, \EMAIL{xur@maths.ox.ac.uk}}
\AUTHOR{Junzi Zhang\footnote{Work done prior to joining or outside of Amazon.}}
\AFF{Amazon.com, \EMAIL{junziz@amazon.com}}
} 

\ABSTRACT{%
This paper presents a general mean-field game (GMFG) framework for simultaneous learning and decision-making in stochastic games with a large population. It first establishes the existence of a unique Nash Equilibrium  to this GMFG, and demonstrates that naively combining 
{\color{black}reinforcement learning} with the fixed-point approach in classical MFGs yields unstable algorithms. It then proposes value-based and policy-based {\color{black}reinforcement learning} algorithms (GMF-V and GMF-P, respectively) with {\color{black}smoothed} policies, with analysis of their convergence properties  and computational  complexities. 
Experiments  on an equilibrium product pricing problem  demonstrate that {\color{black}GMF-V-Q and GMF-P-TRPO, two specific instantiations of GMF-V and GMF-P, respectively, with Q-learning and TRPO,} are both efficient and robust {\color{black}in the GMFG setting}. Moreover, their performance is superior in {\color{black}convergence speed, accuracy, and stability when compared with existing algorithms for multi-agent reinforcement learning in the $N$-player setting}.
}%



\maketitle

%



\section{Introduction}
\label{introduction}

\paragraph{Motivating example.}

{\color{black}
This paper is motivated by the following Ad auction problem for an advertiser.
An Ad auction is a stochastic game on an Ad exchange platform among a large number of players, the
advertisers. In between the time a web user requests a page and the time the page is displayed, usually
within a millisecond, a Vickrey-type of second-best-price auction is run to incentivize interested
advertisers to bid for an Ad slot to display advertisement. Each advertiser has limited information
before each bid: first, her own valuation for a slot depends on some random conversion of clicks for
the item; secondly, she, should she win the bid, only knows the reward after the user’s activities on
the website are finished. In addition, she has a budget constraint in this repeated auction.

The question is, how should she bid in this online sequential repeated game when there is a large
population of bidders competing on the Ad platform, with random conversions of
clicks and rewards?
}


Besides Ad auctions, there are many other real-world problems involving a large number of players and {\color{black}uncertain} systems.  Examples include massive multi-player online role-playing games \citet{MMORPG}, high frequency tradings \citet{algo_trade_mfg}, and the sharing economy \citet{sharing_eco}.

\paragraph{Our work.}

Motivated by these problems, we consider a general framework  of  simultaneous learning and decision-making in stochastic games with a large population. We formulate a general mean-field-game (GMFG) with incorporation of action distributions and (randomized) relaxed policies.  This {\color{black}general} framework can also be viewed as a generalized version of MFGs of {\color{black}extended} McKean-Vlasov type \citet{MKV}, which is a different paradigm from the classical MFG. It is also beyond the scope of the existing {\color{black}reinforcement learning} {\color{black}(RL)} framework for Markov decision processes (MDP), as MDP is technically equivalent to a single player stochastic game.

On the theory front, this {\color{black}general} framework differs from the existing MFGs. We establish under appropriate technical conditions  the existence and uniqueness of the 
Nash equilibrium (NE) to this GMFG.  On the computational front, we show  that naively combining 
{\color{black}reinforcement learning} with the three-step fixed-point approach in classical MFGs yields unstable algorithms. We then propose {\color{black} both value based and policy based reinforcement learning algorithms with smoothed policies} {\color{black}(GMF-V and GMF-P, respectively)}, establish the convergence property and analyze the computational complexity (see Section \ref{sec:proof} for all proof details).
Finally, we apply {\color{black}GMF-V-Q and GMF-P-TRPO, which are two specific instantiations of GMF-V and GMF-P, respectively, with Q-learning and TRPO,} to an {\color{black}equilibrium product pricing} problem\footnote{\color{black}The numerical experiments on the application of GMF-V-Q to the motivating Ad auction problem can be found in the conference version of our paper \citet{guo2019learning}.}.
{\color{black}Both algorithms} have demonstrated to be
 efficient and robust {\color{black}in the GMFG setting}. 
 Their performance  is superior in terms of {\color{black}convergence speed, accuracy and stability}, when compared with existing algorithms for multi-agent reinforcement learning {\color{black}in the $N$-player setting}. 
{\color{black} Note that an earlier and preliminary version \citet{guo2019learning} has been published in NeurIPS. Nevertheless, the conference version focuses only on GMF-V-Q, whereas this paper provides a new meta framework for learning mean-field-game which combines (1) the three-step fixed point approach, (2) the smoothing techniques, and (3) the single-agent algorithms with sample complexity guarantees in the sub-routine. This general framework incorporates both value-based algorithms and policy-based algorithms. In addition, the policy-based RL algorithm (GMF-P-TRPO) in this paper is the first globally convergent policy-based algorithm for solving mean-field-games. Numerical results show that it achieves similar performance as the Q-learning based algorithm (GMF-V-Q) in \citet{guo2019learning}. }

\paragraph{Related works.} On learning large population games with mean-field approximations, \citet{YYTXZ2017} focuses on inverse reinforcement learning  for MFGs without decision making, {\color{black}with its extension in \citet{chen2021maximum} for agent-level inference;}  \citet{YLLZZW2018} studies an MARL problem with a {\color{black}first-order} mean-field approximation term modeling the interaction between one player and all the other finite players, {\color{black}which has been generalized to the setting with partially observable states in \citet{subramanian2020partially};} and \citet{KC2013} and \citet{YMMS2014} consider model-based adaptive learning for MFGs {\color{black}in specific models (\textit{e.g.}, linear-quadratic and oscillator games)}. {\color{black}More recently, \citet{Manymany} studies the local convergence of actor-critic algorithms on finite time horizon MFGs, and \citet{rl_mfg_local} proposes a policy-gradient based algorithm and analyzes the so-called local NE for reinforcement learning in infinite time horizon MFGs.}  For learning large population games without mean-field approximation, see \citet{MARL_literature2, MARL_literature1} and the references therein.  
In the specific topic of learning auctions with a large number of advertisers, \citet{CRZMWYG2017} and \citet{JSLGWZ2018} explore reinforcement learning techniques to search for social optimal solutions with real-word data, and \citet{IJS2011} uses MFGs to model the auction system with unknown conversion of clicks within a Bayesian framework.


{\color{black}However, none of these works consider the problem of simultaneous learning and decision-making in a general MFG framework. Neither do they establish the existence and uniqueness of  the {\color{black}(global)} NE, nor do they present  model-free learning algorithms with complexity analysis and  convergence to the NE.} Note that in principle, global results are harder to obtain compared to local results.

{\color{black}
Following the conference version \citet{guo2019learning} of the current paper, various efforts have been made to extend our reinforcement learning work {\color{black} in \citet{guo2019learning} to more general MFG settings}. 
{\color{black}These include linear-quadratic MFGs in both discrete-time setting \citet{fu2019actor,uz2020approximate,uz2020reinforcement} and in continuous-time setting \citet{guo2020entropy,wang2021global,delarue2021exploration}, MFGs with general continuous state and/or action spaces \citet{fitted-Q}, entropy regularized MFGs in discrete time \citet{anahtarci2020q,xie2020provable,xie2021learning,cui2021approximately} and in continuous time \citet{guo2020entropy}, and non-stationary MFGs \citet{mishra2020model}. In particular, \citet{cui2021approximately} interprets the softmax smoothing technique proposed in \citet{guo2019learning} from a smoothed equilibrium perspective. 
In addition, different frameworks based on monotonicity assumptions (instead of the contractivity assumption in \citet{guo2019learning}) have also been proposed, and fictitious play algorithms with policy and mean-field averaging \citet{elie2020convergence,perrin2020fictitious} and online mirror descent algorithms \citet{perolat2021scaling} have been proposed to solve MFGs under such assumptions. There are also some recent extensions to reinforcement learning of MFGs with strategic complementarity \citet{lee2021reinforcement} and multiple agent types \citet{ghosh2020model,subramanian2020multi}. These algorithms for reinforcement learning of MFGs have also been applied in economics \citet{angiuli2021reinforcement}, in finance \citet{de2021dealer}, in animal behavior simulation \citet{perrin2021mean}, and in concave utility reinforcement learning \citet{geist2021concave}.} 
In the meantime, the idea of simultaneous learning and decision making with mean-field interaction has been used for analyzing collaborative games with social optimal solution \citet{carmona2019linear, carmona2019model,gu2020mean, LYWK2019,wang2020breaking, pasztor2021efficient, gagrani2020thompson,cui2021discrete,angiuli2020unified}.
}

\paragraph{Notations.}
  {\color{black}  Let $(\mathcal{X},d_\mathcal{X})$ be a metric space and  $\mathcal{X}$ is equipped with the Borel $\sigma$-field $\mathcal{B}(\mathcal{X})$, meaning the $\sigma$-field generated by the open
sets of $\mathcal{X}$. Denote $\mathcal{P}(\mathcal{X})$ for the set of (Borel) probability measures
on $\mathcal{X}$. 
$\mathcal{W}_p$ denotes the Wasserstein distance of order $p$ such that
\begin{eqnarray*}
W_p(\mu, \mu')= \inf\left\{\biggl(\int_{\mathcal{X} \times \mathcal{X}}d^p_{\mathcal{X}}(x, x')\nu_{}(dx, dx')\biggl):  \nu_{} \in \mathcal{P}(\mathcal{X} \times \mathcal{X}) \mbox{ with marginals } \mu, \mu' \in \mathcal{P}(\mathcal{X})\right\}.
\end{eqnarray*}
$\mathcal{P}(\mathcal{X})$ is always equipped with $W_1(\mu, \mu')$. 
The Borel $\sigma$-field of $\mathcal{P}(\mathcal{X})$ is the $\sigma$-field induced by the evaluation $\mathcal{P}(\mathcal{X}) \ni \mu \mapsto \mu(C)$ for any {Borel set} $C \subset \mathcal{X}$. Note that the Borel $\sigma$-field of $\mathcal{P}(\mathcal{X})$ is generated by $W_1$. (See e.g. \citet{Villani2009} and \citet{Lacker2015}).

Given two measurable spaces $(\mathcal{Y}, \mathcal{B}(\mathcal{Y}))$ and $(\mathcal{X}, \mathcal{B}(\mathcal{X}))$ , we say a measure-valued function $f: \mathcal{Y} \to  \mathcal{P}(\mathcal{X})$ is measurable if  $\Lambda_C \circ f: \mathcal{Y} \to [0, 1]$ is measurable for any $C \in \mathcal{B}(\mathcal{X})$, where $\Lambda_C: \mathcal{P}(\mathcal{X}) \ni \mu \mapsto \mu(C) \in [0, 1]$.}

\section{Framework of General MFG (GMFG)}\label{n2mfg}

\subsection{Background: classical $N$-player Markovian game and MFG}\label{classical}
Let us first recall the classical $N$-player game. There are $N$ players in a game.   
 At each step $t$, the state of  player $i   \ \ (=1, 2, \cdots, N)$ is  $s^i_t\in\mathcal{S}\subseteq\mathbb{R}^d$ and she takes an action $a^i_t\in\mathcal{A}\subseteq\mathbb{R}^p$. Here $d, p$ are positive integers.  {\color{black} The state space $(\mathcal{S}, d_{\mathcal{S}})$ and the action space $(\mathcal{A}, d_{\mathcal{A}})$  are two compact  metric spaces, including the case of  $\mathcal{S}$ and $\mathcal{A}$  being finite.}
 Given the current state profile of $N$-players ${\bf s}_t=(s^1_t,\dots,s^N_t)\in\mathcal{S}^N$ and the action $a^i_t$,   player $i$ will receive a reward $r^i({\bf s}_t, a^i_t)$ {\color{black}sampled from a distribution $ R^i({\bf s}_t, a^i_t)$} and her state will change to $s^i_{t+1}$ according to a transition probability function $P^i({\bf s}_t, a^i_t)$. {\color{black}In particular, the probability transition $P^i:$ $\mathcal{S}^N \times  \mathcal{A}\rightarrow\mathcal{P}(\mathcal{S})$ and  the distribution of the reward function $R^i:$ $\mathcal{S}^N \times  \mathcal{A}$ $\to$ $\mathcal{P}([0,{\rm R}_{\max}])$ are both measurable functions with some constant ${\rm R}_{\max}>0$.}

 A Markovian game further restricts the admissible policy/control for player $i$ to be of the form
 $a^i_t\sim\pi^i_t({\bf s}_t)$ {\color{black}with $\pi^i_t$ measurable}. That is,  $\pi^i_t:\mathcal{S}^N\rightarrow \mathcal{P}(\mathcal{A})$  maps each state profile ${\bf s}\in\mathcal{S}^N$ to a randomized action.
The accumulated reward (a.k.a. the value function) for player $i$, given the initial state profile ${\bf s}$ and the policy profile sequence $\pmb{\pi} :=\{\pmb{\pi}_t\}_{t=0}^{\infty} $ with $\pmb{\pi}_t=(\pi^1_t,\dots,\pi^N_t)$, is then defined as
\begin{eqnarray}\label{game}
V^i({\bf s},\pmb{\pi}):=\mathbb{E}\left[\sum_{t=0}^{\infty}\gamma^t r^i({\bf s}_t,a^i_t)\Big| {\bf s}_0={\bf s}\right],
\end{eqnarray}
where $\gamma\in(0,1)$ is the discount factor,  $a^i_t\sim \pi^i_t({\bf s}^t)$, and $s^i_{t+1}\sim P^i({\bf s}_t, a_t^i)$.
The goal of each player is to maximize her value function over all admissible policy sequences {\color{black} such that \eqref{game} is finite}.   

In general, this type of stochastic $N$-player game is notoriously hard to analyze, especially when $N$  is large {\color{black}\citet{PR05}}. Mean field game (MFG),  pioneered by \citet{HMC2006} and \citet{LL2007} {\color{black}in the continuous settings and later developed  in \citet{MFG_n_conv, MFG_gomes, MFG_binact, MFG_discrete_time, MFG_discrete_time2} for discrete settings}, provides an ingenious and  tractable aggregation approach to approximate the otherwise challenging $N$-player stochastic games. 
The basic idea for an MFG goes as follows. Assume all players are  identical, indistinguishable and interchangeable, when $N\to \infty$, one can view the limit of other players' states ${\bf s}_t^{-i}=(s_t^1,\dots,s_t^{i-1},s_t^{i+1},\dots,s_t^N)$ as a population state distribution {\color{black}$\mu_t$ with $\mu_t(s):=\lim_{N \rightarrow \infty}\frac{\sum_{j=1, j\neq i}^N \textbf{I}_{s_t^j=s}}{N}$}.\footnote{{\color{black}Here the indicator function $\textbf{I}_{s_t^j=s}=1$ if $s_t^j=s$ and $0$ otherwise.}} Due to the homogeneity of the players, one can then focus on a single (representative) player.
{\color{black}At time $t$, after the representative player chooses her action $a_t$ according to some policy $\pi_t$, she will receive reward $r(s_t,a_t,\mu_t)$ and her state will evolve under a \textit{controlled stochastic dynamics} of a mean-field type $P(\cdot|s_t,a_t,\mu_t)$.
Here the policy $\pi_t$ depends on both the current state $s_t$ and the current population state distribution $\mu_t$ such that $\pi_t:\mathcal{S}\times \mathcal{P}(\mathcal{S})\rightarrow \mathcal{P}(\mathcal{A})$.
} Then, in mean-field limit, one may consider instead   the following optimization problem,
\[
\begin{array}{ll}
\text{maximize}_{\pmb{\pi}} & V(s,\pmb{\pi},\pmb{\mu}):=\mathbb{E}\left[\sum\limits_{t=0}^\infty \gamma^t r(s_t,a_t,\mu_t)|s_0=s\right]\\
\text{subject to} & s_{t+1}\sim P(s_t,a_t,\mu_t), \quad a_t\sim \pi_t(s_t,\mu_t),
\end{array}
\]
where $\pmb{\pi}:={\{\pi_t\}_{t=0}^{\infty}}$ denotes the policy sequence and $\pmb{\mu} := \{\mu_t\}_{t=0}^\infty$  the distribution flow.


\subsection{General MFG (GMFG)}\label{mfg-set-up}
In the classical MFG setting, {\color{black}the reward and the dynamic for each player are known}. They depend only on  the state of the player $s_t$,   the action of this particular player $a_t$, and  the population state distribution $\mu_t$. In contrast, in the motivating auction example, the reward and the dynamic  are {\color{black}unknown}; they  rely on the actions of {\it all} players,  as well as on $s_t$ and $\mu_t$.



We therefore define the following  general MFG (GMFG) framework.  At time $t$, after the representative player chooses her action $a_t$ according to some {\color{black}measurable} policy  $\pi:\mathcal{S}\times \mathcal{P}(\mathcal{S})\rightarrow \mathcal{P}(\mathcal{A})$, she will receive a {\color{black}(possibly random)} reward $r(s_t,a_t,\mathcal{L}_t)$ {\color{black}sampled from distribution $R(s_t,a_t,\mathcal{L}_t)$} and her state will evolve according to $P(\cdot|s_t,a_t,\mathcal{L}_t)$,  {\color{black} with $\mathcal{L}_t=\mathbb{P}_{s_t,a_t}\i\color{black}n \mathcal{P}(\mathcal{S}\times \mathcal{A})$  the joint distribution  of the state and the action, \textit{i.e.}, the \text{population state-action pair}. This joint distribution $\mathcal {L}_t$  has  marginal distributions $\alpha_t$ for the population action  and $\mu_t$ for the population state. Note the inclusion of $\alpha_t$ allows the reward and the dynamic to depend on all players' actions.} {\color{black} Here $P:\mathcal{S}\times \mathcal{A}\times \mathcal{P}(\mathcal{S}\times \mathcal{A}) \rightarrow \mathcal{P} (\mathcal{S})$ and $R: \mathcal{S}\times \mathcal{A}\times \mathcal{P}(\mathcal{S}\times \mathcal{A}) \rightarrow \mathcal{P} ([0,{\rm R_{max}}])$ are measurable functions with some constant ${\rm R_{max}}>0$.}
The objective of the player is to solve the following control problem:
\begin{equation}\label{mfg}
\begin{array}{ll}
\text{maximize}_{\pmb{\pi}} & V(s,\pmb{\pi},\pmb{\mathcal{L}}):=\mathbb{E}\left[\sum\limits_{t=0}^\infty \gamma^t r(s_t,a_t,{\color{black}\mathcal{L}_t})|s_0=s\right]\\
\text{subject to} & s_{t+1}\sim P(s_t,a_t,{\color{black}\mathcal{L}_t}),\quad a_t\sim \pi_t(s_t,\mu_t).
\end{array}\tag{GMFG}
\end{equation}
{\color{black}Here the  expectation in the objective function is always taken for all randomness in the system. In addition}, $\pmb{\mathcal{L}}:=\{\mathcal{L}_t\}_{t=0}^{\infty}$
and may be time dependent. That is, an infinite-time horizon MFG may have time-dependent NE solutions due to the mean information process in the MFG. This is fundamentally different from the theory of MDP where the optimal control, if exists uniquely, would be time independent in an infinite time horizon setting.

In this paper, we will analyze the existence of NE to GMFG. For ease of exposition, we will first focus on stationry NEs. Accordingly, for notational brevity, we abbreviate $\pmb{\pi}=\{\pi\}_{t=0}^{\infty}$ and $\pmb{\mathcal{L}}=\{\mathcal{L}\}_{t=0}^{\infty}$ as $\pi$ and $\mathcal{L}$, respectively. We will show in the end how this stationary constraint can be relaxed (\textit{cf}. Section \ref{stat_mfg_app}).
\begin{definition}[Stationary NE for GMFGs]\label{nash2_stat} 
In \eqref{mfg}, a player-population profile ($\pi^\star$, $\mathcal{L}^\star$)
  is called a stationary NE if 
\begin{enumerate}
    \item (Single player side) For any policy $\pi$ and any initial state $s\in \mathcal{S}$, 
\begin{equation}
V\left(s,\pi^\star,{\color{black}\mathcal{L}^\star}\right)\geq V\left(s,\pi,{\color{black}\mathcal{L}^\star}\right).
\end{equation}
\item  (Population side) $\mathbb{P}_{s_t,a_t}= {\mathcal{L}^{\star}}$ for all $t\geq 0$, where $\{s_t,a_t\}_{t=0}^{\infty}$ is the dynamics under the policy  $\pi^\star$ starting from $s_0 \sim \mu^{\star}$, with $a_t\sim\pi^\star(s_t,{\color{black}\mu^{\star}})$, $s_{t+1}\sim P(\cdot|s_t,a_t,{\color{black}\mathcal{L}^\star})$, and $\mu^{\star}$ being the population state marginal of $\mathcal{L}^\star$.
\end{enumerate}
\end{definition}

The single player side condition captures the optimality of ${\pi}^\star$, when the population side is fixed. The population side condition ensures the ``consistency'' of the solution: it guarantees that the state and action distribution flow of the single player does match the population state and action sequence $\pmb{\mathcal{L}}^{\star}:=\{\mathcal{L}^\star\}_{t=0}^{\infty}$.

\subsection{Examples of GMFG}\label{sec:examples}
Here we provide three examples under the framework of GMFG.
\paragraph{A toy example.} \label{section:example} 
Take a two-state dynamic system with two choices of controls. The state space $\mathcal{S}$ $=$ $\{0, 1\}$, the action space $\mathcal{A}$ $=$ $\{L, R\}$.
Here  the action $L$ means to move left and $R$ means to move right. The dynamic of the representative agent in the mean-field system  $\{s_t\}_{t\ge 1}$ goes as follows: if  the agent  is in state $s_t$ and she takes  action $a_t={L}$ at time $t$, then $s_{t + 1}=0$; if  she takes action $a_t=R$, then $s_{t+1}=1$. At the end of each round, the agent will receive a reward  $-W_2(\mu_t, B)-W_2(\beta_t(s_t,\cdot), B)$, which depends on all agents, where $W_2$ is the $\ell_2$-Wasserstein distance. Here $\mu_t(\cdot)$ denotes the state distribution of the mean-field population at time $t$, {\color{black}$\beta_t(s,\cdot) := \mathcal{L}_t(s,\cdot)/\mu_t(s)$} denotes the action distribution of the population in state $s$ $(s=0,1)$ at time $t$ {\color{black}(set $\beta_t(s,\cdot):=(0.5,0.5)$ when $\mu_t(s)=0$)}, and $B$ is a given Bernoulli distribution with parameter $p$ ($0< p< 1$). 

{\color{black}As a demonstrating example, here we provide the calculation for one stationary NE solution.} Note that $-W_2(\mu, B)$ $\leq$ $0$ for any distribution $\mu$ over $\mathcal{S}$. Similarly, $-W_2(\alpha, B)$ $\leq$ $0$ for any distribution $\alpha$ over $\mathcal{A}$. Hence for each policy ${\pmb \pi}$, given population distribution flow $\pmb{\mathcal{L}} = \{\mathcal{L}_t\}_{t=1}^{\infty}$,
\begin{eqnarray}\label{eq:V0_expl}
V(0,{\pmb \pi},\pmb{\mathcal{L}}) &=& -\sum_{t=1}^\infty \gamma^t \mathbb{E}[W_2(\mu_t, B)+W_2(\beta_t(s_t,\cdot), B)|s_0=0] \;\leq\; 0,
\end{eqnarray}
and  
\begin{eqnarray}\label{eq:V1_expl}
V(1,{\pmb \pi}, \pmb{\mathcal{L}}) &=&  -\sum_{t=1}^\infty \gamma^t \mathbb{E}[W_2(\mu_t, B)+W_2(\beta_t(s_t,\cdot), B)|s_0=1] \;\leq\; 0.
\end{eqnarray} 
 
 It is easy to check that  $\mu^\star={\color{black}(p,1-p)}$ 
 and $\pi^\star(s,\mu^\star)={\color{black}(p,1-p)}$ ($s=0,1$).
{\color{black}is a pair of stationary mean-field solution.}
And $\mathcal{L}^\star$ is defined with $\mathcal{L}^\star(s,a)=\mu^\star(s)\pi^\star(a|s,\mu^\star)$ for any $s\in\mathcal{S},~a\in\mathcal{A}$, accordingly, where $\pi(a|s,\mu)$ is defined as the probability of taking action $a$ following the action distribution $\pi(s,\mu)$. 
In this case, the corresponding optimal value function is defined as
\[
V(0,\pi^\star,\mathcal{L}^\star)=V(1,\pi^\star,\mathcal{L}^\star)=0,
\]
{\color{black}which reaches the upper bound in \eqref{eq:V0_expl} and \eqref{eq:V1_expl}.}

\paragraph{Repeated auction.}
Take a representative advertiser in the auction aforementioned in the motivating example in Section \ref{introduction}.
Denote $s_t \in \{0,1,2,\cdots,s_{\max}\}$ as the budget of this player at time $t$, where $s_{\max}\in \mathbb{N}^+$ is  the maximum budget allowed on the Ad exchange with a unit bidding price.
Denote  {\color{black}$a_t\in\{0,1,2,\cdots,a_{\max}\}$ as the bid price submitted by this player, where $a_{\max}$ is the maximum bid set by the bidder},  and $\alpha_t$ as the bidding/(action) distribution of the population. {\color{black} At time $t$, all advertisers are randomly divided into different groups and each group of advertisers competes for one slot to display their ads. Assuming that there are $M$ advertisers in each group, then the representative advertiser competes with $M-1$ other representative players whose bidding prices are independently sampled from $\alpha_t$. Let $w_t^M$ denote whether the representative player wins the bid. Then if she takes action $a_t$, the probability she will win the bid is $\mathbb{P}(w_t^M=1) = F_{\alpha_t}(a_t)^{M-1}$, where $F_{\alpha_t}$ is the cumulative distribution function of a random variable $X\sim\alpha_t$. 

If this advertiser does not win the bid, her reward $r_t=0$. If she wins, there are several components in her reward: $a_t^M$, the second best bid in a Vickrey auction, paid by the winning advertiser; $v_t$, the conversion of clicks of the slot; and $\rho$, the rate of penalty for overshooting if the payment $a_t^M$ exceeds her budget $s_t$. Therefore, at each time $t$, her reward with bid $a_t$ and budget $s_t$ is 
  \begin{eqnarray}
\label{reward}
r_t = {\color{black}\textbf{I}_{\{w_t^M=1\}}}\left[(v_t-a^M_t)-(1+\rho){\color{black}\textbf{I}_{\{s_{t}<a^M_t\}}}(a^M_t-s_{t})\right],
\end{eqnarray}
where the first term is the profit of wining the auction and the second term is the penalty of overshooting.
And the budget dynamics $s_t$ follows, \begin{eqnarray}\label{transition}
    {s}_{t+1}=\left\{
                \begin{array}{ll}
                 s_t, \qquad & w^M_t \neq 1,\\
                 s_t - a^M_t, \qquad & w^M_t=1 \text{ and } a^M_t \leq s_t,\\
                 0, \qquad & w^M_t =1 \text{ and } a^M_t > s_t.
                \end{array}
              \right.
  \end{eqnarray}
  That is,  if this player does not win the bid, the budget remains the same; if she wins and has sufficient money to pay, her budget will decrease from $s_t$ to $s_t - a^M_t$; however, if she wins  but does not have enough money to pay, her budget will be $0$ after the payment and there will be a penalty in the reward function. 
  
  Notice that both distributions of $w_t^M$ and $a_t^M$ depend on the population distribution $\mathcal{L}_t$ (or more specifically $\alpha_t$). In fact, the reward function $r(s_t,a_t)=r_t$ and the transition probability $s_{t+1}\sim P(\cdot|s_t,a_t,\mathcal{L}_t)$ specified by \eqref{reward} and \eqref{transition} are fully characterized by the probabilities $\mathbb{P}(w_t^M=1,a_t^M\leq \cdot|s_t,a_t,\mathcal{L}_t)$ and $\mathbb{P}(w_t^M=0)$ (since $r_t= 0$ and $s_{t+1}=s_t$ whenever $w_t^M=0$), with 
  \[
  \mathbb{P}(w_t^M=1,a_t^M\leq x|s_t,a_t,\mathcal{L}_t)=F_{\alpha_t}(\min\{x,a_t\})^{M-1},\quad \mathbb{P}(w_t^M=0)=1-F_{\alpha_t}(a_t)^{M-1}.
  \]
  Clearly the above model fits into the framework of \eqref{mfg}, with the following transition probability.
  
 \begin{eqnarray}\label{detail_transition}
    \mathbb{P}(s'|s,a,\mathcal{L})=\left\{
                \begin{array}{ll}
            F_\alpha(a)^{M-1}-F_\alpha(\min\{s,a\})^{M-1}, \qquad & s'=0,\\
            1-F_\alpha(a)^{M-1}, \qquad & s'=s,\\
            F_\alpha(\min\{s-s',a\})^{M-1}-F_\alpha(\min\{s-s'-1,a\})^{M-1}, \qquad & 0<s'<s,
                \end{array}
              \right.
  \end{eqnarray} 
where $\alpha$ is the action marginal of $\mathcal{L}$. The reward model can be explicitly written similarly.
 
  In practice, one may modify the dynamics of $s_{t+1}$  with a non-negative random budget fulfillment $\Delta({s}_{t+1})$ after the auction clearing such that 
$\hat{s}_{t+1} = {s}_{t+1} + \Delta({s}_{t+1})$ \citet{andelman2004auctions,GKP2012}.

Experiments of this repeated auction problem can be found in the conference version \citet{guo2019learning} of this paper, and will not be repeated here.


}

{\color{black}
\paragraph{Equilibrium price.} 
Another example, adapted from \citet[Section 3]{gueant2011mean} is to consider a large number (continuum) of homogeneous firms producing the same product under perfect competition, 
 and the price of the product is determined endogenously by the supply-demand equilibrium \citet{bernstein2006regional}. Each firm, meanwhile, maintains a certain inventory level of the raw materials for production. 

Given the homogeneity of the firms, it is sufficient to focus on a representative firm paired with the population distribution. In each period $t$, the representative firm decides a quantity $q_t$ to consume the raw materials for production and a quantity $h_t$ to replenish the inventory of raw materials. For simplicity, we assume each unit of the raw material is used to produce one unit of the product. 
Both the new products and ordered raw materials will be available at the end of this given period $t$. The representative agent makes decision based on her current inventory level of the raw material,  denoted as $s_t$, which evolves according to
\begin{eqnarray}
s_{t+1} = s_t - \min\{q_t,s_t\} + h_t.
\end{eqnarray}
Note that if the firm  overproduces and exceeds her current inventory capacity (i.e., $q_t>s_t$), then the firm will pay a cost for an emergency order of the raw material.
Finally, the reward during this period $t$ is given by
\begin{eqnarray}
r_t = (p_t-c_0)\, q_t - c_1\,q_t^2 - c_2\,h_t - (c_2+c_3)\, \max\{q_t-s_t,0\} -c_4\, s_t.
\end{eqnarray}
Here $p_t$ is the selling price of the product of all firms; $c_0>0$ is the manufacturing cost and labor cost for making one unit of the product;  $c_1>0$ is the quadratic cost which can be viewed as the transient price impact associated with the production level $q_t$; $c_2>0$ is the cost of regular orders of the raw materials; $c_3>0$ is the additional cost for the emergency order of the raw materials; and finally, $c_4>0$ is the inventory cost.

The price $p_t$ is determined according to the supply-demand equilibrium on the market at each moment. On one hand, the normalized demand (per producer) on the market $D(p_t)$ follows (\citet{gueant2011mean})
\begin{eqnarray}
D(p_t) := d p_t^{-\sigma},
\end{eqnarray}
where $d$  denotes some benchmark demand level and $\sigma$ is the elasticity of demand that can be interpreted as the elasticity of substitution between the given product and any
other good. On the other hand, the (average) supply in this market is given by the average  production of all firms which follows
$\mathbb{E}_{q_t\sim\pi_t}[q_t]$ under some policy $\pi_t$. If all firms are restricted to stationary policies (denoted as $\pi$), then this leads to a stationary equilibrium price $q$ which satisfies the supply-demand equilibrium:
\begin{eqnarray}\label{price_equil_model}
\mathbb{E}_{q\sim\pi}[q] = d\, p^{-\sigma}.
\end{eqnarray}

To fit into the theoretical framework proposed in Section \ref{n2mfg}, we set $\mathcal{S} = \{0,1,\cdots,S\}$ and  $\mathcal{A} = \{(q,h)\,|\,q\in \{0,1,\cdots,Q\} \,\,{\rm and }\,\, h\in \{0,1,\cdots,H\} \}$ for some positive integers $S,Q$ and $H$. 
}

\section{Solution for GMFGs}\label{MFG_basic}
We now establish the existence and uniqueness of the {\color{black}stationary} NE  to (\ref{mfg}), by
generalizing the classical fixed-point approach for MFGs to this GMFG setting. (See \citet{HMC2006} and \citet{LL2007} for the classical case.)  It consists of three steps. 

\paragraph{Step A.}
Fix ${\mathcal{L}}$,   (\ref{mfg}) becomes the  classical {\color{black}single-player} optimization problem.  Indeed, with ${\mathcal{L}}$ fixed, the population state distribution ${\mu}$ is also fixed, and hence the space of admissible policies is reduced to the single-player case. Solving (\ref{mfg}) is now reduced to finding a policy 
 $\pi_{\mathcal{L}}^\star\in \Pi:=\{\pi\,|\,\pi:\mathcal{S}\rightarrow\mathcal{P}(\mathcal{A})\}$ to maximize 
 \[
\begin{array}{ll}
  V(s,\pi_{\mathcal{L}}, \mathcal{L}):= &\mathbb{E}\left[\sum\limits_{t=0}^{\infty}\gamma^tr(s_t,a_t,\mathcal{L})|s_0=s\right],\\
\text{subject to}& s_{t+1}\sim P(s_t,a_t,\mathcal{L}),\quad a_t\sim\pi_{\mathcal{L}}(s_t).
\end{array}
\]
{\color{black}Notice that with ${\mathcal{L}}$ fixed, one can safely suppress the dependency on $\mu$ in the admissible policies.}

Now given this fixed $\mathcal{L}$ and the solution $\pi_{\mathcal{L}}^\star$ to the above optimization problem, one can
define a mapping from the fixed population distribution $\mathcal{L}$ to a chosen optimal randomized policy sequence. That is, 
$$\Gamma_1:\mathcal{P}(\mathcal{S}\times \mathcal{A})\rightarrow\Pi, $$
such that  $\pi_{\mathcal{L}}^\star=\Gamma_1(\mathcal{L})$. {\color{black}Note that the optimal policy of an MDP in general may not be unique. To ensure that $\Gamma_1$ is a single-valued instead of set-valued mapping, here $\Gamma_1$ includes a policy selection component to select a single optimal policy from the set of optimal policies for a given $\mathcal{L}$, which is guaranteed to exist by Zermelo's Axiom of Choice. For example, when the action space is finite, one can utilize the \textbf{argmax-e} operator and set the ``maximizing'' actions with equal probabilities (see Section \ref{naive_alg} for the detailed definition). In addition,  for non-degenerate linear-quadratic MFGs \citet{MF-LQR} and general MFGs where the Bellman mappings are strongly concave in actions \citet{fitted-Q} and the action space is convex in the Euclidean space, the optimal policy $\pi_\mathcal{L}^\star$ for a given $\mathcal{L}$ is unique under appropriate assumptions. Hence no policy selection is needed in such cases.}

Note that this $\pi_{\mathcal{L}}^\star$ satisfies the single player side condition in Definition \ref{nash2_stat} for the population state-action pair $\mathcal{L}$, 
\begin{equation}
V\left(s,\pi_{\mathcal{L}}^\star,\mathcal{L}\right)\geq V\left(s,\pi,\mathcal{L}\right),
\end{equation}
for any policy $\pi$ and any initial state $s\in\mathcal{S}$. 

As in the MFG literature \citet{HMC2006}, a feedback regularity condition is needed for analyzing Step A. 
\begin{assumption}\label{policy_assumption_stat}
There exists a constant $d_1\geq 0$,
 such that for any $\mathcal{L}, \mathcal{L}' \in \mathcal{P}(\mathcal{S}\times\mathcal{A})$, 
\begin{equation}\label{Gamma1_lip_stat}
D(\Gamma_1(\mathcal{L}),\Gamma_1(\mathcal{L}'){\color{black})} \leq d_1W_1(\mathcal{L}, \mathcal{L}'),
\end{equation}
where 
\begin{equation}
\begin{split}
D(\pi,\pi')&:=\sup_{s\in\mathcal{S}}W_1(\pi(s),\pi'(s)),
\end{split}
\end{equation}
and $W_1$ is the $\ell_1$-Wasserstein distance (a.k.a. earth mover distance) between probability measures \citet{metrics_prob,  COT_cuturi, OT_ON}. 
\end{assumption}

 \paragraph{Step B.} Given $\pi_{\mathcal{L}}^\star$ obtained from Step A, update the initial $\mathcal{L}$ to $\mathcal{L}'$ following the controlled dynamics $P(\cdot|s_t, a_t,\mathcal{L})$.

Accordingly, for any admissible policy $\pi \in \Pi$ and a joint population state-action pair  $\mathcal{L}\in \mathcal{P}(\mathcal{S}\times \mathcal{A})$, 
define a mapping $\Gamma_2:\Pi\times \mathcal{P}(\mathcal{S}\times\mathcal{A})\rightarrow \mathcal{P}(\mathcal{S}\times\mathcal{A})$ as follows: 
\begin{eqnarray}
\Gamma_2(\pi, \mathcal{L}):=\hat{\mathcal{L}}= \mathbb{P}_{s_1,a_1},
\end{eqnarray}
where {\color{black}$a_1\sim \pi(s_1)$}, $s_{1}\sim \mu P(\cdot|\cdot,a_0,\mathcal{L})$,  $a_0\sim\pi(s_0)$, $s_0 \sim \mu$, and $\mu$ is the population state marginal of $\mathcal{L}$.

One needs a standard assumption in this step.
\begin{assumption}\label{population_assumption_stat}
There exist constants $d_2,~d_3\geq 0$, such that for any admissible policies $\pi,\pi_1,\pi_2$ and joint distributions $\mathcal{L}, \mathcal{L}_1, \mathcal{L}_2$, 
\begin{equation}\label{Gamma2_lip1_stat}
W_1(\Gamma_2(\pi_1,\mathcal{L}),\Gamma_2(\pi_2,\mathcal{L})) \leq d_2 D(\pi_1,\pi_2),\\
\end{equation}
\begin{equation}\label{Gamma2_lip2_stat}
W_1(\Gamma_2(\pi,\mathcal{L}_1),\Gamma_2(\pi,\mathcal{L}_2)) \leq d_3 W_1(\mathcal{L}_1,\mathcal{L}_2).
\end{equation}
\end{assumption}

\paragraph{Step C.} Repeat Step A and Step B until $\mathcal{L}'$ matches $\mathcal{L}$.

This step is to ensure the population side condition. To ensure the convergence of
 the combined step one and step two,  it suffices if  $\Gamma:\mathcal{P}(\mathcal{S}\times\mathcal{A})\rightarrow \mathcal{P}(\mathcal{S}\times\mathcal{A})$ with $\Gamma(\mathcal{L}):=\Gamma_2(\Gamma_1(\mathcal{L}), \mathcal{L})$ is a contractive mapping under  the $W_1$ distance. Then by the Banach fixed point theorem and the completeness of the related metric spaces (\textit{cf}. Appendix \ref{appendixA}), there exists  a unique stationary NE of the GMFG.  That is, 

 \begin{theorem}[Existence and Uniqueness of stationary GMFG solution] \label{thm1_stat} Given Assu-\\mptions \ref{policy_assumption_stat} and \ref{population_assumption_stat}, and assume $d_1d_2+d_3< 1$. Then there exists a unique stationary NE  to \eqref{mfg}.
\end{theorem}
\begin{proof}{[Proof of Theorem \ref{thm1_stat}]}
First by Definition \ref{nash2_stat} and the definitions of $\Gamma_i$ $(i=1,2)$, $(\pi,\mathcal{L})$ is a stationary NE  iff $\mathcal{L}=\Gamma(\mathcal{L})=\Gamma_2(\Gamma_1(\mathcal{L}),\mathcal{L})$ and $\pi=\Gamma_1(\mathcal{L})$, where $\Gamma(\mathcal{L})=\Gamma_2(\Gamma_1(\mathcal{L}),\mathcal{L})$. This indicates that for any $\mathcal{L}_1,\mathcal{L}_2\in\mathcal{P}(\mathcal{S}\times\mathcal{A})$,
\begin{equation}
\begin{split}
    &W_1(\Gamma(\mathcal{L}_1),\Gamma(\mathcal{L}_2))=W_1(\Gamma_2(\Gamma_1(\mathcal{L}_1),\mathcal{L}_1),\Gamma_2(\Gamma_1(\mathcal{L}_2),\mathcal{L}_2))\\
   &\leq W_1(\Gamma_2(\Gamma_1(\mathcal{L}_1),\mathcal{L}_1),\Gamma_2(\Gamma_1(\mathcal{L}_2),\mathcal{L}_1))+W_1(\Gamma_2(\Gamma_1(\mathcal{L}_2),\mathcal{L}_1),\Gamma_2(\Gamma_1(\mathcal{L}_2),\mathcal{L}_2))\\
   &\leq (d_1d_2+d_3)W_1(\mathcal{L}_1,\mathcal{L}_2).
    \end{split}
\end{equation}
And since $d_1d_2+d_3\in[0,1)$, by the Banach fixed-point theorem, we conclude that there exists a unique fixed-point of $\Gamma$, or equivalently, a unique stationary MFG solution to \eqref{mfg}.
\end{proof}

{\color{black}
\begin{remark}[Existence and Uniqueness of the GMFG solution] \label{rmk:contract}
(1) In general, there may multiple optimal policies in Step A under a fixed mean-field information $\mathcal{L}$. In this case, the candidate fixed point(s) are the fixed point(s) of a set-valued map as described in \citet{lacker2015mean}. To simplify the analysis, we specify a rule in Step A to select one optimal policy to ensure that $\Gamma$ is an injection.

(2)
In the MFG literature, the uniqueness of the MFG solution can be verified under the small parameter condition \citet{CHM_MFG} or the monotonicity condition \citet{LL2007}. Our condition of $d_1 d_2 +d_3<1$ extends the small parameter condition in \citet{CHM_MFG} for strict controls to relaxed controls. 

(3) Finally, Theorem \ref{thm1_stat} can be extended to a non-stationary setting, as will be shown in Section \ref{stat_mfg_app}. 
\end{remark}
}


{\color{black}
\begin{remark}
Assumptions \ref{policy_assumption_stat} and \ref{population_assumption_stat} can be more explicit in specific problem settings. 

For instance, {\color{black}when the action space is the Euclidean space or its convex subset}, explicit conditions on $P$ and $r$ have been described for the linear-quadratic MFG (LQ-MFG) \citet{MF-LQR} and later generalized in \citet{fitted-Q}. 

When the action space is finite, the following lemma explicitly characterizes Assumption \ref{population_assumption_stat}. 
\begin{lemma}\label{assumption2_exp}
Suppose that $\max_{s,a,\mathcal{L},s'}P(s'|s,a,\mathcal{L})\leq c_1$, and that $P(s'|s,a,\cdot)$ is $c_2$-Lipschitz in $W_1$, \textit{i.e.},
\begin{equation}
|P(s'|s,a,\mathcal{L}_1)-P(s'|s,a,\mathcal{L}_2)|\leq c_2W_1(\mathcal{L}_1,\mathcal{L}_2).
\end{equation}
Then in Assumption \ref{population_assumption_stat}, $d_2$ and $d_3$ can be chosen as 
\begin{equation}
d_2=\frac{2\text{diam}(\mathcal{S})\text{diam}(\mathcal{A})|\mathcal{S}|c_1}{d_{\min}(\mathcal{A})}
\end{equation}
 and $\text{diam}(\mathcal{S})\text{diam}(\mathcal{A})|\mathcal{S}|\left(\frac{c_2}{2}+\frac{1}{d_{\min}(\mathcal{S}\times\mathcal{A})}\right)$
 respectively. Here $d_{\min}(\mathcal{A})=\min_{a\neq a'\in\mathcal{A}}\|a-a'\|_2$, which is guaranteed to be positive when $\mathcal{A}$ is finite.
\end{lemma}

{\color{black} When entropy regularization is introduced into the system (see \textit{e.g.}, \citet{anahtarci2020q, xie2020provable}), Assumption \ref{policy_assumption_stat} can be reduced to boundedness and Lipschitz continuity conditions on $P$ and $r$ as in Lemma \ref{assumption2_exp}. Moreover, Theorem \ref{thm1_stat} and all subsequent theoretical results hold whenever the composed mapping $\Gamma$ is contractive (in $W_1$), independent of Assumptions \ref{policy_assumption_stat} or \ref{population_assumption_stat}. In Section \ref{performance_eval}, we numerically verify that the $\Gamma$ mapping is contractive for various choices of the model parameters in our tested problems. } 





\end{remark}
}

\section{Naive algorithm and stabilization techniques}\label{AIQL}


In this section, we design algorithms for the GMFG. Since the reward and transition distributions are unknown, this is  simultaneously learning the system and finding the  NE of the game. 
We will focus on the case with finite state and action spaces, \textit{i.e.}, $|\mathcal{S}|,|\mathcal{A}|<\infty$. 
We will look for stationary (time independent)  NEs. This stationarity property enables developing appropriate  {\color{black}stationary reinforcement learning} algorithms, suitable for an infinite time horizon game. Instead of knowing the transition probability $P$ and the reward $r$ explicitly, the algorithms we propose only assume access to a simulator oracle, which is described below. {\color{black}This is not restrictive in practice.
For instance, in the ad auction example, one may adopt the bid recommendation perspective of the publisher, say Google, Facebook or Amazon, who acts as the auctioneer and owns the Ad slot inventory on its own Ad exchange platform. 
In this case, a high quality auction simulator is typically built and maintained by a team of the publisher. 
See also {\color{black}\citet{rl_mfg_local}} 
for more examples.}

\paragraph{\color{black} Simulator oracle.}
{\color{black}
For any policy $\pi\in\Pi$, given the current state $s\in\mathcal{S}$, for any population distribution $\mathcal{L}$, one can obtain a \textit{sample} of the next state $s'\sim P(\cdot|s,\pi(s),\mathcal{L})$, a reward $r= r(s,\pi(s),\mathcal{L})$, and the next population distribution $\mathcal{L}'=\mathbb{P}_{s',\pi(s')}$. For brevity, we denote the simulator as $(s',r, \mathcal{L}')=\mathcal{G}(s,\pi,\mathcal{L})$. 
} {\color{black}This simulator oracle can be weakened to fit the $N$-player setting, see Section \ref{appl_n_game}.}

{\color{black}In the following, we begin with a naive algorithm that simply combines the three-step fixed point approach with general RL algorithms,  and demonstrate that this algorithm can be unstable (Section \ref{naive_alg}). We then propose some smoothing and projection techniques to resolve the issue (Section \ref{address_instab}). In Section \ref{value-based} and Section \ref{policy_based}, we design general  value-based and policy-based RL algorithms, and establish the corresponding convergence and complexity results.  These two algorithms include most of the RL algorithms in the literature.
We then illustrate by two concrete examples based on Q-learning and trust-region policy optimization algorithms. }



\subsection{Naive algorithm and its {\color{black}issue}}\label{naive_alg}


We follow the three-step fixed-point approach described in Section \ref{MFG_basic}. Notice the fact that with $\mathcal{L}$ fixed, Step A in Section \ref{MFG_basic} becomes a standard learning problem for an infinite horizon discounted MDP. 
More specifically, the MDP to be solved is $\mathcal{M}_{\mathcal{L}}=(\mathcal{S},\mathcal{A},P_{\mathcal{L}},r_{\mathcal{L}},\gamma)$, where $P_{\mathcal{L}}(s'|s,a)=P(s'|s,a,{\mathcal{L}})$ and $r_{\mathcal{L}}(s,a)=r(s,a,{\mathcal{L}})$. 
{\color{black}
In general, for an MDP $\mathcal{M}=(\mathcal{S},\mathcal{A}, P,r,\gamma)$, for any policy $\pi$ one can define its value functions $V^\pi_{\mathcal{M}}(s)=\mathbb{E}\left[\sum_{t=0}^\infty\gamma^tr(s_t,a_t)|s_0=s\right]$ and its Q-functions $Q^\pi_{\mathcal{M}}(s,a) = \mathbb{E}\left[\sum_{t=0}^\infty\gamma^tr(s_t,a_t)|s_0=s,a_0=a\right]$, where $s_t,a_t$ is the trajectory under policy $\pi$. One can also define the optimal Q-function as the unique solution of the Bellman equation: $$Q_{\mathcal{M}}^\star(s,a) = \mathbb{E}[r(s,a)]+\gamma\max_{a'}\sum_{s'\in\mathcal{S}}P(s'|s,a)Q_{\mathcal{M}}^\star(s',a')$$ for all $s,a$ and its optimal value function
$V_{\mathcal{M}}^\star(s) = \max_{a}Q_{\mathcal{M}}^\star(s,a)$ for all $s$. 
}
We also use the shorthand $V_{\mathcal{L}}^\star=V_{\mathcal{M}_{\mathcal{L}}}^\star$ and $Q_{\mathcal{L}}^\star=Q_{\mathcal{M}_{\mathcal{L}}}^\star$ for notational brevity. 
Whenever the context is clear, we may omit $\mathcal{M}$, $\mathcal{L}$ and $\mathcal{M}_{\mathcal{L}}$ for notational convenience. 

Given the optimal Q-function $Q_{\mathcal{L}}^\star$, one can obtain an optimal policy $\pi_{\mathcal{L}}^\star$ with $\pi_{\mathcal{L}}^\star(s)=\textbf{argmax-e}(Q_{\mathcal{L}}^\star(s,\cdot))$. Here the \textbf{argmax-e} operator is defined so that actions with equal maximum Q-values would have equal probabilities to be selected. Hereafter, we specify $\Gamma_1$ as a mapping to the aforementioned choice of the optimal policy, \textit{i.e.}, the $s$-component $\Gamma_1(\mathcal{L})_s=\textbf{argmax-e}(Q_{\mathcal{L}}^\star(s,\cdot))$ for any $s\in\mathcal{S}$.



The population update in Step B can then be directly obtained from the simulator $\mathcal{G}$ following policy $\pi_{\mathcal{L}}^\star$. Combining these two steps leads to the following naive algorithm (Algorithm \ref{naive}).


\begin{algorithm}[h]
  \caption{\textbf{Naive Reinforcement Learning for GMFGs}}
  \label{naive}
\begin{algorithmic}[1]
  \STATE \textbf{Input}: Initial population state-action pair $\mathcal{L}_0$
 \FOR {$k=0, 1, \cdots$}
  \STATE Obtain the optimal Q-function $Q_k(s,a)= Q_{\mathcal{L}_k}^\star(s,a)$ of an MDP with dynamics $P_{\mathcal{L}_k}(s'|s,a)$ and reward distributions $R_{\mathcal{L}_k}(s,a)$.
  \STATE Compute $\pi_k\in \Pi$ with $\pi_k(s)=\textbf{argmax-e}\left(Q_k(s,\cdot)\right)$.
  \STATE Sample $s\sim \mu_k$, where $\mu_k$ is the population state marginal of $\mathcal{L}_k$, and obtain $\mathcal{L}_{k+1}$ from $\mathcal{G}(s,\pi_k,\mathcal{L}_k)$.
\ENDFOR
\end{algorithmic}
\end{algorithm}

Unfortunately, in practice, one cannot obtain the exact optimal Q-function $Q_k$. In fact, invoking any commonly used RL algorithm with the simulator $\mathcal{G}$ leads to an approximation $\hat{Q}_k$ of the actual $Q_k$. This approximation error is then magnified by the discontinuous and sensitive \textbf{argmax-e}, which eventually leads to an unstable algorithm (see Figure \ref{fig:naive} for an example of divergence).  
To see why \textbf{argmax-e} is not continuous, consider the following simple example.
Let $x=(1,1)$, then $\textbf{argmax-e}(x)=(1/2,1/2)$.
For any $\epsilon>0$, let $y_{\epsilon}=(1,1-\epsilon)$, then $\textbf{argmax-e}(y_{\epsilon})=(1,0)$. Hence $\lim_{\epsilon\rightarrow0}y_{\epsilon}=x$ but $$\lim_{\epsilon\rightarrow0}\textbf{argmax-e}(y_{\epsilon})\neq\textbf{argmax-e}(x).$$ This instability issue will be addressed by introducing smoothing and projection techniques. 


{\color{black}

\subsection{Restoring stability}\label{address_instab}
\paragraph{Smoothing techniques.} 
%

To address the instability caused, we replace \textbf{argmax-e} with a smooth function that is a good approximation to \textbf{argmax-e} while being Lipschitz continuous. One such candidate is the softmax operator $\textbf{softmax}_{c}:\mathbb{R}^n\rightarrow\mathbb{R}^n$, with $$\textbf{softmax}_{c}(x)_i=\frac{\exp(cx_i)}{\sum_{j=1}^n\exp(cx_j)}, \quad i=1,\dots,n,$$ for some positive constant $c$.  
The resulting policies are sometimes called Boltzmann policies, and are widely used in the literature of reinforcement learning \citet{Mellowmax,softQ}.

The softmax operator can be generalized to a wide class of operators. In fact, for positive constants $c,~c'>0$, one can consider a parametrized family $\mathcal{F}_{c,c'}\subseteq \{f_{c,c'}\,:\,\mathbb{R}^n\rightarrow\mathbb{R}^n\}$ of all ``smoothed'' \textbf{argmax-e}'s, \textit{i.e.}, all $f_{c,c'}:\mathbb{R}^n\rightarrow\mathbb{R}^n$ that satisfies the following two conditions:
\begin{itemize}
\item Condition 1: $f_{c,c'}$ is $c$-Lipschitz, \textit{i.e.}, $\|f_{c,c'}(x)-f_{c,c'}(y)\|_2\leq c\|x-y\|_2$.
\item Condition 2: $f_{c,c'}$ is a good approximation of \textbf{argmax-e}, \textit{i.e.}, 
\[
\|f_{c,c'}(x)-\textbf{argmax-e}(x)\|_2\leq 2n\exp(-c'\delta),
\]
where $\delta=x_{\max}-\max_{x_j<x_{\max}}x_j$, $x_{\max}=\max_{i=1,\dots,n}x_i$, and $\delta:=\infty$ when all $x_j$ are equal.
\end{itemize}
Notice that $\mathcal{F}_{c,c'}$ is closed under convex combinations, \textit{i.e.}, if $f_{c,c'},~g_{c,c'}\in\mathcal{F}_{c,c'}$, then for any $\lambda\in[0,1]$, $\lambda f_{c,c'}+(1-\lambda)g_{c,c‘}$ also satisfies the two conditions. Hence $\mathcal{F}_{c,c'}$ is convex. 

To have a better idea of what $\mathcal{F}_{c,c'}$ looks like, we describe a subset $\mathcal{B}_{c,c'}$ of $\mathcal{F}_{c,c'}$ consisting of the generalized softmax operator $\textbf{softmax}_{h}:\mathbb{R}^n\rightarrow\mathbb{R}^n$, defined as 
\begin{eqnarray}\label{eqn:soft_max_operator}
\textbf{softmax}_{h}(x)_i=\frac{\exp(h(x_i))}{\sum_{j=1}^n\exp(h(x_j))}, \quad i=1,\dots,n,
\end{eqnarray}
where $h:\mathbb{R}\rightarrow\mathbb{R}$ satisfies $c'(x-y)\leq h(x)-h(y)\leq c(x-y)$ for any $x\geq y$. When $h$ is continuously differentiable, a sufficient condition is that $c'\leq h'(x)\leq c$. In particular, if $h(x)\equiv cx$ for some constant $c>0$, the operator reduces to the classical softmax operator, in which case we overload the notation to write $\textbf{softmax}_{h}$ as $\textbf{softmax}_c$.

This operator is Lipschitz continuous and close to the \textbf{argmax-e} (see Lemmas  \ref{softmax-lip} and \ref{soft-arg-diff} in the Appendix), and in particular one can show that $\mathcal{B}_{c,c'}\subseteq \mathcal{F}_{c,c'}$.  As a result,
even though smoothed (\textit{e.g.}, Boltzmann) policies are not optimal, the difference between the smoothed and the optimal one can always be controlled by choosing a {\color{black} function $h$ with appropriate parameters $c, c'$.} Note that other smoothing operators (\textit{e.g.}, Mellowmax \citet{Mellowmax}, which is  a softmax operator with time-varying and problem dependent temperatures) may also be considered.

}

\paragraph{Error control in updating $\mathcal {L}$.}
Given the sub-optimality of the  smoothed policy, one needs to characterize the difference between the optimal policy and the non-optimal ones. In particular, one can define the action gap  between the best and the second best actions in terms of the  Q-value as $$\delta^s(
\mathcal{L}):=\max_{a'\in\mathcal{A}}Q^\star_{\mathcal{L}}(s,a')-\max_{a\notin \text{argmax}_{a\in\mathcal{A}}Q^\star_{\mathcal{L}}(s,a)}Q^\star_{\mathcal{L}}(s,a)>0.$$ 
Action gap is important for approximation algorithms \citet{gap-increase}, and is closely related to the problem-dependent bounds for regret analysis in reinforcement learning and  multi-armed bandits, and advantage learning algorithms including A{\color{black}3}C \citet{A2C}. 

The problem is: in order for the learning algorithm to converge in terms of $\mathcal{L}$ (Theorems \ref{conv_AIQL} and \ref{conv_AIQL-II}), one needs to ensure a definite differentiation between the optimal policy and the sub-optimal ones. {\color{black}This is problematic as} the infimum of $\delta^s(\mathcal{L})$ over an infinite number of $\mathcal{L}$ can be $0$. To address this, the population distribution at step $k$, say  $\mathcal{L}_k$,
 needs to be projected to a finite grid,  called $\epsilon$-net. The relation between
the $\epsilon$-net and \text{action gaps} is as follows:

\textit{For any $\epsilon>0$,
there exist a positive function $\phi(\epsilon)$ and  an $\epsilon$-net $S_{\epsilon}:=\{\mathcal{L}^{(1)},\dots,\mathcal{L}^{(N_{\epsilon})}\}$ $\subseteq$ $\mathcal{P}(\mathcal{S}\times\mathcal{A})$, with the properties that $\min_{i=1,\dots,N_{\epsilon}} d_{TV}(\mathcal{L},\mathcal{L}^{(i)})\leq \epsilon$ for any $\mathcal{L}\in\mathcal{P}(\mathcal{S}\times\mathcal{A})$, and that 
$\max_{a'\in\mathcal{A}}Q^\star_{\mathcal{L}^{(i)}}(s,a')-Q^\star_{\mathcal{L}^{(i)}}(s,a)\geq \phi(\epsilon)$
 for any $i=1,\dots,N_{\epsilon}$,  $s\in\mathcal{S}$, and any $a\notin \text{argmax}_{a\in\mathcal{A}}Q^\star_{\mathcal{L}^{(i)}}(s,a)$. 
}

Here the existence of $\epsilon$-nets is trivial due to the compactness of the probability simplex $\mathcal{P}(\mathcal{S}\times\mathcal{A})$, and the existence of $\phi(\epsilon)$ comes from the finiteness of the action set $\mathcal{A}$.

{\color{black}In practice, $\phi(\epsilon)$ often takes the form of  $D\epsilon^{\alpha}$  with $D>0$ and the exponent $\alpha>0$ characterizing the decay rate  of the action gaps.} {\color{black}In general, experiments are robust with respect to the choice of  $\epsilon$-net.}



\color{black}
%


{\color{black}In the next section, we propose value based and policy based algorithms for learning GMFG.}

\section{RL Algorithms for {\color{black}(stationary)} GMFGs}
\subsection{Value-based algorithms} \label{value-based}
 We start by introducing the following definition.



\begin{definition}[{\color{black}Value-based Guarantee}]\label{value_based_gua}
For an arbitrary MDP $\mathcal{M}$, we say that an algorithm has a \textit{value-based guarantee} with parameters $\{C_{\mathcal{M}}^{(i)},\alpha_1^{(i)},\alpha_2^{(i)},\alpha_3^{(i)},\alpha_4^{(i)}\}_{i=1}^m$, if for any $\epsilon,~\delta>0$, after obtaining 
\begin{equation}\label{sample_typebound}
T_{\mathcal{M}}(\epsilon,\delta)=\sum_{i=1}^mC_{\mathcal{M}}^{(i)}\left(\frac{1}{\epsilon}\right)^{\alpha_1^{(i)}}\left(\log\frac{1}{\epsilon}\right)^{\alpha_2^{(i)}}\left(\frac{1}{\delta}\right)^{\alpha_3^{(i)}}\left(\log\frac{1}{\delta}\right)^{\alpha_4^{(i)}}
\end{equation} samples from the simulator oracle $\mathcal{G}$, with probability at least $1-2\delta$, it outputs an approximate Q-function $\hat{Q}^{T_{\mathcal{M}}(\epsilon,\delta)}$  which satisfies $\|\hat{Q}^{T_{\mathcal{M}}(\epsilon,\delta)}-Q^\star\|_\infty\leq \epsilon$. Here the norm $\|\cdot\|_{\infty}$ is understood element-wisely.
\end{definition}

\subsubsection{GMF-V}

We now state the first main algorithm (Algorithm \ref{GMF-V}). It applies to any algorithm \textit{Alg} with a value-based guarantee.
\begin{algorithm}[h]
  \caption{\textbf{GMF-V(\textit{Alg}, ~$f_{c,c'}$)}}
  \label{GMF-V}
\begin{algorithmic}[1]
  \STATE \textbf{Input}: Initial $\mathcal{L}_0$, $\epsilon$-net $S_{\epsilon}$, temperatures $c,~c'>0$, tolerances $\epsilon_k,~\delta_k>0$, $k=0,1,\dots$. 
 \FOR {$k=0, 1, \cdots$}
  \STATE 
  \begin{addmargin}[1em]{0em}
  Apply \textit{Alg} to  find the approximate Q-function $\hat{Q}_k^\star=\hat{Q}^{T_k}$ of the MDP $\mathcal{M}_{\mathcal{L}_k}$, where $T_k=T_{\mathcal{M}_{\mathcal{L}_k}}(\epsilon_k,\delta_k)$.
 \end{addmargin}
  \STATE 
    \begin{addmargin}[1em]{0em}
  Compute  $\pi_k(s)=f_{c,c'}(\hat{Q}^\star_k(s,\cdot))$.
  \end{addmargin}
  \STATE 
    \begin{addmargin}[1em]{0em}
  Sample $s\sim \mu_k$ ($\mu_k$ is the population state marginal of $\mathcal{L}_k$), obtain $\tilde{\mathcal{L}}_{k+1}$ from $\mathcal{G}(s,\pi_k,\mathcal{L}_k)$.
  \end{addmargin}
  \STATE 
    \begin{addmargin}[1em]{0em}
  Find $\mathcal{L}_{k+1}=\textbf{Proj}_{S_{\epsilon}}(\tilde{\mathcal{L}}_{k+1})$
  \end{addmargin}
\ENDFOR
\end{algorithmic}
\end{algorithm}

\vspace{1em}

Here $\textbf{Proj}_{S_{\epsilon}}(\mathcal{L})=\text{argmin}_{\mathcal{L}^{(1)},\dots,\mathcal{L}^{(N_{\epsilon})}}d_{TV}(\mathcal{L}^{(i)},\mathcal{L})$. 
{For computational tractability, it is sufficient to choose $S_{\epsilon}$ as a truncation grid so that projection of $\tilde{\mathcal{L}}_k$ onto the $\epsilon$-net reduces to truncating $\tilde{\mathcal{L}}_k$ to a certain number of digits. For instance, in experiments in Section \ref{experiments}, the number of digits is chosen to be 4. Appropriate choices of the hyper-parameters $c,~c',~\epsilon$ and tolerances $\epsilon_k,~\delta_k$ ($k\geq 0$) are given in Theorems \ref{conv_AIQL}. Our experiment shows the algorithm is robust with respect to these hyper-parameters.} 



We next establish the convergence of the above GMF-V algorithm to an approximate Nash equilibrium of \eqref{mfg}, with complexity analysis.

\begin{theorem}[Convergence and complexity of GMF-V]\label{conv_AIQL}
Assume the same assumptions as Theorem \ref{thm1_stat}, {\color{black}and $f_{c,c'}\subseteq \mathcal{F}_{c,c'}$}. Suppose that \textit{Alg} has a value-based guarantee with parameters 
\[
\{C_{\mathcal{M}}^{(i)},\alpha_1^{(i)},\alpha_2^{(i)},\alpha_3^{(i)},\alpha_4^{(i)}\}_{i=1}^m.
\] 
For any $\epsilon,~\delta>0$, set $\delta_k=\delta/ K_{\epsilon,\eta}$, $\epsilon_k=(k+1)^{-(1+\eta)}$ for some $\eta\in(0,1]$ $(k=0,\dots,K_{\epsilon,\eta}-1)$, and  {\color{black}$c\geq c'\geq\frac{\log(1/\epsilon)}{\phi(\epsilon)}$}.
Then with probability at least $1-2\delta$, $$W_1(\mathcal{L}_{K_{\epsilon,\eta}},\mathcal{L}^\star)\leq C_0\epsilon.$$   Here $K_{\epsilon,\eta}:=\left\lceil 2\max\left\{(\eta\epsilon/c)^{-1/\eta},\log_d(\epsilon/\max\{\text{diam}(\mathcal{S})\text{diam}(\mathcal{A}),c\})+1\right\}\right\rceil$ is the number of outer iterations, and the constant $C_0$ is independent of $\delta$, $\epsilon$ and $\eta$. \\
Moreover,  the total number of {\color{black}samples} $T=\sum_{k=0}^{K_{\epsilon,\eta} -1}T_{\mathcal{M}_{\mathcal{L}_k}}(\delta_k,\epsilon_k)$ is bounded by
\begin{equation}\label{Tbound-I}
T\leq \sum_{i=1}^m\dfrac{2^{\alpha_2^{(i)}}}{2\alpha_1^{(i)}+1}C_{\mathcal{M}}^{(i)}K_{\epsilon,\eta}^{2\alpha_1^{(i)}+1}(K_{\epsilon,\eta}/\delta)^{\alpha_3^{(i)}}\left(\log(K_{\epsilon,\eta}/\delta)\right)^{\alpha_2^{(i)}+\alpha_4^{(i)}}.
\end{equation}
\end{theorem}

The proof of Theorem \ref{conv_AIQL} (in Section \ref{thm2-thm3}) depends on the Lipschitz continuity of the smoothing operator $f_{c,c'}$, the closeness between $f_{c,c'}$ and the \textbf{argmax-e} (Lemma~\ref{soft-arg-diff} in Section \ref{bcc-fcc}), and the complexity of \textit{Alg} provided by the value-based guarantee. 

\subsubsection{GMF-V-Q: GMF-V with {\color{black}Q-learning}}
As an example of the GMF-V algorithm, we describe algorithm GMF-V-Q, a Q-learning based GMF-V algorithm. 
{\color{black}
For an MDP $\mathcal{M}=(\mathcal{S},\mathcal{A},P,r,\gamma)$, the synchronous Q-learning algorithm approximates the value iteration by stochastic approximation.
At each step $l$, with state $s$ and action $a$, the system reaches state $s'$ according to the
controlled dynamics, and the Q-function approximation $Q_l$ is updated by
\begin{equation}\label{q-update}
\hat{Q}^{l+1}(s,a)= (1-\beta_l)\hat{Q}^l(s,a)+\beta_l\left[r(s,a)+\gamma\max_{\bar{a}}\hat{Q}^l(s',\bar{a})\right],\quad \forall s\in\mathcal{S},\,a\in\mathcal{A},
\end{equation}
where $\hat{Q}^0(s,a)=C$ for some constant $C\in\mathbb{R}$ for any $s\in\mathcal{S}$ and $a\in\mathcal{A}$, and the step size $\beta_l$ can  be chosen as (\citet{Q-rate})
\begin{equation}\label{q-step-size}
\beta_l=|l+1|^{-h},
\end{equation}
with $h\in(1/2,1)$. 
}

The corresponding {\color{black}synchronous} Q-learning based algorithm with the standard \textbf{softmax} operator is GMF-V-Q (Algorithm \ref{AIQL_MFG}), and will be used in the experiment (Section \ref{experiments}).

\begin{algorithm}[h]
  \caption{\textbf{Q-learning for GMFGs (GMF-V-Q)}}
  \label{AIQL_MFG}
\begin{algorithmic}[1]
  \STATE \textbf{Input}: Initial $\mathcal{L}_0$, {\color{black}$\epsilon$-net $S_{\epsilon}$, tolerances $\epsilon_k,~\delta_k>0$, $k=0,1,\dots$.} 
 \FOR {$k=0, 1, \cdots$}
  \STATE 
  \begin{addmargin}[1em]{0em}
  Perform {\color{black} synchronous} Q-learning with stepsizes \eqref{q-step-size} for {\color{black}$T_k=T_{\mathcal{M}_{\mathcal{L}_k}}(\epsilon_k,\delta_k)$} iterations to find the approximate Q-function {\color{black}$\hat{Q}_k^\star=\hat{Q}^{T_k}$} of the MDP {\color{black}$\mathcal{M}_{\mathcal{L}_k}$}. 
  \end{addmargin}
  \STATE 
  \begin{addmargin}[1em]{0em}
  Compute $\pi_k\in\Pi$ with $\pi_k(s)=\textbf{softmax}_c(\hat{Q}^\star_k(s,\cdot))$.
  \end{addmargin}
  \STATE 
  \begin{addmargin}[1em]{0em}
  Sample $s\sim \mu_k$ ($\mu_k$ is the population state marginal of $\mathcal{L}_k$), obtain $\tilde{\mathcal{L}}_{k+1}$ from $\mathcal{G}(s,\pi_k,\mathcal{L}_k)$.
  \end{addmargin}
  \STATE
  \begin{addmargin}[1em]{0em} 
  Find $\mathcal{L}_{k+1}=\textbf{Proj}_{S_{\epsilon}}(\tilde{\mathcal{L}}_{k+1})$
  \end{addmargin}
\ENDFOR
\end{algorithmic}
\end{algorithm}

{\color{black}
Let us first recall the following sample complexity result for synchronous Q-learning method.
\begin{lemma}[\citet{Q-rate}: sample complexity of synchronous Q-learning]\label{Q-finite-bd} For an MDP, say $\mathcal{M}=(\mathcal{S},\mathcal{A},P,r,\gamma)$, suppose that the Q-learning algorithm takes step-sizes \eqref{q-step-size}.
Then $\|{\color{black}\hat{Q}^{T_{\mathcal{M}}(\delta,\epsilon)}}-Q_{\mathcal{M}}^\star\|_{\infty}\leq \epsilon$  with probability at least $1-2\delta$. Here {\color{black}$\hat{Q}^T$} is the $T$-th update in the  Q-learning updates \eqref{q-update}, $Q_{\mathcal{M}}^\star$ is the (optimal) Q-function, and 
\[
T_{\mathcal{M}}(\epsilon,\delta) = \Omega\left(\left(\frac{V_{\max}^2\log\left(\frac{|\mathcal{S}||\mathcal{A}|V_{\max}}{\delta\beta\epsilon}\right)}{\beta^2\epsilon^2}\right)^{\frac{1}{h}} + \left(\frac{1}{\beta}\log\frac{V_{\max}}{\epsilon}\right)^{\frac{1}{1-h}}\right),
\]

where $\beta=(1-\gamma)/2$, $V_{\max}=R_{\max}/(1-\gamma)$, and $R_{\max}$ {\color{black} is such that a.s.
$0\leq {\color{black}r(s,a)}\leq R_{\max}$. 
}
\end{lemma} 
}


{\color{black}
This lemma implies immediately the value-based guarantee (as in Definition \ref{value_based_gua}) and the convergence for GMF-V-Q. Similar results can be established for asynchronous Q-learning method, as shown in Appendix \ref{asyn-q}.
\begin{corollary}
The synchronous Q-learning algorithm with appropriate choices of step-sizes (\textit{cf}. \eqref{q-step-size}) satisfies the value-based guarantee with parameters $\{\tilde{C}_{\mathcal{M}}^{(i)},\alpha_1^{(i)},\alpha_2^{(i)},\alpha_3^{(i)},\alpha_4^{(i)}\}_{i=1}^{3}$, where $C_{\mathcal{M}}^{(i)}(i=1,2,3)$ are constants depending on  $|\mathcal{S}|,|\mathcal{A}|,V_{\max},\beta$ and $h$, and
\[
\begin{split}
&\text{$\alpha_1
^{(1)}=2/h$, $\alpha_2^{(1)}=1/h$, $\alpha_3^{(1)}=\alpha_4^{(1)}=0$;}\\
&\text{$\alpha_1
^{(2)}=2/h$, $\alpha_2^{(2)}=\alpha_3^{(2)}=0$ and $\alpha_4^{(2)}=1/h$;}\\
&\text{$\alpha_1
^{(3)}=0$, $\alpha_2^{(3)}=1/(1-h)$, $\alpha_3^{(3)}=0$ and $\alpha_4^{(3)}=0$.}
\end{split}
\]
In addition, assume the same assumptions as Theorem \ref{thm1_stat}, 
then for Algorithm \ref{AIQL_MFG} with synchronous Q-learning method, with probability at least $1-2\delta$, $W_1(\mathcal{L}_{K_{\epsilon,\eta}},\mathcal{L}^\star)\leq C_0\epsilon$, where $K_{\epsilon,\eta}$ is defined as in Theorem \ref{conv_AIQL}.
And the total number of samples $T=\sum_{k=0}^{K_{\epsilon,\eta} -1}T_{\mathcal{M}_{\mathcal{L}_k}}(\epsilon_k,\delta_k)$ is bounded by
\[
T\leq O\left(K_{\epsilon,\eta}^{\frac{4}{h}+1}\left(\log\frac{K_{\epsilon,\eta}}{\delta}\right)^{\frac{1}{h}}+\left(\log \frac{K_{\epsilon,\eta}}{\delta}\right)^{\frac{1}{1-h}}\right).
\]
\end{corollary}
}

\subsection{Policy-based algorithms}\label{policy_based} In addition to algorithms with value-based guarantees (\textit{cf.} Definition \ref{value_based_gua}), there are also numerous algorithms with \textit{policy-based guarantees}.
\begin{definition}[Policy-based Guarantee]\label{policy_based_gua}
For an arbitrary MDP $\mathcal{M}$, we say that an algorithm has a \textit{policy-based guarantee}  with parameters $\{C_{\mathcal{M}}^{(i)},\alpha_1^{(i)},\alpha_2^{(i)},\alpha_3^{(i)},\alpha_4^{(i)}\}_{i=1}^m$, if for any $\epsilon,~\delta>0$, after obtaining 
\begin{equation}\label{sample_typebound} 
T_{\mathcal{M}}(\epsilon,\delta)=\sum_{i=1}^mC_{\mathcal{M}}^{(i)}\left(\frac{1}{\epsilon}\right)^{\alpha_1^{(i)}}\left(\log\frac{1}{\epsilon}\right)^{\alpha_2^{(i)}}\left(\frac{1}{\delta}\right)^{\alpha_3^{(i)}}\left(\log\frac{1}{\delta}\right)^{\alpha_4^{(i)}}
\end{equation} samples from the simulator oracle $\mathcal{G}$, with probability at least $1-2\delta$, it outputs an approximate policy $\pi_{T_{\mathcal{M}}(\epsilon,\delta)}$, which satisfies $V_{\mathcal{M}}^\star(s)-V_{\mathcal{M}}^{\pi_{T_{\mathcal{M}}(\epsilon,\delta)}}(s)\leq \epsilon$, $\forall\,s\in\mathcal{S}$. 
\end{definition}

\subsubsection{GMF-P}
Before we present policy-based RL algorithms, let us first establish a connection between policy-based and value-based guarantees. 


{\color{black} To start, take any policy $\pi\in\Pi$, consider the following synchronous temporal difference (TD) iterations:
\begin{equation}\label{TD_update}
    \tilde{Q}_{\pi}^{l+1}(s,a) = (1-\beta_l)\tilde{Q}_{\pi}^l(s,a)+\beta_l\left[r(s,a)+\gamma \tilde{Q}_{\pi}^l(s',a')\right],\quad \forall s\in\mathcal{S},\,a\in\mathcal{A},
\end{equation}
where $a'\sim \pi(s')$, $\tilde{Q}_{\pi}^0(s,a)=C$ for some constant $C\in\mathbb{R}$ and any $s\in\mathcal{S}$ and $a\in\mathcal{A}$, and the step size $\beta_l=(l+1)^{-h}$ for some $h\in(1/2,1)$. 


Then we have
\begin{lemma}\label{conv_2to1}
Suppose that the algorithm \textit{Alg} satisfies a policy-based guarantee with parameters $\{C_{\mathcal{M}}^{(i)},\alpha_1^{(i)},\alpha_2^{(i)},\alpha_3^{(i)},\alpha_4^{(i)}\}_{i=1}^m$.
Let $\tilde{Q}_{\pi}^l$ be defined by \eqref{TD_update}. Then for any $\delta\in(0,1)$ and $\epsilon>0$, 
with probability at least $1-2\delta$,
$\left\|\tilde{Q}_{\pi_{T_{\mathcal{M}}(\epsilon,\delta/2)}}^l-Q_{\mathcal{M}}^\star\right\|_{\infty}\leq \epsilon$ if 
\begin{equation}\label{TD_l_lb}
l=\Omega\left(\left(\frac{V_{\max}\log\left(\frac{|\mathcal{S}||\mathcal{A}|V_{\max}}{\delta\beta^2\epsilon}\right)}{\beta^4\epsilon^2}\right)^{1/h}+\left(\frac{1}{\beta}\log\frac{V_{\max}}{\beta\epsilon}\right)^{1/(1-h)}\right),
\end{equation}
where $V_{\max}=R_{\max}/(1-\gamma)$ and  $\beta=(1-\gamma)/2$.

Consequently, the algorithm \textit{Alg} (combined with TD updates \eqref{TD_update}) also has a value-based guarantee with parameters $\{\tilde{C}_{\mathcal{M}}^{(i)},\alpha_1^{(i)},\alpha_2^{(i)},\alpha_3^{(i)},\alpha_4^{(i)}\}_{i=1}^{m+3}$, where $\tilde{C}_{\mathcal{M}}^{(i)}$ is some constant multiple of $C_{\mathcal{M}}^{(i)}$ ($i=1,\dots,m$), $\tilde{C}_{\mathcal{M}}^{(m+i)}$ ($i=1,2,3$) are constants depending on $V_{\max}$, $|\mathcal{S}|$, $|\mathcal{A}|$, $\beta$ and $h$, and we have 
\begin{equation}\label{param_policy2value}
\begin{split}
&\text{$\alpha_1
^{(m+1)}=2/h$, $\alpha_2^{(m+1)}=1/h$, $\alpha_3^{(m+1)}=\alpha_4^{(m+1)}=0$;}\\
&\text{$\alpha_1
^{(m+2)}=2/h$, $\alpha_2^{(m+2)}=\alpha_3^{(m+2)}=0$ and $\alpha_4^{(m+2)}=1/h$;}\\
&\text{$\alpha_1
^{(m+3)}=0$, $\alpha_2^{(m+3)}=1/(1-h)$, $\alpha_3^{(m+3)}=0$ and $\alpha_4^{(m+3)}=0$.}
\end{split}
\end{equation}
\end{lemma}
}

The above lemma indicates that any algorithm with a policy-based guarantee also satisfies a value-based guarantee with similar parameters {\color{black}(when combined with the TD updates)}. The policy-based algorithm GMF-P (Algorithm \ref{GMF-P}) makes use of Lemma \ref{conv_2to1} {\color{black}to select the hyper-parameter  $l$ so that the resulting
{\color{black}$\tilde{Q}_{\pi_{T_{\mathcal{M}}(\epsilon,\delta/2)}}^l$} 
forms a good value-based certificate.  

}

{\color{black}

{\color{black}
\begin{algorithm}[h]
  \caption{\textbf{GMF-P(\textit{Alg}, ~$f_{c,c'}$)}}
  \label{GMF-P}
\begin{algorithmic}[1]
  \STATE \textbf{Input}: Initial $\mathcal{L}_0$, $\epsilon$-net $S_{\epsilon}$, temperatures $c,~c'>0$, tolerances $\epsilon_k,~\delta_k>0$, $k=0,1,\dots$. 
 \FOR {$k=0, 1, \cdots$}
  \STATE 
  \begin{addmargin}[1em]{0em}
  Apply \textit{Alg} to  find the approximate policy  $\hat{\pi}_k=\pi_{T_k}$  of the MDP $\mathcal{M}_k:=\mathcal{M}_{\mathcal{L}_k}$, where $T_k=T_{\mathcal{M}_{k}}(\epsilon_k,\delta_k/2)$. 
 \end{addmargin}\vspace{0.3em}
  \STATE 
  \begin{addmargin}[1em]{0em}
  Compute $\tilde{Q}_{\hat{\pi}_k}^{l_k}$ using TD updates \eqref{TD_update} for MDP $\mathcal{M}_k$, with $l_k$ satisfying \eqref{TD_l_lb} (with $\epsilon$ and $\delta$ replaced by $\epsilon_k$ and $\delta_k/2$, respectively).   
  \end{addmargin}
  \STATE 
    \begin{addmargin}[1em]{0em}
  Compute  $\pi_k(s)=f_{c,c'}(\tilde{Q}_{\hat{\pi}_k}^{l_k}(s,\cdot))$.
  \end{addmargin}
  \STATE 
    \begin{addmargin}[1em]{0em}
  Sample $s\sim \mu_k$ ($\mu_k$ is the population state marginal of $\mathcal{L}_k$), obtain $\tilde{\mathcal{L}}_{k+1}$ from $\mathcal{G}(s,\pi_k,\mathcal{L}_k)$.
  \end{addmargin}
  \STATE 
    \begin{addmargin}[1em]{0em}
  Find $\mathcal{L}_{k+1}=\textbf{Proj}_{S_{\epsilon}}(\tilde{\mathcal{L}}_{k+1})$
  \end{addmargin}
\ENDFOR
\end{algorithmic}
\end{algorithm}
}


We next present the convergence property for the GMF-P algorithm by combining the proofs of Lemma \ref{conv_2to1} and Theorem \ref{conv_AIQL}. 

\begin{theorem}[Convergence and complexity of GMF-P]\label{conv_AIQL-II}
Assume the same assumptions as in Theorem \ref{thm1_stat}, and in addition that $f_{c,c'}\subseteq \mathcal{F}_{c,c'}$. Suppose that \textit{Alg} has a policy-based guarantee with parameters 
\[
\{C_{\mathcal{M}}^{(i)},\alpha_1^{(i)},\alpha_2^{(i)},\alpha_3^{(i)},\alpha_4^{(i)}\}_{i=1}^m.
\] 
Then for any $\epsilon,~\delta>0$, set $\delta_k=\delta/ K_{\epsilon,\eta}$, $\epsilon_k=(k+1)^{-(1+\eta)}$ for some $\eta\in(0,1]$ $(k=0,\dots,K_{\epsilon,\eta}-1)$, and  {\color{black}$c\geq c'\geq \frac{\log(1/\epsilon)}{\phi(\epsilon)}$},
with probability at least $1-2\delta$, $$W_1(\mathcal{L}_{K_{\epsilon,\eta}},\mathcal{L}^\star)\leq C_0\epsilon.$$   Here $K_{\epsilon,\eta}:=\left\lceil 2\max\left\{(\eta\epsilon/c)^{-1/\eta},\log_d(\epsilon/\max\{\text{diam}(\mathcal{S})\text{diam}(\mathcal{A}),c\})+1\right\}\right\rceil$ is the number of outer iterations, and the constant $C_0$ is independent of $\delta$, $\epsilon$ and $\eta$. \\
Moreover,  the total number of {\color{black}samples} $T=\sum_{k=0}^{K_{\epsilon,\eta} -1}T_{\mathcal{M}_{\mathcal{L}_k}}(\delta_k,\epsilon_k)$ is bounded by
\begin{equation}\label{Tbound-II}
T\leq \sum_{i=1}^{{\color{black}m+3}}\dfrac{2^{\alpha_2^{(i)}}}{2\alpha_1^{(i)}+1}\tilde{C}_{\mathcal{M}}^{(i)}K_{\epsilon,\eta}^{2\alpha_1^{(i)}+1}(K_{\epsilon,\eta}/\delta)^{\alpha_3^{(i)}}\left(\log(K_{\epsilon,\eta}/\delta)\right)^{\alpha_2^{(i)}+\alpha_4^{(i)}},
\end{equation}
 where the parameters $\{\tilde{C}_{\mathcal{M}}^{(i)},\alpha_1^{(i)},\alpha_2^{(i)},\alpha_3^{(i)},\alpha_4^{(i)}\}_{i=1}^{m+3}$ are defined in Lemma \ref{conv_2to1}.
\end{theorem}

}

{\color{black}

\subsubsection{GMF-P-TRPO: GMF-P with TRPO}
A special form of the GMF-P algorithm utilizes the trust region policy optimization (TRPO) algorithm \citet{schulman2015trust, SEM2019}. We call it GMF-P-TRPO.

{\color{black} Sample-based TRPO \citet{SEM2019} assumes access to a $\nu$-restart model. That is, it can only access sampled trajectories and restarts according to the distribution $\nu$. Here we pick $\nu$ such that  $C^{\pi^\star}:=\left\|\frac{d_{\text{Unif}_{\mathcal{S}}}^{\pi^\star}}{\nu}\right\|_{\infty}=\max_{s \in \mathcal{S}}\left| \frac{d_{\text{Unif}_{\mathcal{S}}}^{\pi^\star}(s)}{\nu(s)}\right|<\infty$, where $d_{\rho}^{\pi}=(1-\gamma)\rho (I-\gamma P^{\pi})^{-1}$ and  $\text{Unif}_{\mathcal{S}}$ is the uniform distribution on set $\mathcal{S}$. } Sample-based TRPO samples $M_0$ trajectories per
episode. { \color{black}The initial state $s_0$ at the beginning of each episode is sampled from $\nu$. In every trajectory $m$ ($m=1,2,\cdots,M_0$) of the $l$-th episode, it first samples $s_m \sim d_{\nu}^{\pi_l}$ and takes an action $a_m \sim \text{Unif}_{\mathcal{A}}$ where
$\text{Unif}_{\mathcal{A}}$ is the uniform distribution on the set $\mathcal{A}$.
}
Then, by following the current $\pi_l$, it estimates $Q^{\pi_l}(s_m,a_m)$ using a rollout. Denote this estimate as $\hat{Q}^{\pi_l}(s_m,a_m,m)$ and observe that it is (nearly) an unbiased estimator of $Q^{\pi_l}(s_m,a_m)$. We assume that each rollout runs sufficiently long so that the bias is sufficiently small. Sample-Based TRPO updates the policy at the end of the $l$-th episode, by the following proximal problem
\[
\pi_{l+1} \in \arg {\color{black}\max_{\pi \in \Delta_{\mathcal{A}}^{|\mathcal{S}|}}} \left\{\frac{1}{M_0} \sum_{m=1}^{M_0} \frac{1}{t_l(1-\gamma)}B_{w}(s_m;\pi,\pi_l) +\langle \hat{\nabla}V^{\pi_l}[m],\pi(s_m)-\pi_l(s_m)\rangle \right\},
\]
where the estimation of the gradient is
{\color{black}
\[
\hat{\nabla}V^{\pi_l}[m]:=\frac{1}{1-\gamma}|\mathcal{A}| \hat{Q}^{\pi_l}(s_m,\cdot,m)\circ\textbf{I}_{\{\cdot=a_m\}}.
\]
}
{\color{black} 
Given two policies $\pi_1$ and $\pi_2$, we denote their Bregman distance associated with a strongly convex function $w$ as
 $B_w(s;\pi_1,\pi_2) =B_{w}(\pi_1(s),\pi_2(s))$, where
 $B_{w}(x,y):=w(x)-w(y)-\langle \nabla w(y),x-y \rangle$ and $\pi_i(s)\in P(\mathcal{A})$  $(i=1,2)$. 
Denote $B_w(\pi_1,\pi_2)\in \mathbb{R}^{|\mathcal{S}|}$ as the corresponding state-wise vector. Here we consider two common cases for $w$: when $w(x) = \frac{1}{2}\|x\|_2^2$ is the Euclidean distance, $B_w(x,y)=\frac{1}{2}\|x-y\|_2^2$; when $w(x) = H(x)$ is the negative entropy, $B_w(x,y)=d_{\text{KL}}(x||y)$.} {\color{black}We refer to \cite[Section 6.2]{SEM2019} for more detailed discussion on Sample-based TRPO.}

The above guarantee follows from the sample complexity result below {\color{black}by specifying $\mu := \text{Unif}_{\mathcal{S}}$.} {\color{black}Notice that here for any $\mu\in\mathcal{P}(\mathcal{S})$, we define $V^\star(\mu):=\sum_{s\in\mathcal{S}}\mu(s) V^\star(s)$, and similarly $V^{\pi_k}(\mu):=\sum_{s\in\mathcal{S}}\mu(s) V^{\pi_k}(s)$.} 

The sample complexity of TRPO algorithm can be characterized as below.

\begin{lemma}[Theorem 5 in \citet{SEM2019}: sample complexity of TRPO]\label{sample_comp_trpo}
Let $\{\pi_l\}_{l \ge 0}$ be the sequence generated by Sample-Based TRPO, using $$M_0 \ge \Omega(\frac{|\mathcal{A}|^2 {\color{black}C^2}(|\mathcal{S}|\log|\mathcal{A}|+\log 1/\delta)}{(1-\gamma)^2\epsilon^2})$$ samples in each episode, with $t_l = \frac{(1-\gamma)}{C_{\omega,1}C\sqrt{l+1}}$.  Let $\{V^N_{\text{best}}\}_{N \ge 0}$ be the sequence of best achieved values,
$V^N_{\text{best}}(\mu):= \max_{l=0,1,\cdots,N} V^{\pi_l}(\mu)$, where $\mu\in\mathcal{P}(\mathcal{S})$. Then with probability greater than $1-\delta$ for every $\epsilon>0$, the following holds for all $N\geq 1$:
$$ V^\star(\mu)-V^N_{\text{best}}(\mu) \leq O\left(\frac{C_{\omega,1}C}{(1-\gamma)^2\sqrt{N}}+\frac{C^{\pi^\star}\epsilon}{(1-\gamma)^2}\right).$$

\end{lemma}
{\color{black}Here $C>0$ is the upper bound on the reward function $r$, $C_{w,1}=\sqrt{|\mathcal{A}|}$ in the euclidean case and $C_{w,1}=1$ in the non-euclidean case, $C_{w,2}=1$ for the euclidean case and $C_{w,2}=|\mathcal{A}|^2$ for the non-euclidean case. } {\color{black}Note that unlike the case of Q-learning, here we are only guaranteed to have \emph{some} iterate among iterations $0,\dots,N$ that satisfy the desired sub-optimality bound. Note that this is a common pattern of the theoretical results for policy optimization algorithms in the RL literature \citet{agarwal2021theory, wang2019neural}, unless the (oracle) access to exact policy gradients is assumed \citet{mei2020global}. For simplicity, hereafter we assume an oracle access to such an iterate after running TRPO. In practice, with additional (polynomial number of) samples, one can explicitly identify a single policy satisfying the desired bound with high probability; see \textit{e.g.}, the two-phase technique in \citet{ghadimi2013stochastic}.}

{\color{black}Note that \cite[Theorem 5]{SEM2019} has both regularized version and unregularized version of TRPO. Here we only adopt the unregularized version which fits the framework of Algorithm \ref{GMF-P}. For more materials on regularized MDPs and reinforcement learning, we refer the readers to \citet{neu2017unified, geist2019theory, derman2020distributional}.}

{\color{black}
Based on the sample complexity in Lemma \ref{sample_comp_trpo}, the following policy-based guarantee for TRPO algorithm and the convergence result for GMF-P-TRPO can be obtained.
\begin{corollary}
Let $t_l = \frac{(1-\gamma)}{C_{\omega,1}C\sqrt{l+1}}$, then TRPO algorithm satisfies the policy-based guarantee with parameters $\{\tilde{C}_{\mathcal{M}}^{(i)},\alpha_1^{(i)},\alpha_2^{(i)},\alpha_3^{(i)},\alpha_4^{(i)}\}_{i=1}^{2}$, where $C_{\mathcal{M}}^{(i)}(i=1,2)$ are constants depending on  $|\mathcal{S}|,|\mathcal{A}|,V_{\max},\beta$ and $h$, and we have:
\begin{eqnarray*}
\alpha_1^{(1)} = 5/2, \,\,\alpha_j^{(1)}=0 \,\,\mbox{for} \,\,j=2,3,4,\\
\alpha_1^{(2)} = 5/2,\,\,\alpha_4^{(2)} = 1,\alpha_2^{(2)}=\alpha_3^{(2)}=0.
\end{eqnarray*}
In addition, under same assumptions as Theorem \ref{thm1_stat}, then for Algorithm \ref{GMF-P} using TRPO method, with probability at least $1-2\delta$, $W_1(\mathcal{L}_{K_{\epsilon,\eta}},\mathcal{L}^\star)\leq C_0\epsilon$, where $K_{\epsilon,\eta}$ is defined as in Theorem \ref{conv_AIQL}.
And the total number of samples $T=\sum_{k=0}^{K_{\epsilon,\eta} -1}T_{\mathcal{M}_{\mathcal{L}_k}}(\delta_k,\epsilon_k)$ is bounded by
\[
T\leq O\left(K_{\epsilon,\eta}^6\left(\log\frac{K_{\epsilon,\eta}}{\delta}\right) + K_{\epsilon,\eta}^{\frac{4}{h}+1}\left(\log\frac{K_{\epsilon,\eta}}{\delta}\right)^{\frac{1}{h}}+\left(\log \frac{K_{\epsilon,\eta}}{\delta}\right)^{\frac{1}{1-h}}\right).
\]
\end{corollary}
}

}

\section{{\color{black}Applications to $N$-player Games}}\label{appl_n_game}
{\color{black} 
In this section, we discuss a potential application of our modeling and approach to $N$-player settings. To this end, we consider extensions of Algorithms \ref{GMF-V} and \ref{GMF-P} with weaker assumptions on the simulator access. In particular, we weaken the simulator oracle assumption in Section \ref{AIQL} as follows.
\paragraph{Weak simulator oracle.} For each player $i$, given any policy $\pi\in\Pi$, the current state $s_i\in\mathcal{S}$, for any empirical population state-action distribution $\mathcal{L}_N$, one can obtain a \textit{sample} of the next state $s_i'\sim P_{\mathcal{L}_N}(\cdot|s_i,\pi(s_i))=P(\cdot|s_i,\pi(s_i),\mathcal{L}_N)$ and a reward $r= r_{\mathcal{L}_N}(s_i,\pi(s_i))=r(s_i,\pi(s_i),\mathcal{L}_N)$.  For brevity, we denote the simulator as $(s_i',r)=\mathcal{G}_W(s_i,\pi,\mathcal{L}_N)$. 

We say that $\mathcal{L}_N$ is an empirical population state-action distribution of $N$-players if for each $s\in\mathcal{S},~a\in\mathcal{A}$, $\mathcal{L}_N(s,a)=\frac{1}{N}\sum_{i=1}^N\textbf{I}_{s_i=s,a_i=a}$ for some state-action profile of $\{s_i,a_i\}_{i=1}^N$. Equivalently, this holds if $N\mathcal{L}_N(s,a)$ is a non-negative integer for each $s\in\mathcal{S},~a\in\mathcal{A}$, and $\sum_{s,a}\mathcal{L}_N(s,a)=1$. We denote the set of empirical population state-action distributions as {\color{black}$\textbf{Emp}_{N}$}. 

\paragraph{RL algorithms with access only to $\mathcal{G}_W$.}

Compared to the original simulator oracle $\mathcal{G}$, the weak simulator $\mathcal{G}_W$ only accepts empirical population state-action distributions as inputs, and does not directly output the next (empirical) population state-action distribution.

To make use of the simulator $\mathcal{G}_W$, we modify Algorithm \ref{GMF-V} and Algorithm \ref{GMF-P} to  algorithms (Algorithms \ref{GMF-VW} and \ref{GMF-PW}). {\color{black}In particular, see Step 6 in Algorithm \ref{GMF-VW} and Step 7 in Algorithm \ref{GMF-PW} for generating empirical distributions from simulator $\mathcal{G}_W$.}

\begin{algorithm}[h]
  \caption{\textbf{GMF-VW(\textit{Alg}, ~$f_{c,c'}$)}: weak simulator}
  \label{GMF-VW}
\begin{algorithmic}[1]
  \STATE \textbf{Input}: Initial $\mathcal{L}_0$, temperatures $c,~c'>0$, tolerances $\epsilon_k,~\delta_k>0$, $k=0,1,\dots$. 
 \FOR {$k=0, 1, \cdots$}
  \STATE 
  \begin{addmargin}[1em]{0em}
  Apply \textit{Alg} to  find the approximate Q-function $\hat{Q}_k^\star=\hat{Q}^{T_k}$ of the MDP $\mathcal{M}_{\mathcal{L}_k}$, where $T_k=T_{\mathcal{M}_{\mathcal{L}_k}}(\epsilon_k,\delta_k)$.
 \end{addmargin}
  \STATE 
    \begin{addmargin}[1em]{0em}
  Compute  $\pi_k(s)=f_{c,c'}(\hat{Q}^\star_k(s,\cdot))$.
  \end{addmargin}\vspace{-0.3cm}
    \begin{addmargin}[1em]{0em}
     \FOR {$i=1, 2, \cdots,N$}
 \STATE
 \begin{addmargin}[1em]{0em}
 Sample $s_i\overset{\text{i.i.d.}}{\sim} \mu_k$,  then obtain $s_i'$ i.i.d. from $\mathcal{G}_W(s_i,\pi_k,\mathcal{L}_k)$ and $a_i'\overset{\text{i.i.d.}}{\sim}\pi_k(s_i')$. 
 \end{addmargin}\vspace{0.3em}
 \ENDFOR
 \STATE
 Compute $\mathcal{L}_{k+1}$ with $\mathcal{L}_{k+1}(s,a)=\frac{1}{N}\sum_{i=1}^N\textbf{I}_{s_i'=s,a_i'=a}$.
  \end{addmargin}
\ENDFOR
\end{algorithmic}
\end{algorithm}

\begin{algorithm}[h]
  \caption{\textbf{GMF-PW(\textit{Alg}, ~$f_{c,c'}$)}: weak simulator}
  \label{GMF-PW}
\begin{algorithmic}[1]
  \STATE \textbf{Input}: {\color{black}Initial $\mathcal{L}_0$,  temperatures $c,~c'>0$, tolerances $\epsilon_k,~\delta_k>0$, $k=0,1,\dots$.} 
 \FOR {$k=0, 1, \cdots$}
  \STATE 
  \begin{addmargin}[1em]{0em}
  Apply \textit{Alg} to  find the approximate policy  $\hat{\pi}_k=\pi_{T_k}$  of the MDP $\mathcal{M}_k:=\mathcal{M}_{\mathcal{L}_k}$, where $T_k=T_{\mathcal{M}_{\mathcal{L}_k}}(\epsilon_k,\delta_k/2)$. 
 \end{addmargin}\vspace{0.3em}
  \STATE 
  \begin{addmargin}[1em]{0em}
  Compute $\tilde{Q}_{\hat{\pi}_k}^{l_k}$ using TD updates \eqref{TD_update} for MDP $\mathcal{M}_k$, with $l_k$ satisfying \eqref{TD_l_lb} (with $\epsilon$ and $\delta$ replaced by $\epsilon_k$ and $\delta_k/2$, respectively).
  \end{addmargin}
  \STATE 
    \begin{addmargin}[1em]{0em}
  Compute  $\pi_k(s)=f_{c,c'}(\hat{Q}_{\mathcal{M}_k,M_k,l_k}^{\hat{\pi}_k}(s,\cdot))$.
  \end{addmargin}\vspace{-0.6em}
      \begin{addmargin}[1em]{0em}
     \FOR {$i=1, 2, \cdots,N$}
 \STATE
 \begin{addmargin}[1em]{0em}
 Sample $s_i\overset{\text{i.i.d.}}{\sim} \mu_k$,  then obtain $s_i'$ i.i.d. from $\mathcal{G}_W(s_i,\pi_k,\mathcal{L}_k)$ and $a_i'\overset{\text{i.i.d.}}{\sim}\pi_k(s_i')$.
 \end{addmargin}\vspace{0.3em}
 \ENDFOR
 \STATE 
 Compute $\mathcal{L}_{k+1}$ with $\mathcal{L}_{k+1}(s,a)=\frac{1}{N}\sum_{i=1}^N\textbf{I}_{s_i'=s,a_i'=a}$.
  \end{addmargin}
\ENDFOR
\end{algorithmic}
\end{algorithm}

One can observe that {\color{black}$\textbf{Emp}_{N}$} already serves as an $1/N$-net. So one can directly use it without additional projections. 
The definition of $\mathcal{L}_k$ also makes sure that $\mathcal{L}_k\in{\color{black}\textbf{Emp}_{N}}$ as required for the input of the weaker simulator. 


{\color{black}Convergence results similar to Theorems \ref{conv_AIQL} and \ref{conv_AIQL-II} can be obtained for Algorithms \ref{GMF-VW} and \ref{GMF-PW}, respectively.} (See Appendix \ref{sec:weak}.) Here the major difference is an additional $O(1/\sqrt{N})$ term in the finite step error bound. {\color{black} It is worth mentioning that $O(1/\sqrt{N})$ is consistent with the literature on MFG approximation errors of finite $N$-player games \citet{HMC2006}.}
}
\section{Proof of the main results}
\label{sec:proof}
\subsection{Proof of Lemma \ref{assumption2_exp}}
In this section, we provide the proof of Lemma \ref{assumption2_exp}.
\begin{proof}{[Proof of Lemma \ref{assumption2_exp}]}
We begin by noticing that $\mathcal{L}'=\Gamma_2(\pi,\mathcal{L})$ can be expanded and computed as follows:
\begin{equation}
\mu'(s')=\sum\nolimits_{s\in\mathcal{S},a\in\mathcal{A}}\mu(s)P(s'|s,a,\mathcal{L})\pi(a|s),\quad \mathcal{L}'(s',a')=\mu'(s')\pi(a'|s'),
\end{equation}
where $\mu$ is the state marginal distribution of $\mathcal{L}$.

Since the Wasserstein distance $W_1$ can be related to the total variation distance via the following inequalities \citet{metrics_prob}: 
\begin{equation}\label{W1_tv}
d_{\min}(\mathcal{X})d_{TV}(\nu,\nu')\leq W_1(\nu,\nu')\leq \text{diam}(\mathcal{X})d_{TV}(\nu,\nu'),
\end{equation}
where $d_{\min}(\mathcal{X})=\min_{x\neq y\in\mathcal{X}}\|x-y\|_2$, we have 

\begin{equation}
\begin{split}
W_1&(\Gamma_2(\pi_1,\mathcal{L}),\Gamma_2(\pi_2,\mathcal{L}))\leq \text{diam}(\mathcal{S}\times\mathcal{A})d_{TV}(\Gamma_2(\pi_1,\mathcal{L}),\Gamma_2(\pi_2,\mathcal{L}))\\
=&\dfrac{\text{diam}(\mathcal{S}\times\mathcal{A})}{2}\sum_{s'\in\mathcal{S},a'\in\mathcal{A}}\left|\sum_{s\in\mathcal{S},a\in\mathcal{A}}\mu(s)P(s'|s,a,\mathcal{L})\left(\pi_1(a|s)\pi_1(a'|s')-\pi_2(a|s)\pi_2(a'|s')\right)\right|\\
\leq& \dfrac{\text{diam}(\mathcal{S}\times\mathcal{A})}{2}\max_{s,a,\mathcal{L},s'}P(s'|s,a,\mathcal{L})\sum_{s,a,s',a'}\mu(s)(\pi_1(a|s)+\pi_2(a|s))|\pi_1(a'|s')-\pi_2(a'|s')|\\
\leq& \dfrac{\text{diam}(\mathcal{S}\times\mathcal{A})}{2}\max_{s,a,\mathcal{L},s'}P(s'|s,a,\mathcal{L})\sum_{s',a'}|\pi_1(a'|s')-\pi_2(a'|s')|\cdot (1+1)\\
=&2\text{diam}(\mathcal{S}\times\mathcal{A})\max_{s,a,\mathcal{L},s'}P(s'|s,a,\mathcal{L})\sum_{s'}d_{TV}(\pi_1(s'),\pi_2(s'))\\
\leq & \frac{2\text{diam}(\mathcal{S}\times\mathcal{A})\max_{s,a,\mathcal{L},s'}P(s'|s,a,\mathcal{L})|\mathcal{S}|}{d_{\min}(\mathcal{A})}D(\pi_1,\pi_2)=  \frac{2\text{diam}(\mathcal{S})\text{diam}(\mathcal{A})|\mathcal{S}|c_1}{d_{\min}(\mathcal{A})}D(\pi_1,\pi_2).
\end{split}
\end{equation}

Similarly, we have
\begin{equation}
\begin{split}
W_1&(\Gamma_2(\pi,\mathcal{L}_1),\Gamma_2(\pi,\mathcal{L}_2))\leq \text{diam}(\mathcal{S}\times\mathcal{A})d_{TV}(\Gamma_2(\pi,\mathcal{L}_1),\Gamma_2(\pi,\mathcal{L}_2))\\
=&\dfrac{\text{diam}(\mathcal{S}\times\mathcal{A})}{2}\sum_{s'\in\mathcal{S},a'\in\mathcal{A}}\left|\sum_{s\in\mathcal{S},a\in\mathcal{A}}\pi(a|s)\pi(a'|s')\left(\mu_1(s)P(s'|s,a,\mathcal{L}_1)-\mu_2(s)P(s'|s,a,\mathcal{L}_2)\right)\right|\\
\leq&\dfrac{\text{diam}(\mathcal{S}\times\mathcal{A})}{2}\sum_{s'\in\mathcal{S},a'\in\mathcal{A}}\left|\sum_{s\in\mathcal{S},a\in\mathcal{A}}\pi(a|s)\pi(a'|s')\mu_1(s)\left(P(s'|s,a,\mathcal{L}_1)-P(s'|s,a,\mathcal{L}_2)\right)\right|\\
&+\dfrac{\text{diam}(\mathcal{S}\times\mathcal{A})}{2}\sum_{s'\in\mathcal{S},a'\in\mathcal{A}}\left|\sum_{s\in\mathcal{S},a\in\mathcal{A}}\pi(a|s)\pi(a'|s')(\mu_1(s)-\mu_2(s))P(s'|s,a,\mathcal{L}_2)\right|\\
\leq & \dfrac{\text{diam}(\mathcal{S}\times\mathcal{A})}{2}\sum_{s,a,s',a'}\mu_1(s)\pi(a|s)\pi(a'|s')\left|P(s'|s,a,\mathcal{L}_1)-P(s'|s,a,\mathcal{L}_2)\right|\\
&+  \dfrac{\text{diam}(\mathcal{S}\times\mathcal{A})}{2}\sum_{s,a,s',a'}|\mu_1(s)-\mu_2(s)|\pi(a|s)\pi(a'|s')P(s'|s,a,\mathcal{L}_2)\\
\leq & \frac{\text{diam}(\mathcal{S})\text{diam}(\mathcal{A})|\mathcal{S}|c_2}{2}W_1(\mathcal{L}_1,\mathcal{L}_2) + \frac{\text{diam}(\mathcal{S})\text{diam}(\mathcal{A})|\mathcal{S}|}{2}2d_{TV}(\mathcal{L}_1,\mathcal{L}_2)\\
\leq & \text{diam}(\mathcal{S})\text{diam}(\mathcal{A})|\mathcal{S}|\left(\frac{c_2}{2}+\frac{1}{d_{\min}(\mathcal{S}\times\mathcal{A})}\right)W_1(\mathcal{L}_1,\mathcal{L}_2).
\end{split}
\end{equation}
Here $\mu_1$ and $\mu_2$ are the state marginals of $\mathcal{L}_1$ and $\mathcal{L}_2$, respectively. 

This completes {\color{black}the} proof.
\qed
\end{proof}

{\color{black}\subsection{Proof of Lemma \ref{conv_2to1}}\label{conv_2tol_proof}
For notation simplicity, in the following analysis we fix the MDP and omit the notation $\mathcal{M}$.

We begin by establishing the convergence rate of the synchronous TD updates \eqref{TD_update}. 
\begin{lemma}\label{TD_convergence}
Take $\tilde{Q}_{\pi}^l$ from \eqref{TD_update}. Then for any $\delta\in(0,1)$ and $\epsilon>0$, 
with probability at least $1-\delta$, 
$\|\tilde{Q}_{\pi_{T(\epsilon,\delta)}}^l-Q^{\pi_{T(\epsilon,\delta)}}\|_{\infty}\leq \epsilon$ if
\begin{equation}
l=\Omega\left(\left(\frac{V_{\max}\log\left(\frac{|\mathcal{S}||\mathcal{A}|V_{\max}}{\delta\beta\epsilon}\right)}{\beta^2\epsilon^2}\right)^{1/h}+\left(\frac{1}{\beta}\log\frac{V_{\max}}{\epsilon}\right)^{1/(1-h)}\right),
\end{equation}
where $V_{\max}=R_{\max}/(1-\gamma)$ and  $\beta=(1-\gamma)/2$.
\end{lemma}
The proof is adapted from that of \cite[Theorem 2]{Q-rate}, with the $\max$ term in the Bellman operator modified to actions sampled from the current policy $\pi$. The details are omitted. 

\begin{proof}{[Proof of Lemma \ref{conv_2to1}]} 
First, if $V^\star(s')-V^{\pi_{T(\epsilon,\delta/2)}}(s')\leq \epsilon$, then
\begin{equation}\label{type2_to_type1_exact}
\begin{split}
\left|Q^{\pi_{T(\epsilon,\delta/2)}}(s,a)-Q^\star(s,a)\right|&=\gamma \left|\sum_{s'\in\mathcal{S}}P(s'|s,a)V^{\pi_{T(\epsilon,\delta/2)}}(s')-\sum_{s'\in\mathcal{S}}P(s'|s,a)V^\star(s')\right|\\
&\leq \gamma\sum_{s'\in\mathcal{S}}P(s'|s,a)\left|V^{\pi_{T(\epsilon,\delta/2)}}(s')-V^\star(s')\right|\leq \gamma\epsilon<\epsilon.
\end{split}
\end{equation}
for any $s\in\mathcal{S},~a\in\mathcal{A}$. Since \textit{Alg} is assumed to satisfying the policy-based guarantee, \eqref{type2_to_type1_exact} holds with probability at least $1-\delta$. 

In addition, by Lemma \ref{TD_convergence}, whenever $l$ satisfies \eqref{TD_l_lb},  with probability at least $1-\delta/2\geq 1-\delta$,
\begin{equation}\label{Qpi-Qstar-TD}
\|\tilde{Q}_{\pi_{T(\epsilon,\delta/2)}}^l-Q^{\pi_{T(\epsilon,\delta/2)}}\|_{\infty}\leq (1-\gamma)\epsilon.
\end{equation}

Combining \eqref{type2_to_type1_exact} and \eqref{Qpi-Qstar-TD}, then for any $l$ satisfying \eqref{TD_l_lb}, with probability at least $1-2\delta$, we have 
\[
\|\tilde{Q}_{\pi_{T(\epsilon,\delta/2)}}^l-Q^\star\|_\infty\leq \gamma\epsilon + (1-\gamma)\epsilon = \epsilon.
\]

The above result shows that for any $\delta\in(0,1)$ and $\epsilon>0$, after obtaining $T(\epsilon,\delta/2)+|\mathcal{S}||\mathcal{A}|l$ samples (with $l$ satisfying the lower bound \eqref{TD_l_lb}) from the simulator, with probability at least $1-2\delta$, it outputs an approximate $Q$-function $\tilde{Q}_{\pi_{T(\epsilon,\delta/2)}}^l$ which satisfies $\|\tilde{Q}_{\pi_{T(\epsilon,\delta/2)}}^l-Q^\star\|_\infty\leq \epsilon$.
Thus \textit{Alg} also has a value-based guarantee with parameters
\begin{equation}
\{\tilde{C}_{\mathcal{M}}^{(i)},\alpha_1^{(i)},\alpha_2^{(i)},\alpha_3^{(i)},\alpha_4^{(i)}\}_{i=1}^{m+3},
\end{equation}
specified in \eqref{param_policy2value}. Here the first $m$ groups of parameters come from $T(\epsilon,\delta/2)$ while the last three groups of parameters come from $|\mathcal{S}||\mathcal{A}|l$ (with the lower bound \eqref{TD_l_lb} of $l$ plugged in here).  
\qed
\end{proof}
}

\subsection{Proof of $\mathcal{B}_{c,c'}\subseteq \mathcal{F}_{c,c'}$} \label{bcc-fcc}


{\color{black}
\begin{lemma}\label{softmax-lip}
Suppose that $h:\mathbb{R}\rightarrow\mathbb{R}$ satisfies $h(a)-h(b)\leq c(a-b)$ for any $a\geq b\in\mathbb{R}$. Then the softmax function $\emph{\textbf{softmax}}_h$ is $c$-Lipschitz, \textit{i.e.}, $\|\emph{\textbf{softmax}}_h(x)-\emph{\textbf{softmax}}_h(y)\|_2\leq c\|x-y\|_2$ for any $x,~y\in\mathbb{R}^n$.
\end{lemma}
\begin{proof}{[Proof of Lemma \ref{softmax-lip}]}
Notice that $\textbf{softmax}_h(x) = \textbf{softmax}(\tilde{h}(x))$, where $$\textbf{softmax}(x)_i=\frac{\exp(x_i)}{\sum_{j=1}^n\exp(x_j)} (i=1,\dots,n)$$ is the standard softmax function and $\tilde{h}(x)_i=h(x_i)$ for $i=1,\dots,n$. Now since \textbf{softmax} is $1$-Lipschitz continuous (\textit{cf.}  \cite[Proposition 4]{softmax}), and $\tilde{h}$ is $c$-Lipschitz continuous, we conclude that the composition $\textbf{softmax}\circ\tilde{h}$ is $c$-Lipschitz continuous.
\qed
\end{proof}
}
\vspace{0.3em}

{\color{black} Notice that for a finite set $\mathcal{X}\subseteq\mathbb{R}^k$ and any two (discrete) distributions $\nu,~\nu'$ over $\mathcal{X}$, we have
\begin{equation}
\begin{split}
W_1(\nu,\nu')&\leq \text{diam}(\mathcal{X})d_{TV}(\nu,\nu')=\frac{\text{diam}(\mathcal{X})}{2}\|\nu-\nu'\|_1\leq \frac{\text{diam}(\mathcal{X})\sqrt{|\mathcal{X}|}}{2}\|\nu-\nu'\|_2,
\end{split}
\end{equation}
where in computing the $\ell_1$-norm, $\nu,~\nu'$ are viewed as vectors of length $|\mathcal{X}|$. 

Lemma \ref{softmax-lip} implies that for any $x,~y\in\mathbb{R}^{|\mathcal{X}|}$, when $\textbf{softmax}_c(x)$ and $\textbf{softmax}_c(y)$ are viewed as probability distributions over $\mathcal{X}$, we have 
\[
W_1(\textbf{softmax}_c(x),\textbf{softmax}_c(y))\leq \frac{\text{diam}(\mathcal{X})\sqrt{|\mathcal{X}|}c}{2}\|x-y\|_2\leq \frac{\text{diam}(\mathcal{X})|\mathcal{X}|c}{2}\|x-y\|_{\infty}.
\]}

{\color{black}
\begin{lemma}\label{soft-arg-diff}
Suppose that $h:\mathbb{R}\rightarrow\mathbb{R}$ satisfies $c'(a-b)\leq h(a)-h(b)$ for any $a\leq b\in\mathbb{R}$. 
Then for any $x\in\mathbb{R}^n$, the distance between the $\emph{\textbf{softmax}}_h$ and the \emph{\textbf{argmax-e}} mapping is bounded by
\[
\|\emph{\textbf{softmax}}_h(x)-\emph{\textbf{argmax-e}}(x)\|_2\leq 2n\exp(-c'\delta),
\]
where {\color{black}$\delta=x_{\max}-\max_{x_j<x_{\max}}x_j$, $x_{\max}=\max_{i=1,\dots,n}x_i$, and}  $\delta:=\infty$ when all $x_j$ are equal.
\end{lemma}


{\color{black}
Similar to Lemma \ref{softmax-lip}, Lemma \ref{soft-arg-diff} implies that 
for any $x\in\mathbb{R}^{|\mathcal{X}|}$, viewing $\textbf{softmax}_h(x)$  as probability distributions over $\mathcal{X}$ leads to  
\[
W_1(\textbf{softmax}_{\color{black}h}(x),\textbf{argmax-e}(x))\leq \text{diam}(\mathcal{X})|\mathcal{X}|\exp(-c\delta).
\]
}

\begin{proof}{[Proof of Lemma~\ref{soft-arg-diff}]}
Without loss of generality, assume that $x_1=x_2=\dots= x_m=\max_{i=1,\dots,n}x_i=x^\star>x_j$ for all $m<j\leq n$. Then \[
\textbf{argmax-e}(x)_i=
\begin{cases}
\frac{1}{m}, & i \leq m,\\
0, & otherwise.
\end{cases}
\]
\[
\textbf{softmax}_h(x)_i=
\begin{cases}
\frac{e^{h(x^\star)}}{me^{h(x^\star)}+\sum_{j=m+1}^ne^{h(x_j)}}, & i\leq m, \\
\frac{e^{h(x_i)}}{me^{h(x^\star)}+\sum_{j=m+1}^ne^{h(x_j)}}, & otherwise.
\end{cases}
\]
Therefore 
\begin{equation*}
\begin{split}
\|\textbf{soft}&\textbf{max}_h(x)-\textbf{argmax-e}(x)\|_2\leq \|\textbf{softmax}_h(x)-\textbf{argmax-e}(x)\|_1\\
=&m\left(\frac{1}{m}-\frac{e^{h(x^\star)}}{me^{h(x^\star)}+\sum_{j=m+1}^ne^{h(x_j)}}\right)+\frac{\sum_{i=m+1}^ne^{h(x_i)}}{me^{h(x^\star)}+\sum_{j=m+1}^ne^{h(x_j)}}\\
=& \frac{2\sum_{i=m+1}^ne^{h(x_i)}}{me^{h(x^\star)}+\sum_{i=m+1}^ne^{h(x_i)}}= \frac{2\sum_{i=m+1}^ne^{-c'\delta_i}}{m+\sum_{i=m+1}^ne^{-c\delta_i}}\\
\leq &\frac{2}{m}\sum_{i=m+1}^ne^{-c'\delta_i}\leq \frac{2(n-m)}{m}e^{-c'\delta}\leq  2ne^{-c'\delta},
\end{split}
\end{equation*}
with $\delta_i=x^\star-x^i$.
\qed
\end{proof}
}

We are now ready to present the proofs of Theorems \ref{conv_AIQL} and \ref{conv_AIQL-II}.

\subsection{Proof of Theorems \ref{conv_AIQL} and \ref{conv_AIQL-II}} \label{thm2-thm3}

\begin{proof}{[Proof of Theorem \ref{conv_AIQL}]}
{\color{black}Here we prove the case when we are using GMF-V and \textit{Alg} has a value-based guarantee. }
Define $\hat{\Gamma}_1^k(\mathcal{L}_k):=f_{c,c'}\left(\hat{Q}^\star_{k}\right)$. In the following, $\pi=f_{c,c'}(Q_{\mathcal{L}})$ is understood as the policy $\pi$ with $\pi(s)=f_{c,c'}(Q_{\mathcal{L}}(s,\cdot))$. Let $\mathcal{L}^\star$ be the population state-action pair in a stationary NE of \eqref{mfg}. Then $\pi_k=\hat{\Gamma}_1^k(\mathcal{L}_k)$.
Denoting $d:=d_1d_2+d_3$,  we see 
\begin{equation*}
    \begin{split}
        W_1(\tilde{\mathcal{L}}_{k+1}&,\mathcal{L}^\star)=W_1(\Gamma_2(\pi_k,\mathcal{L}_k),\Gamma_2(\Gamma_1(\mathcal{L}^\star),\mathcal{L}^\star))\\
        \leq& W_1(\Gamma_2(\Gamma_1(\mathcal{L}_k),\mathcal{L}_k),\Gamma_2(\Gamma_1(\mathcal{L}^\star),\mathcal{L}^\star))+W_1(\Gamma_2(\Gamma_1(\mathcal{L}_k),\mathcal{L}_k),\Gamma_2(\hat{\Gamma}_1^k(\mathcal{L}_k),\mathcal{L}_k))\\
        \leq& W_1(\Gamma(\mathcal{L}_k),\Gamma(\mathcal{L}^\star))+d_2D(\Gamma_1(\mathcal{L}_k),\hat{\Gamma}_1^k(\mathcal{L}_k))\\
        \leq & (d_1d_2+d_3)W_1(\mathcal{L}_k,\mathcal{L}^\star)+d_2D(\textbf{argmax-e}(Q_{\mathcal{L}_k}^\star),f_{c,c'}(\hat{Q}_{k}^\star))\\
        \leq& dW_1(\mathcal{L}_k,\mathcal{L}^\star)+d_2D(f_{c,c'}(\hat{Q}_{k}^\star),f_{c,c'}(Q_{\mathcal{L}_k}^\star))\\
        &+d_2D(\textbf{argmax-e}(Q_{\mathcal{L}_k}^\star),f_{c,c'}(Q_{\mathcal{L}_k}^\star))\\
        \leq & dW_1(\mathcal{L}_k,\mathcal{L}^\star)+\frac{cd_2\text{diam}(\mathcal{A})|\mathcal{A}|}{2}\|\hat{Q}_{{\color{black}k}}^\star-Q_{{\color{black}\mathcal{L}_k}}^\star\|_{\infty}+d_2D(\textbf{argmax-e}(Q_{\mathcal{L}_k}^\star),f_{c,c'}(Q_{\mathcal{L}_k}^\star)).
    \end{split}
\end{equation*}
Since $\mathcal{L}_k\in S_{\epsilon}$ by the projection step, {\color{black} by} Lemma \ref{soft-arg-diff} and 
{\color{black} the algorithm \textit{Alg} has a policy-based guarantee, }with the choice of $T_k=T_{\mathcal{M}_{{\color{black}\mathcal{L}_k}}}(\delta_k,\epsilon_k)$),  
we have,  with probability at least $1-2\delta_k$, 
\begin{equation}
W_1(\tilde{\mathcal{L}}_{k+1},\mathcal{L}^\star)\leq  dW_1(\mathcal{L}_k,\mathcal{L}^\star)+\frac{cd_2\text{diam}(\mathcal{A})|\mathcal{A}|}{2}\epsilon_k+d_2\text{diam}(\mathcal{A})|\mathcal{A}|e^{-c'\phi(\epsilon)}.
\end{equation}
Finally, with probability at least $1-2\delta_k$,
\begin{equation*}
    \begin{split}
   W_1(\mathcal{L}_{k+1},\mathcal{L}^\star)&\leq W_1(\tilde{\mathcal{L}}_{k+1},\mathcal{L}^\star)+ W_1(\tilde{\mathcal{L}}_{k+1},\textbf{Proj}_{S_{\epsilon}}(\tilde{\mathcal{L}}_{k+1})) \\
    &\leq dW_1(\mathcal{L}_k,\mathcal{L}^\star)+\frac{cd_2\text{diam}(\mathcal{A})|\mathcal{A}|}{2}\epsilon_k+d_2\text{diam}(\mathcal{A})|\mathcal{A}|e^{-c'\phi(\epsilon)}+\epsilon.
    \end{split}
\end{equation*}
This implies that with probability at least $1-2\sum_{k=0}^{{\color{black}K-1}}\delta_k$, 
\begin{equation}
\begin{split}
W_1(\mathcal{L}_K, \mathcal{L}^\star)\leq &d^KW_1(\mathcal{L}_0,\mathcal{L}^\star)+\frac{cd_2\text{diam}(\mathcal{A})|\mathcal{A}|}{2}\sum_{k=0}^{K-1}d^{K-k}\epsilon_k \\
&+\dfrac{(d_2\text{diam}(\mathcal{A})|\mathcal{A}|e^{-c'\phi(\epsilon)}+\epsilon)(1-d^{\color{black}K})}{1-d}.
\end{split}
\end{equation}
Since $\epsilon_k$ is summable,  we have $\sup_{k\geq 0}\epsilon_k<\infty$, 
$$\sum_{k=0}^{K-1}d^{K-k}\epsilon_k\leq \dfrac{\sup_{k\geq 0}\epsilon_k}{1-d}d^{\lfloor(K-1)/2\rfloor}+\sum_{k=\lceil(K-1)/2\rceil}^{\infty}\epsilon_k.$$
Now plugging in $K=K_{\epsilon,\eta}$, with the choice of $\delta_k$ and $c=\frac{\log(1/\epsilon)}{\phi(\epsilon)}$, and noticing that $d\in[0,1)$, we have with probability at least $1-2\delta$, 
\begin{equation}
\begin{split}
W_1(\mathcal{L}_{K_{\epsilon,\eta}},\mathcal{L}^\star)\leq &d^{K_{\epsilon,\eta}}W_1(\mathcal{L}_0,\mathcal{L}^\star)\\
&+\frac{cd_2\text{diam}(\mathcal{A})|\mathcal{A}|}{2}\left( \dfrac{\sup_{k\geq 0}\epsilon_k}{1-d}d^{\lfloor(K_{\epsilon,\eta}-1)/2\rfloor}+\sum_{k=\lceil(K_{\epsilon,\eta}-1)/2\rceil}^{\infty}\epsilon_k\right)\\
&+\dfrac{(d_2\text{diam}(\mathcal{A})|\mathcal{A}|+1)\epsilon}{1-d}.
\end{split}
\end{equation}
Setting $\epsilon_k=(k+1)^{-(1+\eta)}$, then when $K_{\epsilon,\eta}\geq 2(\log_d(\epsilon{\color{black}/c})+1)$, 
\[
 \dfrac{\sup_{k\geq 0}\epsilon_k}{1-d}d^{\lfloor(K_{\epsilon,\eta}-1)/2\rfloor}\leq \frac{\epsilon{\color{black}/c}}{1-d}.
\] 
Similarly, when $K_{\epsilon,\eta}\geq 2(\eta\epsilon{\color{black}/c})^{-1/\eta}$,  $$\sum_{k=\left\lceil\frac{K_{\epsilon,\eta}-1}{2}\right\rceil}^{\infty}\epsilon_k\leq \epsilon{\color{black}/c}.$$
Finally, when $K_{\epsilon,\eta}\geq \log_d(\epsilon/(\text{diam}(\mathcal{S})\text{diam}(\mathcal{A})))$, 
$d^{K_{\epsilon,\eta}}W_1(\mathcal{L}_0,\mathcal{L}^\star)\leq \epsilon$,
since $W_1(\mathcal{L}_0,\mathcal{L}^\star)\leq \text{diam}(\mathcal{S}\times\mathcal{A}){\color{black}=\text{diam}(\mathcal{S})\text{diam}(\mathcal{A})}$. 

In summary, if $K_{\epsilon,\eta}=\lceil 2\max\{(\eta\epsilon{\color{black}/c})^{-1/\eta}$, $\log_d(\epsilon/\max\{\text{diam}(\mathcal{S})\text{diam}(\mathcal{A}),{\color{black}c}\})+1\}\rceil$, then with probability at least $1-2\delta$,  
\begin{equation}\label{W1bound}
\begin{split}
&W_1(\mathcal{L}_{K_{\epsilon,\eta}},\mathcal{L}^\star)\leq \left(1+\frac{d_2\text{diam}(\mathcal{A})|\mathcal{A}|(2-d)}{2(1-d)}+\dfrac{(d_2\text{diam}(\mathcal{A})|\mathcal{A}|+1)}{1-d}\right)\epsilon=O(\epsilon).
\end{split}
\end{equation}


{\color{black}Finally, if we are using GMF-V and have assumed that \textit{Alg} satisfies a value-based guarantee with parameters $\{C_{\mathcal{M}}^{(i)},\alpha_1^{(i)},\alpha_2^{(i)},\alpha_3^{(i)},\alpha_4^{(i)}\}_{i=1}^m$,} plugging in $\epsilon_k$ and $\delta_k$ into {\color{black}$T_{\mathcal{M}_{\mathcal{L}}}(\delta_k,\epsilon_k)$}, and noticing that $k {\color{black} \leq }  K_{\epsilon,\eta}$ and $\sum_{k=0}^{K_{\epsilon,\eta}-1}(k+1)^{\alpha}\leq \frac{K_{\epsilon,\eta}^{\alpha+1}}{\alpha+1}$, we  have 
{\color{black}
\begin{equation}\label{Tbound_APP}
\begin{split}
T&=\sum_{k=0}^{K_{\epsilon,\eta}}\sum_{i=1}^mC_{\mathcal{M}}^{(i)}\left(\frac{1}{\epsilon_k}\right)^{\alpha_1^{(i)}}\left(\log\frac{1}{\epsilon_k}\right)^{\alpha_2^{(i)}}\left(\frac{1}{\delta_k}\right)^{\alpha_3^{(i)}}\left(\log\frac{1}{\delta_k}\right)^{\alpha_4^{(i)}}\\
&=\sum_{k=0}^{K_{\epsilon,\eta}}\sum_{i=1}^m(1+\eta)^{\alpha_2^{(i)}}C_{\mathcal{M}}^{(i)}(k+1)^{\alpha_1^{(i)}(1+\eta)}(\log(k+1))^{\alpha_2^{(i)}}(K_{\epsilon,\eta}/\delta)^{\alpha_3^{(i)}}\left(\log(K_{\epsilon,\eta}/\delta)\right)^{\alpha_4^{(i)}}\\
&\leq \sum_{i=1}^m\frac{(1+\eta)^{\alpha_2^{(i)}}}{\alpha_1^{(i)}(1+\eta)+1}C_{\mathcal{M}}^{(i)}K_{\epsilon,\eta}^{\alpha_1^{(i)}(1+\eta)+1}\left(\log(K_{\epsilon,\eta}+1)\right)^{\alpha_2^{(i)}}(K_{\epsilon,\eta}/\delta)^{\alpha_3^{(i)}}\left(\log(K_{\epsilon,\eta}/\delta)\right)^{\alpha_4^{(i)}}\\
&\leq \sum_{i=1}^m\dfrac{2^{\alpha_2^{(i)}}}{2\alpha_1^{(i)}+1}C_{\mathcal{M}}^{(i)}K_{\epsilon,\eta}^{2\alpha_1^{(i)}+1}(K_{\epsilon,\eta}/\delta)^{\alpha_3^{(i)}}\left(\log(K_{\epsilon,\eta}/\delta)\right)^{\alpha_2^{(i)}+\alpha_4^{(i)}},
\end{split}
\end{equation}
which completes the proof of the value-based case.
}
\end{proof}

{\color{black}
\begin{proof}{[Proof of Theorem \ref{conv_AIQL-II}]}
If we use GMF-P and assume that \textit{Alg} has the policy-based guarantee, then by Lemma \ref{conv_2to1}, 
\begin{equation}
\mathbb{P}\left(\left\|\tilde{Q}_{\hat{\pi}_k}^{l_k}-Q_{\mathcal{L}_k}^\star\right\|_{\infty}> \epsilon\right)\leq 2\delta.
\end{equation}
Hence one can simply replace $\hat{Q}_{k}^\star$ by $\tilde{Q}_{\hat{\pi}_k}^{l_k}$ in the proof of Theorem \ref{conv_AIQL}, and obtain the same bound on $W_1(\mathcal{L}_{K_{\epsilon,\eta}},\mathcal{L}^\star)$ (\textit{cf.} \eqref{W1bound}). The only difference is that in each iteration, the required number of samples $T_{\mathcal{M}_{\mathcal{L}}}$ now has parameters $\{\tilde{C}_{\mathcal{M}}^{(i)},\alpha_1^{(i)},\alpha_2^{(i)},\alpha_3^{(i)},\alpha_4^{(i)}\}_{i=1}^{m+3}$  as defined in Lemma \ref{conv_2to1}. Hence repeating the proof of \eqref{Tbound_APP} leads to \eqref{Tbound-II}.
\qed
\end{proof}
}

\section{Experiments}\label{experiments}

In this section, we report the performance of the proposed GMF-V-Q Algorithm and GMF-P-TRPO Algorithm with an equilibrium pricing model (see Section \ref{sec:examples}). The objectives of the experiments include 1) testing the convergence and stability of both GMF-V-Q and GMF-P-TRPO  in the GMFG setting, 2) empirically verifying the contractive property of mapping $\Gamma$, and 3) comparing GMF-V-Q  and GMF-P-TRPO with existing multi-agent reinforcement learning algorithms, including the Independent Learner (IL) algorithm \citet{T1993,hu2020evolutionary} and the MF-Q\footnote{\color{black}Note that MF-Q is designed for global states and coupled local actions, while in our equilibrium price example we have coupled local (private) states and decoupled local actions. To suit this setting, we adapt MF-Q by replacing the mean-field action term with the mean-field state term.} algorithm \citet{YLLZZW2018}. Another set of experiments for the repeated auction model (see Section \ref{sec:examples}) is demonstrated in the short version \citet{guo2019learning}.

\subsection{Set-up and parameter configuration}\label{sec:experiment_setup}
 We introduce two testing environments in our numerical experiments, one is the GMFG environment with a continuum of agents (i.e., infinite number of agents) descried in Section \ref{sec:examples} and the other one is an N-player environment with a weak simulator.

\paragraph{Equilibrium price as an $N$-player game.}
We also consider an $N$-player game version of the equilibrium price model, which is the GMFG version described above with an $N$-player weak simulator oracle as described in Section \ref{appl_n_game}.  {In particular, Take $N$ companies. At each time $t$, company $i$ decides a quantity $q_t^i$ for production and a quantity $h_t^i$ to replenish the inventory. Let $s_t^i$ denote the current inventory level of company $i$ at time $t$. Then similar to Section \ref{sec:examples}, the inventory level evolves according to
\[
s_{t+1}^i = s_t^i-\min\{q_t^i,s_t^i\}+h_t
\]
and the reward of company $i$ at time $t$ is given by
\[
r_t^i = (p_t-c_0)q_t^i-c_1(q_t^i)^2-c_2h_t^i-(c_2+c_3)\max\{q_t^i-s_t^i,0\}-c_4s_t^i.
\]
Here $p_t$, the price of the product at time $t$, is determined according to the supply-demand equilibrium on the market. The total supply is $\sum_{i=1}^Nq_t^i$, while the total demand is assumed to be $d_Np_t^{-\sigma}$, where $d_N=dN$ is supposed to be linearly growing as $N$ grows, \textit{i.e.}, the number of customers grows proportionally to the number of producers in the market. Then by equating supply and demand, we obtain that 
\[
\frac{1}{N}\sum_{i=1}^Nq_t^i=d p_t^{-\sigma},
\]
and by taking the limit $N\rightarrow\infty$, we obtain the mean-field counterpart \eqref{price_equil_model}}.

In this setting, accordingly, we test the performance of GMF-VW-Q, which is GMF-VW (Algorithm \ref{GMF-VW}) with synchronous Q-learning and the standard \textbf{softmax} operator (\textit{cf}. Algorithm \ref{AIQL_MFG_W}) and  GMF-PW-TRPO, which is GMF-PW (Algorithm \ref{GMF-PW}) with TRPO and the standard \textbf{softmax} operator.{\color{black}\footnote{For the sake of brevity, we omit the algorithm frame for GMF-PW-TRPO.} }

{\color{black}
\begin{algorithm}[H]
  \caption{\textbf{Q-learning for GMFGs (GMF-VW-Q)}: weak simulator}
  \label{AIQL_MFG_W}
\begin{algorithmic}[1]
  \STATE \textbf{Input}: Initial $\mathcal{L}_0$, {\color{black}$\epsilon$-net $S_{\epsilon}$, tolerances $\epsilon_k,~\delta_k>0$, $k=0,1,\dots$.} 
 \FOR {$k=0, 1, \cdots$}
  \STATE 
  \begin{addmargin}[1em]{0em}
  Perform Q-learning with hyper-parameters in Lemma \ref{Q-finite-bd} for {\color{black}$T_k=T_{\mathcal{M}_{\mathcal{L}_k}}(\epsilon_k,\delta_k)$} iterations to find the approximate Q-function {\color{black}$\hat{Q}_k^\star=\hat{Q}^{T_k}$} of the MDP {\color{black}$\mathcal{M}_{\mathcal{L}_k}$}. 
  \end{addmargin}
  \STATE 
  \begin{addmargin}[1em]{0em}
  Compute $\pi_k\in\Pi$ with $\pi_k(s)=\textbf{softmax}_c(\hat{Q}^\star_k(s,\cdot))$.
  \end{addmargin}\vspace{-0.3cm}
      \begin{addmargin}[1em]{0em}
       \FOR {$i=1, 2, \cdots,N$}
 \STATE
 \begin{addmargin}[1em]{0em}
 Sample $s_i\overset{\text{i.i.d.}}{\sim} \mu_k$,  then obtain $s_i'$ i.i.d. from $\mathcal{G}_W(s_i,\pi_k,\mathcal{L}_k)$ and $a_i'\overset{\text{i.i.d.}}{\sim}\pi_k(s_i')$. 
 \end{addmargin}\vspace{0.3em}
 \ENDFOR
 \STATE
 Compute $\mathcal{L}_{k+1}$ with $\mathcal{L}_{k+1}(s,a)=\frac{1}{N}\sum_{i=1}^N\textbf{I}_{s_i'=s,a_i'=a}$.
  \end{addmargin}
\ENDFOR
\end{algorithmic}
\end{algorithm}
}


 \paragraph{Parameters.}
The model parameters are (unless otherwise specified):  $\gamma=0.2$, $d=50$ and $\sigma=2$. $S=Q=H=10$ and hence $|\mathcal{S}|=10$ and $|\mathcal{A}|=100$. $c_0 = 0.5$, $c_1 = 0.1$, $c_2=0.5$, $c_3 = 0.2$ and $c_4=0.2$.

The algorithm parameters are (unless otherwise specified):  the temperature parameter is set as $c=4.0$ and the learning rate is set as $\eta=0.01$    \footnote{Lemma \ref{Q-finite-bd} indicates that the learning rate should be inverse proportional to the current visitation number of a given state-action pair, we observe that constant learning rate works well in practice which is easier to implement.}. For simplicity, we set the inner iteration $T_k$ to be $100\times |\mathcal{S}|\times |\mathcal{A}|$. 
The $90\%$-confidence intervals are calculated with $20$ sample paths.

\subsection{Performance evaluation in the GMFG setting.}\label{performance_eval}
Our experiments show that  GMF-V-Q and GMF-P-TRPO Algorithms are efficient and robust.

\paragraph{Performance metric.}
We adopt  the following metric  to measure the difference between  a given policy $\pi$ and an NE (here $\epsilon_0>0$ is a safeguard, and is taken as $0.1$ in the experiments): 
$$C_{MF}({\pi}) =\dfrac{\max_{{\pi}^{\prime}} \mathbb{E}_{s\sim \mu}[V(s,\pi^{\prime},\mathcal{L})]-V(s,\pi,\mathcal{L})}{|\max_{{\pi}^{\prime}} \mathbb{E}_{s\sim \mu}[V(s,\pi^{\prime},\mathcal{L})]|+\epsilon_0}.$$
Here {$\mu$ is the invariant distribution of the transition matrix $P^\pi$, where $P^\pi(s,s')=\sum_{a\in\mathcal{A}}P(s'|s,a)\pi(a|s)$ for $s,s'\in\mathcal{S}$, and $\mathcal{L}(s,a)=\mu(s)\pi(a|s)$ for $s,a\in\mathcal{S}\times\mathcal{A}$. Note that in the equilibrium product pricing model we are considering here, the transition model $P$ is independent of the mean-field term $\mathcal{L}$, and hence we write $P(s'|s,a)=P(s'|s,a,\mathcal{L})$. In general, an additional mean-field  matching error term needs to be added into the definition of $C_{MF}(\pi)$.} 
Clearly $C_{MF}({\pi}) \geq 0$, and $C_{MF}({\pi}^\star)=0$ if and only if $({\pi}^\star,\mathcal{L}^\star)$ is an NE where  {$\mathcal{L}^\star$ is the invariant distribution of $P^{\pi^\star}$}.
   A similar metric without normalization has been adopted in \citet{cui2021approximately}.

\paragraph{Contractiveness of mapping $\Gamma$.} As explained in Remark \ref{rmk:contract} from Section \ref{MFG_basic}, the contractiveness property of $\Gamma$ is the key for establishing the uniqueness of MFG solution and hence the convergence of the GMFG algorithm.  To empirically verify whether this property holds for the equilibrium price example, 
we plot the value of $\frac{\|\Gamma(\mathcal{L}_1)-\Gamma(\mathcal{L}_2)\|_{1}}{\|\mathcal{L}_1-\mathcal{L}_2\|_{1}}$ for randomly generated state-action distributions $\mathcal{L}_1$ and $\mathcal{L}_2$. Technically speaking, $\Gamma$ is contractive and there exists a unique MFG solution if the value of $\frac{\|\Gamma(\mathcal{L}_1)-\Gamma(\mathcal{L}_2)\|_{1}}{\|\mathcal{L}_1-\mathcal{L}_2\|_{1}}$ is smaller than one for all choices of $\mathcal{L}_1$ and $\mathcal{L}_2$. 

 W observe from Figure \ref{fig:contraction} that, with various of choices of different model parameters, the quantity $\frac{\|\Gamma(\mathcal{L}_1)-\Gamma(\mathcal{L}_2)\|_{1}}{\|\mathcal{L}_1-\mathcal{L}_2\|_{1}}$ is always smaller than $0.3$ indicating that $\Gamma$ is contractive.

\begin{figure}[H]
\centering
\subfigure[Default setting (see Section \ref{sec:experiment_setup}).]{ \label{fig:contraction1}
\includegraphics[width=.3\textwidth]{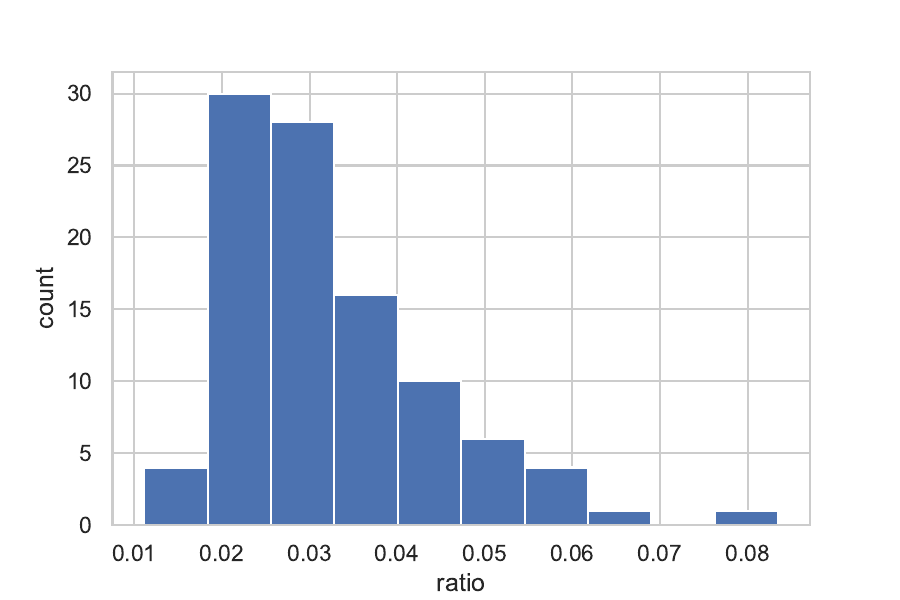}
}
\subfigure[$\gamma=0.1$.]{ \label{fig:contraction2}
\includegraphics[width=.3\textwidth]{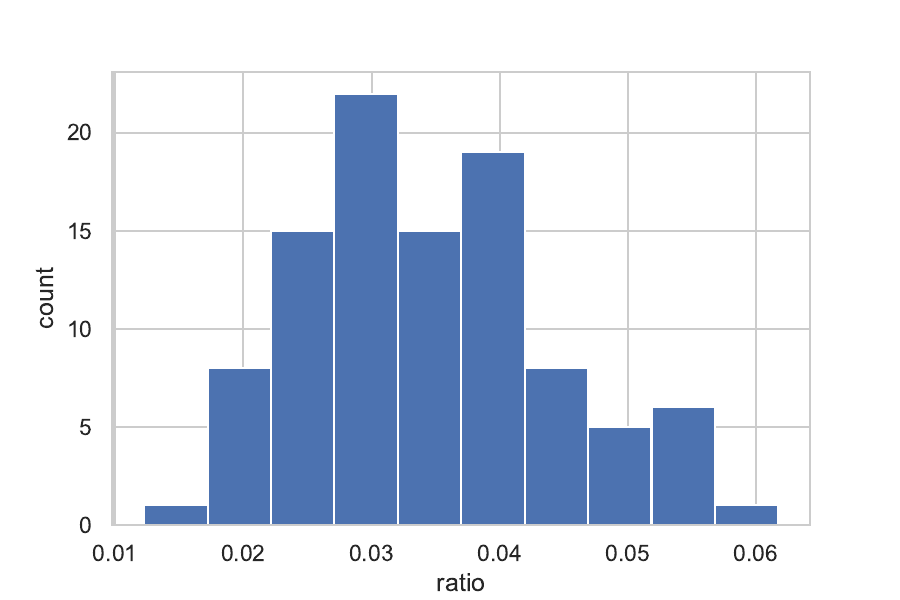}
}
\subfigure[$|\mathcal{S}|=5$ and $|\mathcal{A}|=25$.]{\label{fig:contraction3}
\includegraphics[width=.3\textwidth]{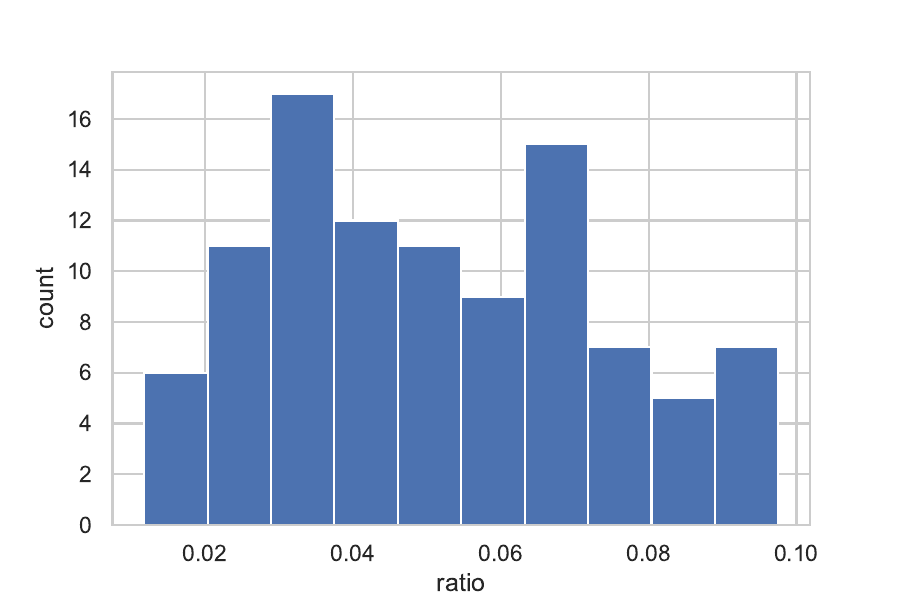}}

\subfigure[$c_0 = 2.5$, $c_1 = 0.5$, $c_2=2.5$, $c_3 = 1.0$ and $c_4=1.0$.]{\label{fig:contraction4}
\includegraphics[width=.3\textwidth]{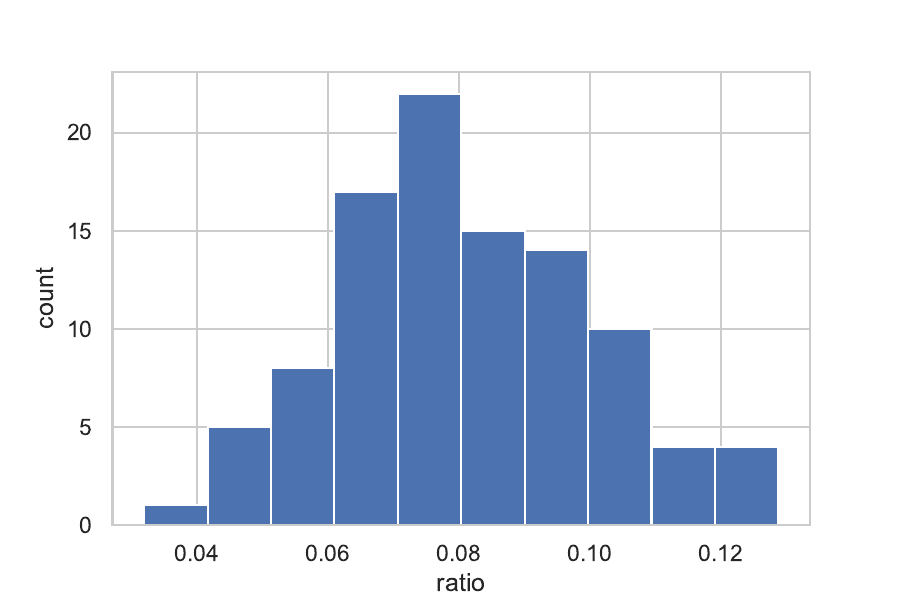}}
\subfigure[$d=200$.]{\label{fig:contraction5}
\includegraphics[width=.3\textwidth]{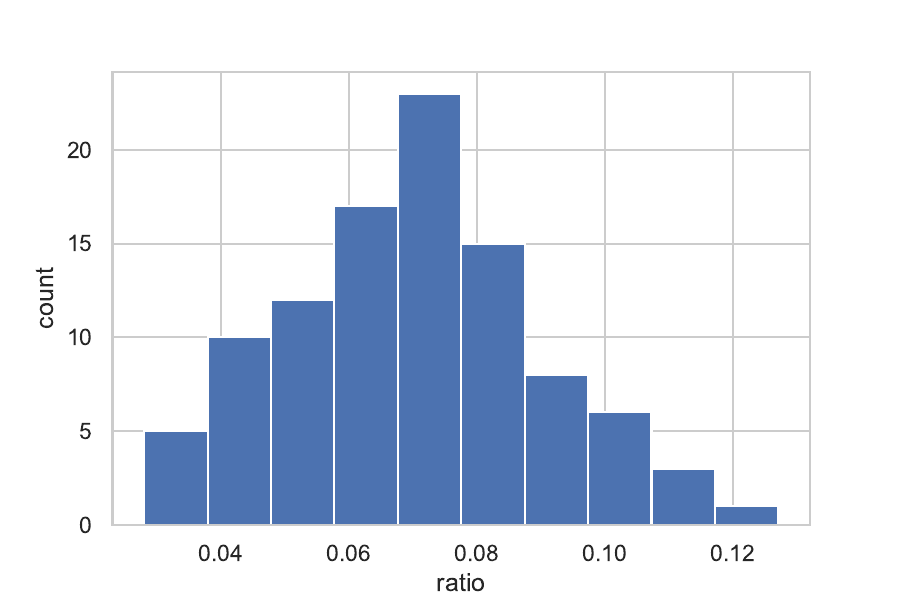}}
\subfigure[$\sigma=1.0$.]{\label{fig:contraction6}
\includegraphics[width=.3\textwidth]{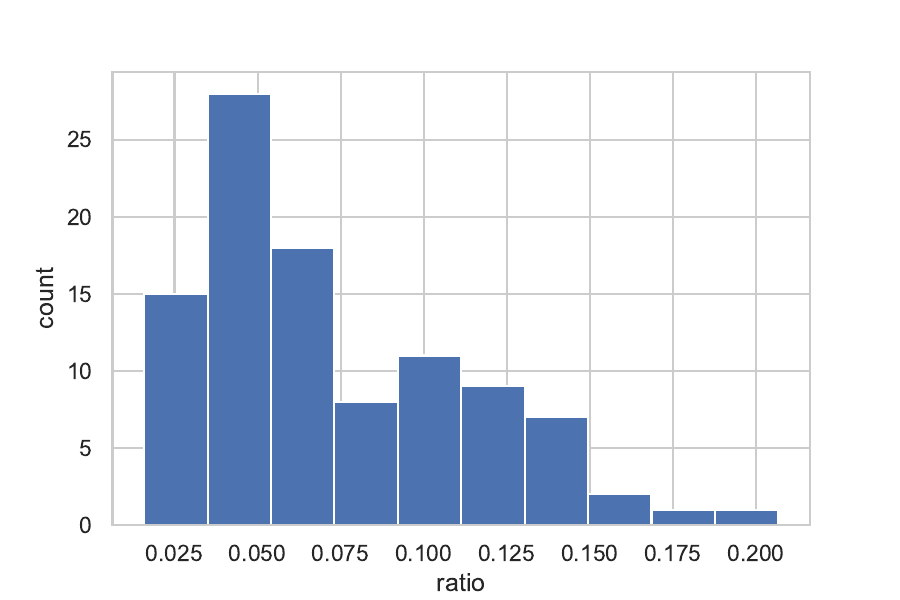}}
\caption{Histogram of  $\frac{\|\Gamma(\mathcal{L}_1)-\Gamma(\mathcal{L}_2)\|_{1}}{\|\mathcal{L}_1-\mathcal{L}_2\|_{1}}$ under various settings ($\mathcal{L}_1$ and $\mathcal{L}_2$ are randomly sampled according to the uniform distribution).}
\label{fig:contraction}
\end{figure}

\paragraph{ Convergence and stability.}
Both GMF-V-Q and GMF-P-TRPO are efficient and robust. First,  both GMF-V-Q and  GMF-P-TRPO converge within about $5$ outer iterations; secondly, as the number of inner iterations increases, the error decreases (Figure~\ref{fig:different_inner}); 
and finally,  the convergence is robust with respect to both the change of number of states and actions (Figure~\ref{fig:different_state}). The performance of GMF-V-Q 
  is (slightly) more stable than GMF-P-TRPO with a smaller variance across 20 repeated experiments (see Figure \ref{fig:different_inner_1} versus Figure \ref{fig:different_inner_2} or Figure \ref{fig:different_state_1} versus Figure \ref{fig:different_state_2}). {This is due to the fact that GMF-P-TRPO uses asynchronous updates, which leads to slightly less stable performance compared to GMF-V-Q, which uses synchronous updates.}

In contrast, the  Naive algorithms, i.e., GMF-V-Q without smoothing (denoted as GMF-V-Q-nonsmoothing) and GMF-P-TRPO without smoothing (denoted as GMF-P-TRPO-nonsmoothing), do not converge even with $50$ outer iterations and $200\times |\mathcal{S}|\times |\mathcal{A}|$ inner iterations within each outer iteration. In particular, GMF-V-Q-nonsmoothing  and GMF-P-TRPO-nonsmoothing present different unstable behaviors (see Figure \ref{fig:naive}).  The  joint distribution $\mathcal{L}_t$ from GMF-V-Q-nonsmoothing keeps fluctuating (Figure~\ref{fig:naive_1}) whereas the  joint distribution $\mathcal{L}_t$ from GMF-P-TRPO (without smoothing) is trapped around the initialization which is far away from the true equilibrium distribution (Figure~\ref{fig:naive_2}).

\begin{figure}[H]
    \centering
\subfigure[GMF-V-Q]{ \label{fig:different_inner_1}
\includegraphics[width=.42\textwidth]{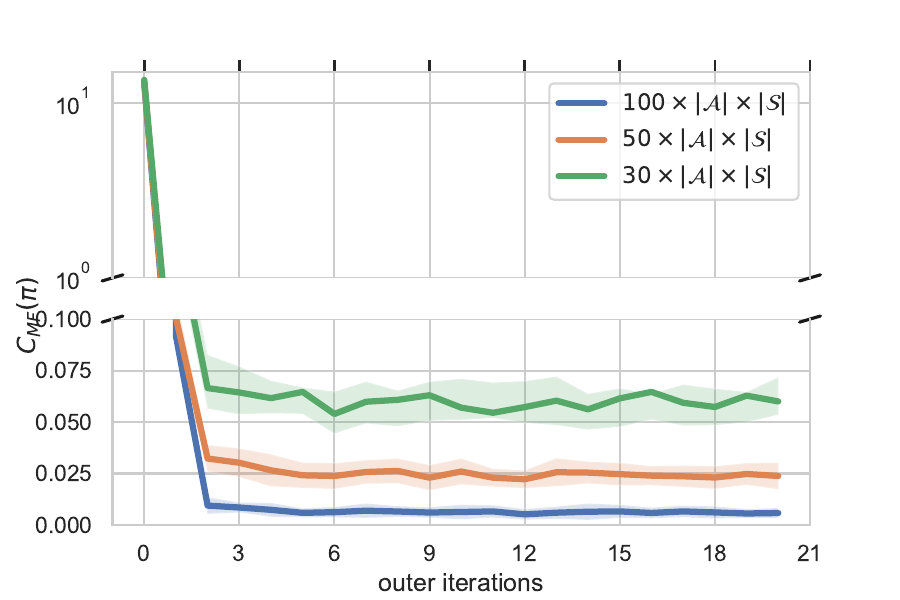}
}
\subfigure[GMF-P-TRPO.]{ \label{fig:different_inner_2}
\includegraphics[width=.42\textwidth]{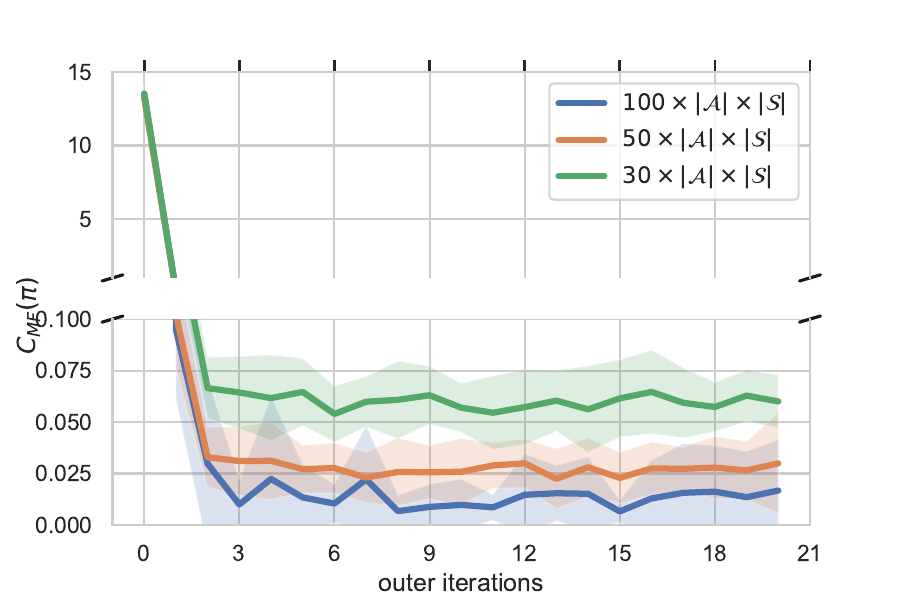}
}
\caption{Convergence with different   number of inner iterations ($|\mathcal{A}|=100$ and $|\mathcal{S}|=10$).\label{fig:different_inner}}
\end{figure}

\begin{figure}[H]
    \centering
\subfigure[GMF-V-Q]{ \label{fig:different_state_1}
\includegraphics[width=.42\textwidth]{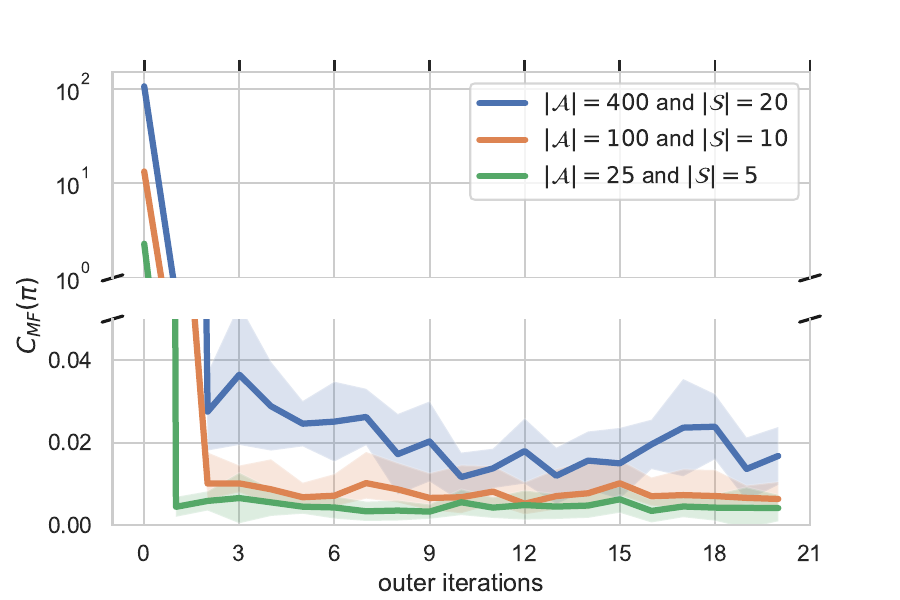}
}
\subfigure[GMF-P-TRPO.]{ \label{fig:different_state_2}
\includegraphics[width=.42\textwidth]{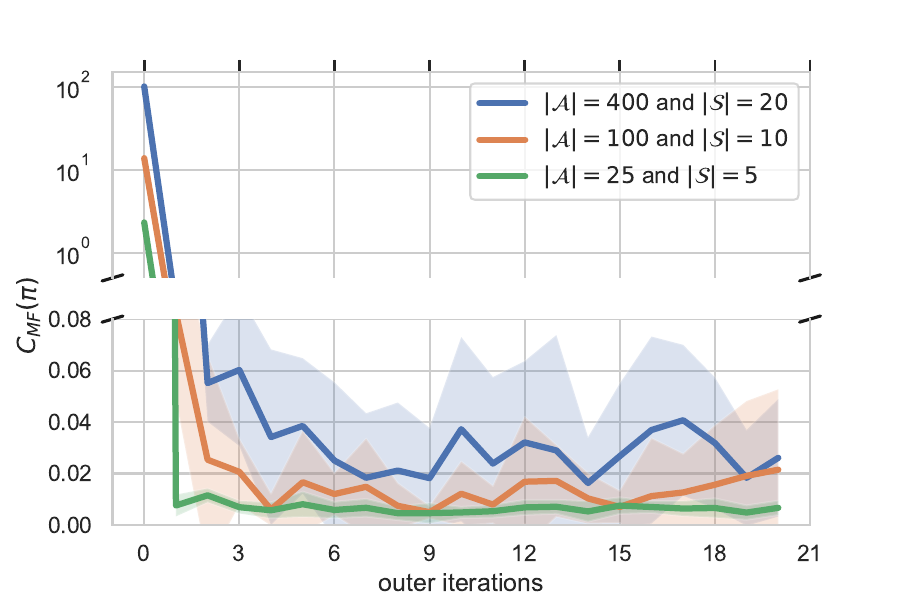}
}
\caption{Convergence with different size of state space and action space .\label{fig:different_state}}
\end{figure}

 \begin{figure}[H]
    \centering
\subfigure[Q-learning (5 sample paths)]{ \label{fig:naive_1}
\includegraphics[width=.42\textwidth]{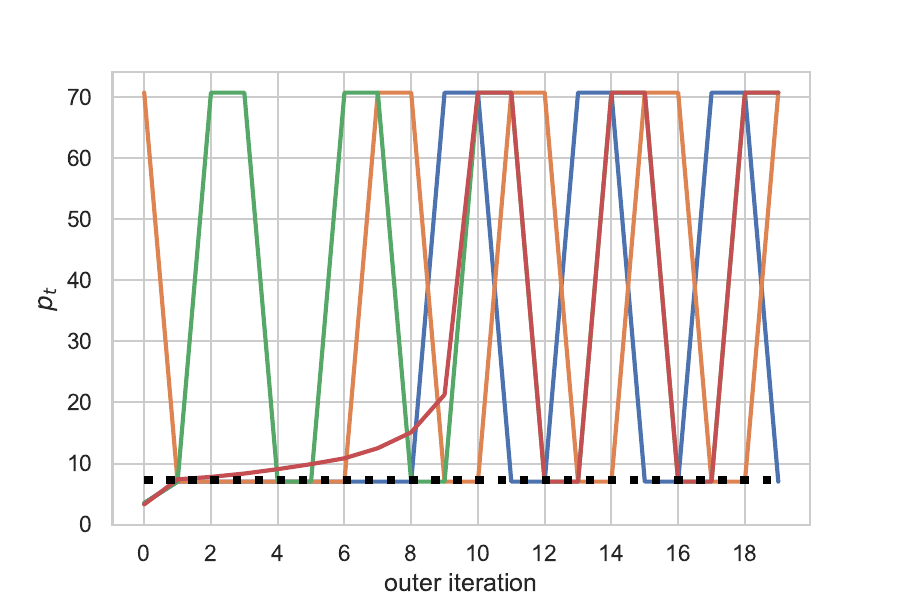}
}
\subfigure[TRPO (5 sample paths).]{ \label{fig:naive_2}
\includegraphics[width=.42\textwidth]{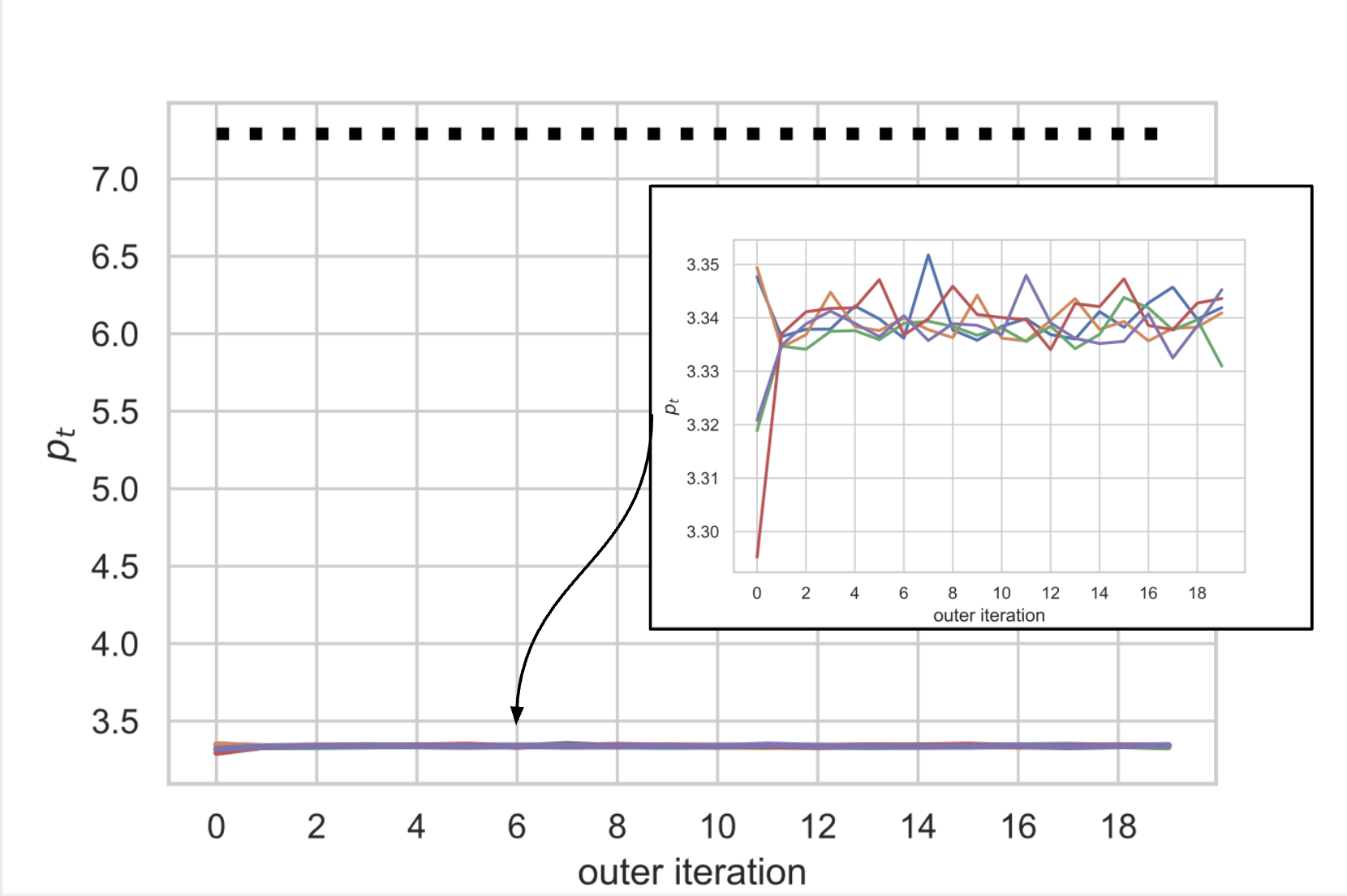}
}
\caption{Fluctuations of algorithms without smoothing (Dotted black line: theoretical value of the equilibrium price).}\label{fig:naive}
\end{figure}

\paragraph{Model verification and interpretation of equilibrium scenario.} 
In Figures \ref{fig:VP} and \ref{fig:VP_v2}, we run both algorithms for 20 outer iterations with the same number of inner iterations (100,000 = $100\times|\mathcal{A}|\times |\mathcal{S}|$) within each outer iteration. The final equilibrium inventory distribution and production distribution from both algorithms are close to each other. 
\begin{figure}[H]
\centering
\subfigure[Equilibrium inventory distribution.]{ \label{fig:VP1}
\includegraphics[width=.3\textwidth]{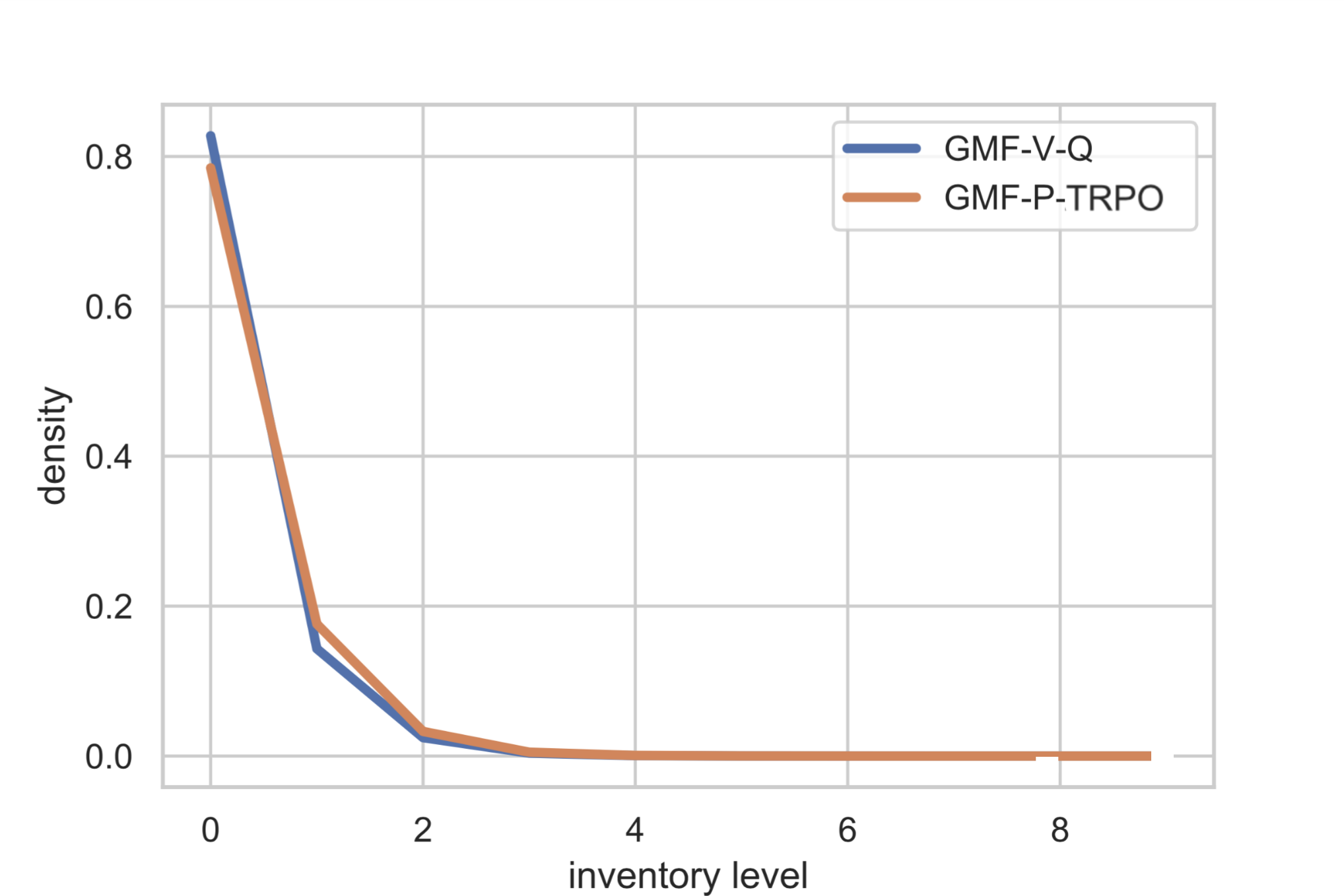}
}
\subfigure[Equilibrium production distribution.]{ \label{fig:VP2}
\includegraphics[width=.3\textwidth]{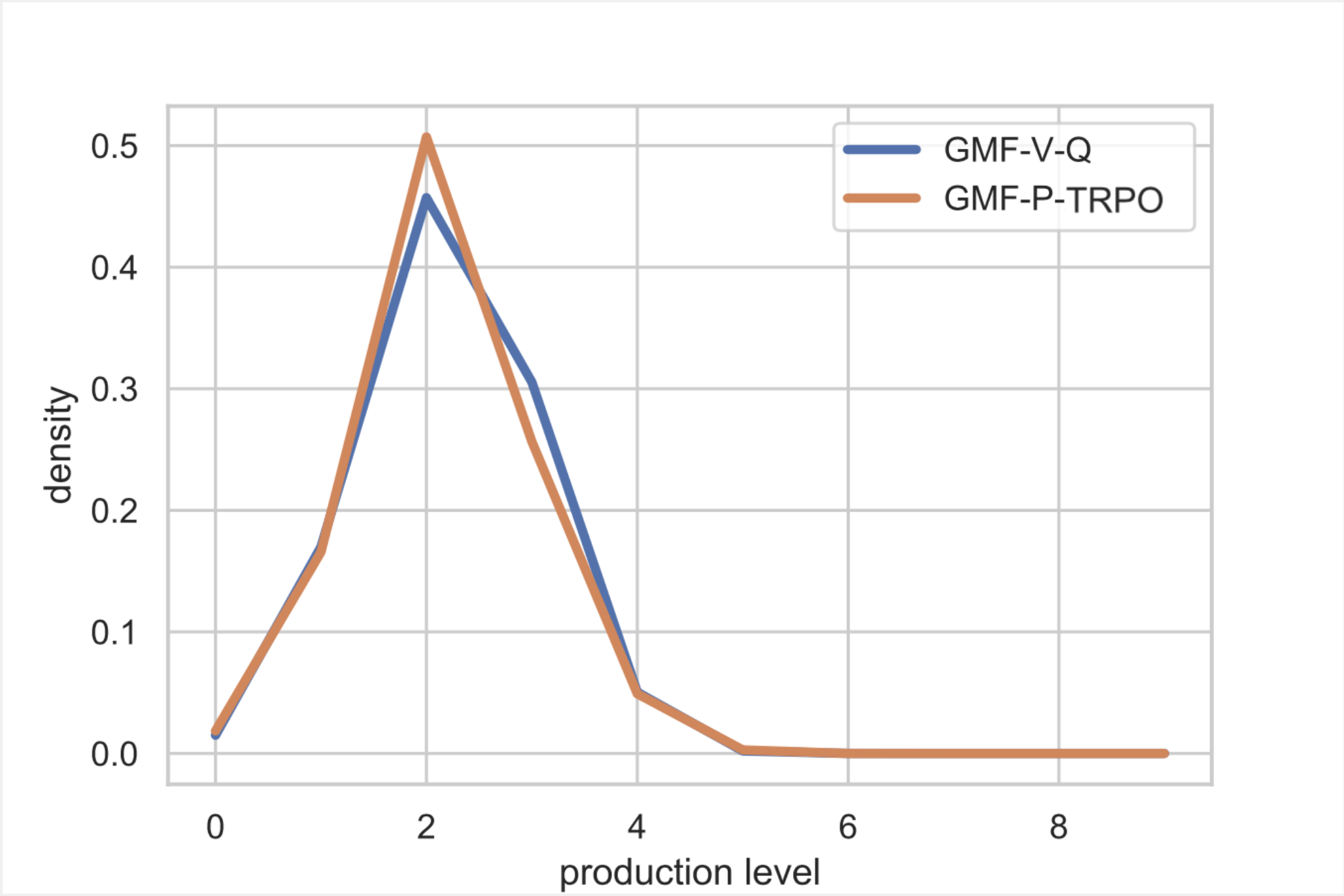}
}
\subfigure[Equilibrium price at the end of each outer iteration.]{\label{fig:VP3}
\includegraphics[width=.3\textwidth]{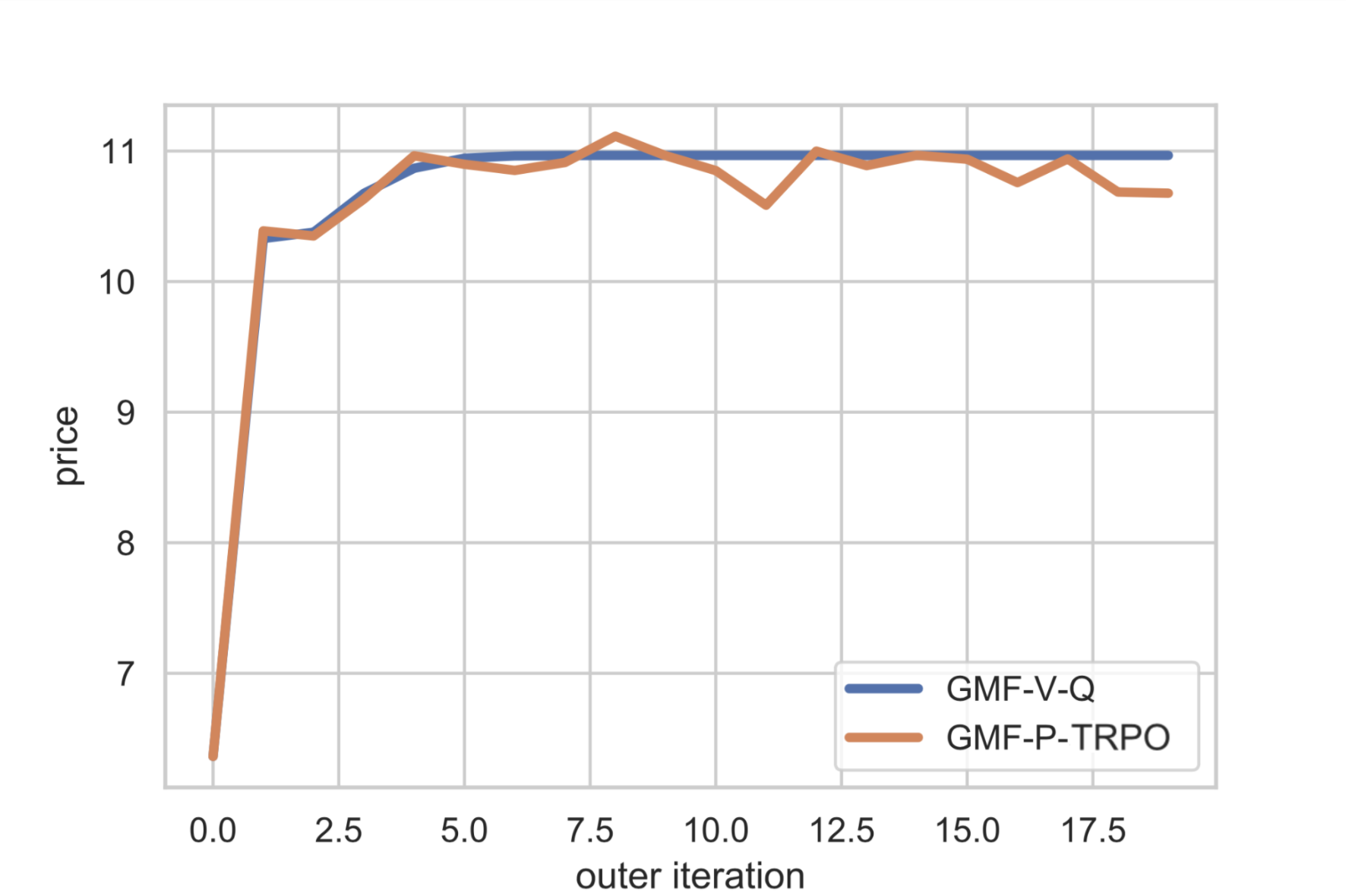}}
\caption{GMF-V-Q versus GMF-P-TRPO ($\sigma = 1.3$ and one trajectory). }
\label{fig:VP}
\end{figure}
\begin{figure}[H]
\centering
\subfigure[Equilibrium inventory distribution.]{ \label{fig:VP1_v2}
\includegraphics[width=.3\textwidth]{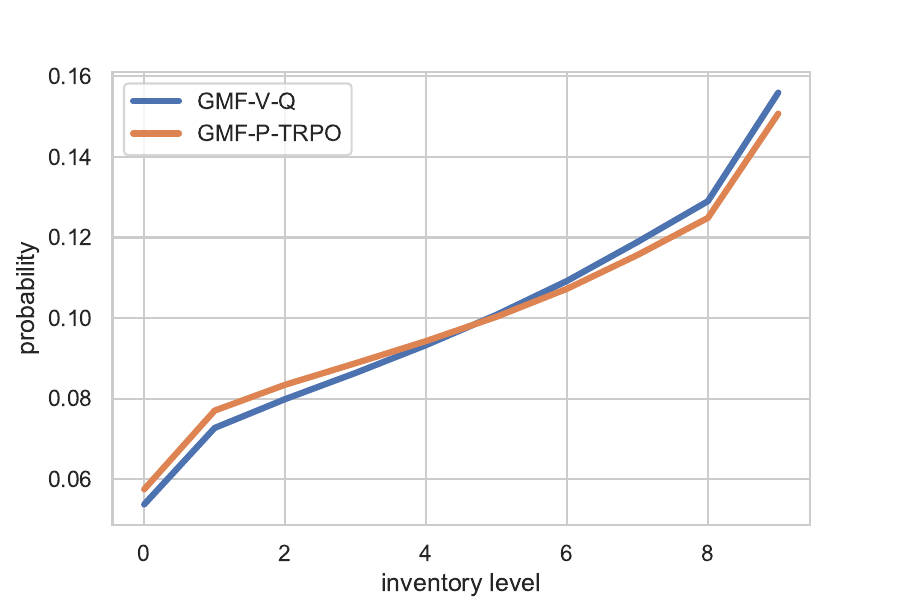}
}
\subfigure[Equilibrium production distribution.]{ \label{fig:VP2_v2}
\includegraphics[width=.3\textwidth]{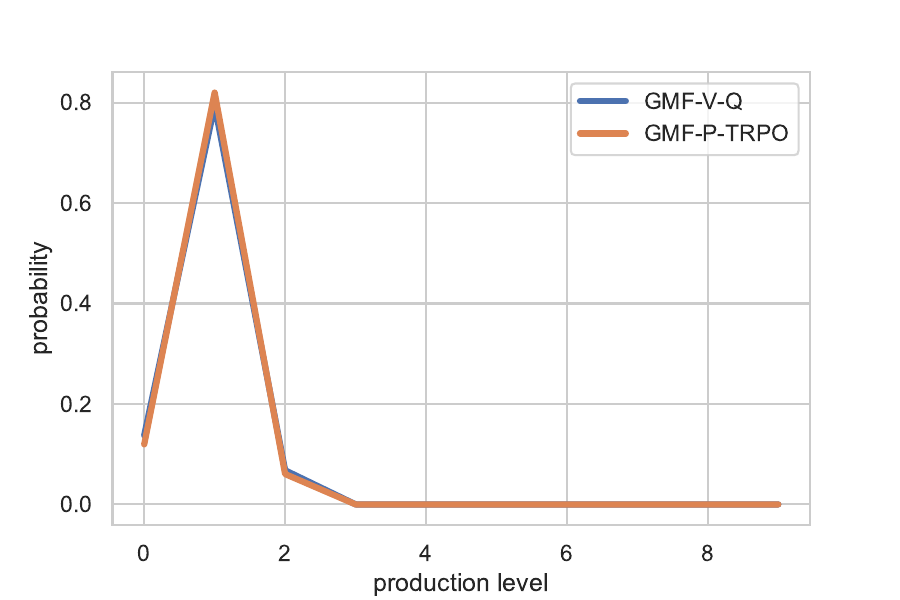}
}
\subfigure[Equilibrium price at the end of each outer iteration.]{\label{fig:VP3_v2}
\includegraphics[width=.3\textwidth]{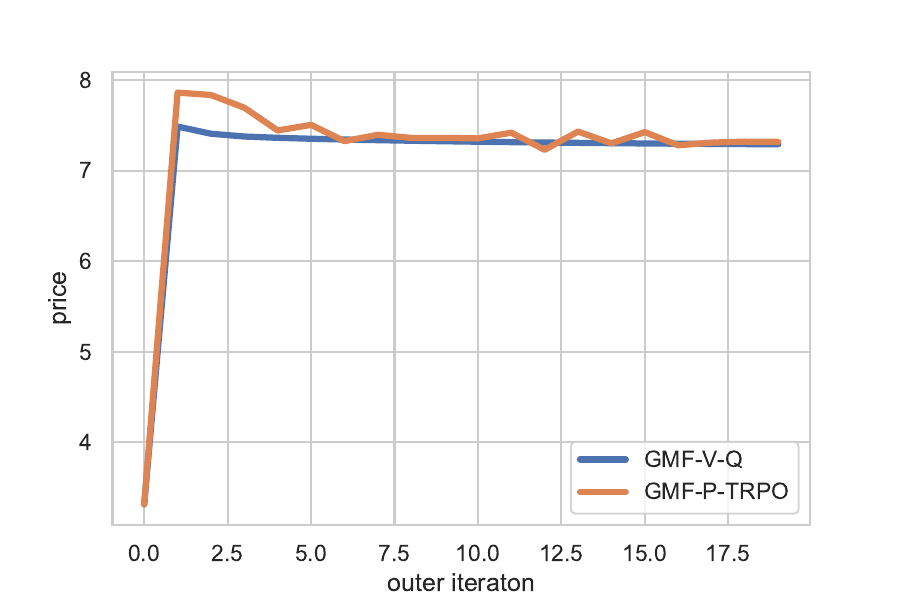}}
\caption{GMF-V-Q versus GMF-P-TRPO ($\sigma =2.0$ and one trajectory). }
\label{fig:VP_v2}
\end{figure}
Note that the demand elasticity $\sigma$  captures how sensitive demand for a product is compared to the changes in other economic factors, such as price or income. When $\sigma$ is increased from $1.3$ to $2.0$ indicating that the demand is more sensitive to price rise, the equilibrium price decreases from $10.9$ to $7.3$ (see Figures \ref{fig:VP3} and \ref{fig:VP3_v2}) and the distribution of the equilibrium production level is centered towards smaller values (see Figures \ref{fig:VP2} and \ref{fig:VP2_v2}).  The equilibrium inventory level has a huge mass at $0$ when $\sigma=1.3$. This implies that producers do not keep large inventories and pay the inventory cost in the equilibrium. On the other hand, the equilibrium inventory is more uniformly distributed when $\sigma=2$.

\subsection{Performance evaluation in the N-player setting}

 \paragraph{Performance metric.}
Similar to the performance metric introduced in Section \ref{performance_eval} for the GMFG setting, we adopt  the following metric  to measure the difference between  a given policy $\pi$ and an NE under THE N-player setting (here $\epsilon_0>0$ is a safeguard, and is taken as $0.1$ in the experiments): {\color{black}
$$C(\pmb{\pi}) = \frac{1}{N|\mathcal{S}|^N}\sum\nolimits_{i=1}^N \sum\nolimits_{{\bf s} \in \mathcal{S}^N  }\dfrac{\max_{{\pi}^i} V^i({\bf s},({\pmb{\pi}^{-i}},\pi^i))-V^i({\bf s},{\pmb{\pi}})}{|\max_{{\pi}^i} V^i({\bf s},({\pmb{\pi}^{-i}},\pi^i))|+\epsilon_0}.$$
}
Clearly $C(\pmb{\pi}) \geq 0$, and $C(\pmb{\pi}^\star)=0$ if and only if $\pmb{\pi}^\star$ is an NE. 
 Policy $\arg \max _{{\pi}_i} V_i({\bf s},({\pmb{\pi}^{-i}},\pi_i))$ is called the best response to $\pmb{\pi}^{-i}$.  A similar metric without normalization has been adopted in \citet{PPP2018}.

\paragraph{Existing algorithms for $N$-player games.}
To test the effectiveness of {\color{black}GMF-VW-Q} for approximating $N$-player games, we next compare {\color{black}GMF-VW-Q} with the IL algorithm and the MF-Q algorithm. The IL algorithm \citet{T1993} considers $N$ independent players and each player solves a decentralized reinforcement learning problem ignoring other players in the system. The MF-Q algorithm \citet{YLLZZW2018} extends the NASH-Q Learning algorithm for the $N$-player game introduced in \citet{HW2003}, adds the aggregate actions $(\bar{\pmb{a}}_{-i}=\frac{\sum_{j \ne i}a_j}{{\color{black}N}-1})$ from the opponents, and  works for the class of games where the interactions are only through the average actions of $N$ players.

\paragraph{Results and analysis.}
Our experiment (Figure \ref{fig:comparison}) shows that {GMF-VW-Q and GMF-PW-TRPO achieve similar performance, and both of them are superior in terms of convergence rate, accuracy, and stability for approximating an $N$-player game. In general, both algorithms converge faster than  IL and  MF-Q and achieve the smallest errors.} 

For instance, when $N=20$,  IL Algorithm converges  with the largest error $0.220$. 
The error from MF-Q is $0.101$,  smaller than  IL  but  still bigger than the error from GMF-VW-Q.
The GMF-VW-Q and GMF-PW-TRPO converge with  the lowest error $0.065$.  Moreover, as $N$ increases, the error of GMF-VW-Q  and GMF-PW-TRPO decease while the errors of both MF-Q and IL increase significantly. As $|\mathcal{S}|$ and  $|\mathcal{A}|$ increase, GMF-VW-Q and GMF-PW-TRPO are robust with respect to this increase of dimensionality, while 
 both MF-Q and IL clearly suffer from the increase of the dimensionality with decreased  convergence rate and accuracy.
Therefore, GMF-VW-Q and GMF-PW-TRPO are more scalable than IL and MF-Q, when the system is complex and the number of players $N$ is large. 


\begin{figure}[H]
\centering
\subfigure[$|\mathcal{S}|=10,|\mathcal{A}|=100,{\color{black}N}=20$.]{ \label{fig:comparison1}
\includegraphics[width=.3\textwidth]{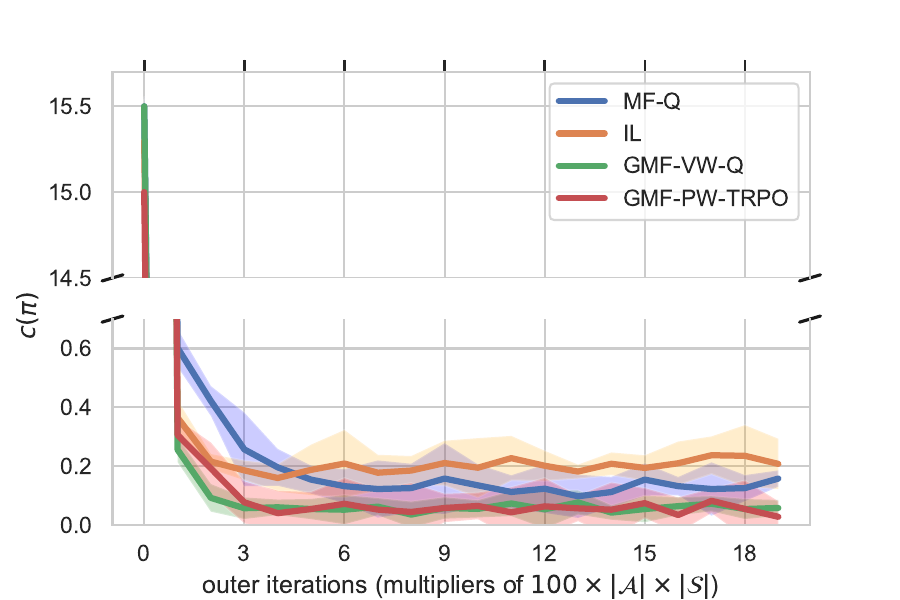}
}
\subfigure[$|\mathcal{S}|=15,|\mathcal{A}|=225,{\color{black}N}=20$.]{ \label{fig:comparison1}
\includegraphics[width=.3\textwidth]{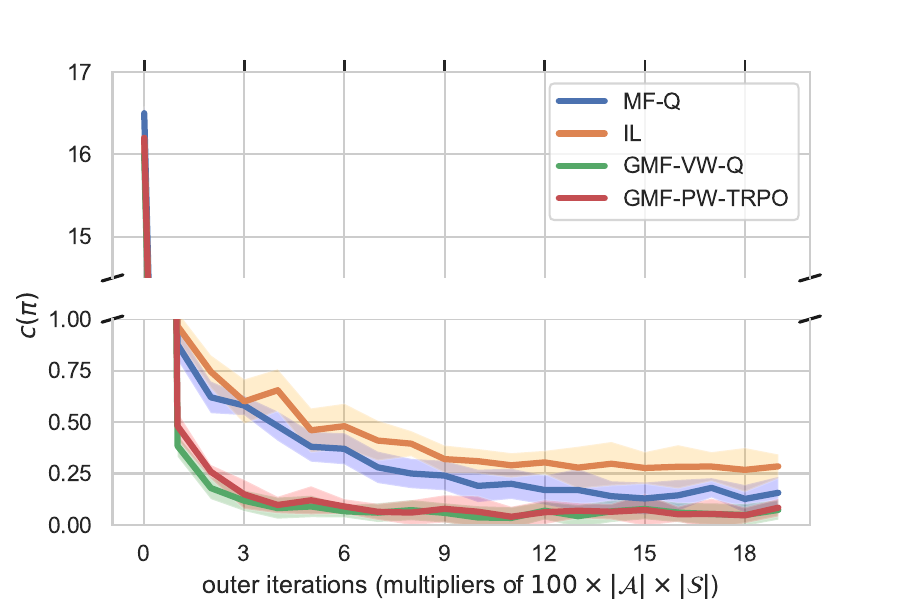}
}
\subfigure[$|\mathcal{S}|=10,|\mathcal{A}|=100,{\color{black}N}=60$.]{\label{fig:comparison2}
\includegraphics[width=.3\textwidth]{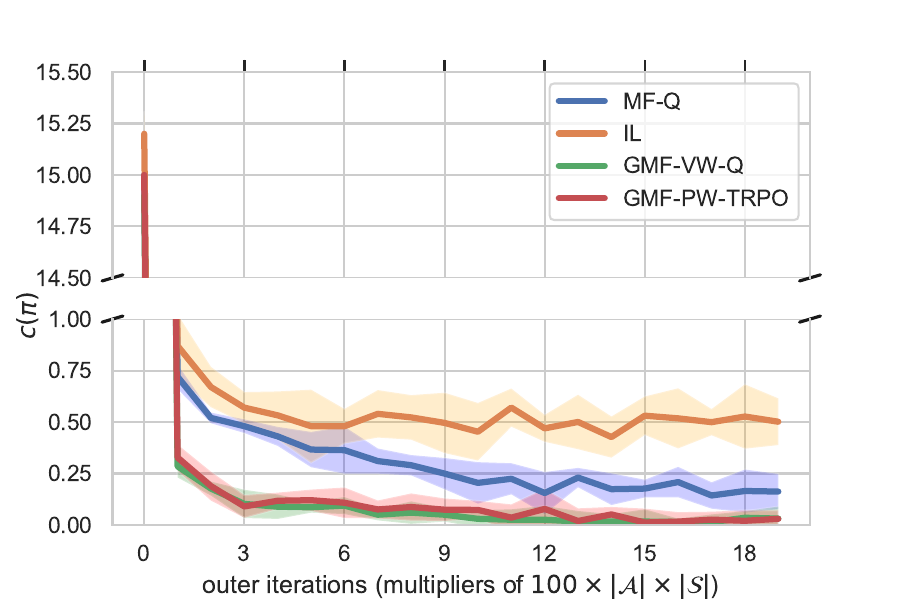}}

\caption{Learning accuracy based on $C(\pmb{\pi})$.}
\label{fig:comparison}
\end{figure}


\section{Extension: Existence and uniqueness for {\color{black}non-stationary} NE of GMFGs} \label{stat_mfg_app}
{\color{black}In this section, we describe the setting of non-stationary NE for GMFGs and establish the corresponding results of existence and uniqueness.}
\begin{definition}[NE for GMFGs]\label{nash2} 
In \eqref{mfg}, a player-population profile $(\pmb{\pi}^{\star},\pmb{\mathcal{L}}^{\star}):=(\{\pi_t^\star\}_{t=0}^{\infty},\{\mathcal{L}_t^\star\}_{t=0}^{\infty})$
  is called an NE if 
\begin{enumerate}
    \item (Single player side) Fix $\pmb{\mathcal{L}}^{\star}$, for any policy sequence $\pmb{\pi}:=\{\pi_t\}_{t=0}^{\infty}$ and initial state $s\in \mathcal{S}$, 
\begin{equation}
V\left(s,\pmb{\pi}^{\star},\pmb{\mathcal{L}}^{\star}\right)\geq V\left(s,\pmb{\pi},\pmb{\mathcal{L}}^{\star}\right).
\end{equation}
\item  (Population side) $\mathbb{P}_{s_t,a_t}= {\mathcal{L}_t^{\star}}$ for all $t\geq 0$, where $\{s_t,a_t\}_{t=0}^{\infty}$ is the dynamics under the policy sequence $\pmb{\pi}^{\star}$ starting from $s_0 \sim \mu_0^{\star}$, with $a_t\sim\pi_t^\star(s_t,{\color{black}\mu_t^{\star}})$, $s_{t+1}\sim P(\cdot|s_t,a_t,{\color{black}\mathcal{L}_t^\star})$, and $\mu_t^{\star}$ being the population state marginal of $\mathcal{L}_t^\star$.
\end{enumerate}
\end{definition}

\paragraph{Step A.}
Fix $\pmb{\mathcal{L}}:=\{\mathcal{L}_t\}_{t=0}^{\infty}$,   (\ref{mfg}) becomes the  classical optimization problem.  Indeed, with $\pmb{\mathcal{L}}$ fixed, the population state distribution sequence $\pmb{\mu}:=\{\mu_t\}_{t=0}^{\infty}$ is also fixed, hence the space of admissible policies is reduced to the single-player case. Solving (\ref{mfg}) is now reduced to finding a policy sequence
 $\pi_{t,\pmb{\mathcal{L}}}^\star\in \Pi:=\{\pi\,|\,\pi:\mathcal{S}\rightarrow\mathcal{P}(\mathcal{A})\}$
 over all admissible  $\pmb{\pi}_{\pmb{\mathcal{L}}}=\{\pi_{t,\pmb{\mathcal{L}}}\}_{t=0}^{\infty}$,
   to maximize 
 \[
\begin{array}{ll}
  V(s,\pmb{\pi}_{\pmb{\mathcal{L}}}, \pmb{\mathcal{L}}):= &\mathbb{E}\left[\sum\limits_{t=0}^{\infty}\gamma^tr(s_t,a_t,\mathcal{L}_t)|s_0=s\right],\\
\text{subject to}& s_{t+1}\sim P(s_t,a_t,\mathcal{L}_t),\quad a_t\sim\pi_{t,\pmb{\mathcal{L}}}(s_t).
\end{array}
\]
Notice that with $\pmb{\mathcal{L}}$ fixed, one can safely suppress the dependency on $\mu_t$ in the admissible policies. 
Moreover,  given this fixed $\pmb{\mathcal{L}}$ sequence and the solution $\pmb{\pi}_{\pmb{\mathcal{L}}}^\star:=\{\pi_{t,\pmb{\mathcal{L}}}^\star\}_{t=0}^{\infty}$, one can 
define a mapping from the fixed population distribution sequence $\pmb{\mathcal{L}}$ to an optimal randomized policy sequence. That is, 
$$\Gamma_1:\{\mathcal{P}(\mathcal{S}\times \mathcal{A})\}_{t=0}^{\infty}\rightarrow\{\Pi\}_{t=0}^{\infty}, $$
such that  $\pmb{\pi}_{\pmb{\mathcal{L}}}^\star=\Gamma_1(\pmb{\mathcal{L}})$.
Note that this $\pmb{\pi}_{\pmb{\mathcal{L}}}^\star$ sequence satisfies the single player side condition in Definition \ref{nash2} for the population state-action pair sequence $\pmb{\mathcal{L}}$. That is, 
$
V\left(s,\pmb{\pi}_{\pmb{\mathcal{L}}}^\star,{\color{black}\pmb{\mathcal{L}}}\right)\geq V\left(s,\pmb{\pi},{\color{black}\pmb{\mathcal{L}}}\right),$
for any policy sequence $\pmb{\pi} = \{\pi_t\}_{t=0}^{\infty}$ and any initial state $s\in\mathcal{S}$. 

Accordingly, a similar feedback regularity condition is needed in this step.
\begin{assumption}\label{policy_assumption}
There exists a constant $d_1\geq 0$,
 such that for any $\pmb{\mathcal{L}}, \pmb{\mathcal{L}}^{\prime} \in \{\mathcal{P}(\mathcal{S}\times\mathcal{A})\}_{t=0}^{\infty}$, 
\begin{equation}\label{Gamma1_lip}
D(\Gamma_1(\pmb{\mathcal{L}}),\Gamma_1( \pmb{\mathcal{L}}^{\prime})) \leq d_1\mathcal{W}_1(\pmb{\mathcal{L}}, \pmb{\mathcal{L}}^{\prime}),
\end{equation}
where 
\begin{equation}
\begin{split}
D(\pmb{\pi},\pmb{\pi}^{\prime})&:=\sup_{s\in\mathcal{S}}\mathcal{W}_1(\pmb{\pi}(s),\pmb{\pi}^{\prime}(s))=\sup_{s\in\mathcal{S}}\sup_{t\in\mathbb{N}}W_1(\pi_t(s),\pi_t'(s)),\\
\mathcal{W}_1(\pmb{\mathcal{L}}, \pmb{\mathcal{L}}^{\prime})&:=\sup_{t\in\mathbb{N}}W_1(\mathcal{L}_t,\mathcal{L}_t'),
\end{split}
\end{equation}
and $W_1$ is the $\ell_1$-Wasserstein distance between probability measures. 
\end{assumption}



 \paragraph{Step B.} Based on the analysis in Step A and $\pmb{\pi}_{\pmb{\mathcal{L}}}^\star=\{\pi_{t,\pmb{\mathcal{L}}}^\star\}_{t=0}^{\infty}$, update the initial sequence $\pmb{\mathcal{L}}$ to $ \pmb{\mathcal{L}}^{\prime}$ following the controlled dynamics $P(\cdot|s_t, a_t,\mathcal{L}_t)$.

Accordingly, for any admissible policy sequence $\pmb{\pi} \in \{\Pi\}_{t=0}^{\infty}$ and a joint population state-action pair sequence $\pmb{\mathcal{L}}\in \{\mathcal{P}(\mathcal{S}\times \mathcal{A})\}_{t=0}^{\infty}$, define a mapping $\Gamma_2:\{\Pi\}_{t=0}^{\infty}\times \{\mathcal{P}(\mathcal{S}\times\mathcal{A})\}_{t=0}^{\infty}\rightarrow \{\mathcal{P}(\mathcal{S}\times\mathcal{A})\}_{t=0}^{\infty}$ as follows: 
\begin{eqnarray}
\Gamma_2(\pmb{\pi},\pmb{\mathcal{L}}):= \pmb{\hat{\mathcal{L}}} = \{\mathbb{P}_{s_t,a_t}\}_{t=0}^{\infty},
\end{eqnarray}
where $s_{t+1}\sim \mu_t P(\cdot|\cdot,a_t,\mathcal{L}_t)$,  $a_t\sim \pi_t(s_t)$, $s_0 \sim \mu_0$, and $\mu_t$ is the population state marginal of $\mathcal{L}_t$.

One also needs a similar assumption in this step.
\begin{assumption}\label{population_assumption}
There exist constants $d_2,~d_3\geq 0$, such that for any admissible policy sequences $\pmb{\pi},\pmb{\pi}^1,\pmb{\pi}^2$ and joint distribution sequences $\pmb{\mathcal{L}}, \pmb{\mathcal{L}}^{1}, \pmb{\mathcal{L}}^{2}$, 
\begin{equation}\label{Gamma2_lip1}
\mathcal{W}_1(\Gamma_2(\pmb{\pi}^1,\pmb{\mathcal{L}}),\Gamma_2(\pmb{\pi}^2,\pmb{\mathcal{L}})) \leq d_2 D(\pmb{\pi}^1,\pmb{\pi}^2), 
\end{equation}
\begin{equation}\label{Gamma2_lip2}
\mathcal{W}_1(\Gamma_2(\pmb{\pi},\pmb{\mathcal{L}}^1{\color{black})},\Gamma_2(\pmb{\pi},\pmb{\mathcal{L}}^2)) \leq d_3 \mathcal{W}_1(\pmb{\mathcal{L}}^1,\pmb{\mathcal{L}}^2).
\end{equation}
\end{assumption}
Similarly, Assumption \ref{population_assumption} can be reduced to Lipschitz continuity  and boundedness of the transition dynamics $P$ under certain conditions. 


\paragraph{Step C.} Repeat Step A and Step B until $ \pmb{\mathcal{L}}^{\prime}$ matches $\pmb{\mathcal{L}}$.

This step is to take care of the population side condition. To ensure the convergence of
 the combined step A and step B,  it suffices if  $\Gamma:\{\mathcal{P}(\mathcal{S}\times\mathcal{A})\}_{t=0}^{\infty}\rightarrow \{\mathcal{P}(\mathcal{S}\times\mathcal{A})\}_{t=0}^{\infty}$ is a contractive mapping under the $\mathcal{W}_1$ distance, with $\Gamma(\pmb{\mathcal{L}}):=\Gamma_2(\Gamma_1(\pmb{\mathcal{L}}), \pmb{\mathcal{L}})$.  Then by the Banach fixed point theorem {\color{black}and the completeness of the related metric spaces}, there exists  a unique NE to the GMFG. 

 In summary, we have                                                                                                                    
 \begin{theorem}[Existence and Uniqueness of GMFG solution] \label{thm1} Given Assumptions \ref{policy_assumption} and \ref{population_assumption}, and assuming that $d_1d_2+d_3< 1$, there exists a unique NE  to \eqref{mfg}.
 
\end{theorem}

The proof of Theorem \ref{thm1} can be established by modifying  appropriately the fixed-point approach for the {\color{black}stationary} GMFG in Theorem \ref{thm1_stat}.



\newpage
\bibliographystyle{informs2014} 
\bibliography{mfg_rl}

\newpage
%
%
%
 \begin{APPENDICES}
\section{Distance metrics and completeness} \label{appendixA}
This section reviews some basic properties of the Wasserstein distance. 
It then proves that the metrics defined in the main text are indeed distance functions and define complete metric spaces. 

\paragraph{$\ell_1$-Wasserstein distance and dual representation.} 
The $\ell_1$ Wasserstein distance over $\mathcal{P}(\mathcal{X})$ for $\mathcal{X}\subseteq\mathbb{R}^k$ is defined as 
\begin{equation}\label{W1}
W_1(\nu,\nu'):=\inf_{M\in\mathcal{M}(\nu,\nu')}\int_{\mathcal{X}\times\mathcal{X}}\|x-y\|_2\text{d}M(x,y).
\end{equation}
where $\mathcal{M}(\nu,\nu')$ is the set of all measures (couplings) on $\mathcal{X}\times\mathcal{X}$, with marginals $\nu$ and $\nu'$ on the two components, respectively.

The Kantorovich duality theorem enables the following equivalent dual representation of $W_1$:
\begin{equation}\label{W1_dual}
W_1(\nu,\nu')=\sup_{\|f\|_L\leq 1}\left|\int_{\mathcal{X}}fd\nu-\int_{\mathcal{X}}fd\nu'\right|,
\end{equation}
where the supremum is taken over all $1$-Lipschitz functions $f$, \textit{i.e.}, $f$ satisfying $|f(x)-f(y)|\leq \|x-y\|_2$ for all $x,y\in\mathcal{X}$.

The Wasserstein distance $W_1$ can also be related to the total variation distance via the following inequalities \citet{metrics_prob}: 
\begin{equation}\label{W1_tv}
d_{\min}(\mathcal{X})d_{TV}(\nu,\nu')\leq W_1(\nu,\nu')\leq \text{diam}(\mathcal{X})d_{TV}(\nu,\nu'),
\end{equation}
where $d_{\min}(\mathcal{X})=\min_{x\neq y\in\mathcal{X}}\|x-y\|_2$, which is guaranteed to be positive when $\mathcal{X}$ is finite.

When $\mathcal{S}$ and $\mathcal{A}$ are compact,  for any compact subset $\mathcal{X}\subseteq\mathbb{R}^k$, and for any $\nu,\nu'\in\mathcal{P}(\mathcal{X})$, $W_1(\nu,\nu')\leq \text{diam}(\mathcal{X})d_{TV}(\nu,\nu')\leq \text{diam}(\mathcal{X})<\infty$, where $\text{diam}(\mathcal{X})=\sup_{x,y\in\mathcal{X}}\|x-y\|_2$ and $d_{TV}$ is the total variation distance.
Moreover,  one can verify  
\begin{lemma}\label{metrics}
Both $D$ and $\mathcal{W}_1$ are distance functions, and they are finite for any input distribution pairs. In addition, both $(\{\Pi\}_{t=0}^{\infty},D)$ and $(\{\mathcal{P}(\mathcal{S}\times\mathcal{A})\}_{t=0}^{\infty}, \mathcal{W}_1)$ are \textit{complete metric spaces}.
\end{lemma}
These facts enable the usage of Banach fixed-point mapping theorem for the proof of existence and uniqueness (Theorems \ref{thm1} and \ref{thm1_stat}).
\begin{proof}{[Proof of Lemma \ref{metrics}]}
It is known that for any compact set $\mathcal{X}\subseteq\mathbb{R}^k$, $(\mathcal{P}(\mathcal{X}), W_1)$ defines a complete metric space \citet{wass_comp}. Since $W_1(\nu,\nu')\leq \text{diam}(\mathcal{X})$ is uniformly bounded for any $\nu,~\nu'\in\mathcal{P}(\mathcal{X})$, we know that $\mathcal{W}_1(\pmb{\mathcal{L}},\pmb{\mathcal{L}}')\leq \text{diam}(\mathcal{X})$ and $D(\pmb{\pi},\pmb{\pi'})\leq \text{diam}(\mathcal{X})$ as well, so they are both finite for any input distribution pairs. It is clear that they are distance functions based on the fact that $W_1$ is a distance function.

Finally, we show the completeness of the two metric spaces $(\{\Pi\}_{t=0}^{\infty},D)$ and $(\{\mathcal{P}(\mathcal{S}\times\mathcal{A})\}_{t=0}^{\infty}, \mathcal{W}_1)$. Take $(\{\Pi\}_{t=0}^{\infty},D)$ for example. Suppose that $\pmb{\pi}^k$ is a Cauchy sequence in $(\{\Pi\}_{t=0}^{\infty},D)$. Then for any $\epsilon>0$, there exists a positive integer $N$, such that for any $m,~n\geq N$, 
\begin{equation}
D(\pmb{\pi}^n,\pmb{\pi}^{m})\leq \epsilon\Longrightarrow W_1(\pi_t^n(s),\pi_t^{m}(s))\leq \epsilon\text{ for any $s\in\mathcal{S}$, $t\in\mathbb{N}$},
\end{equation}
which implies that $\pi_t^k(s)$ forms a Cauchy sequence in $(\mathcal{P}(\mathcal{A}),W_1)$, and hence by the completeness of $(\mathcal{P}(\mathcal{A}),W_1)$, $\pi_t^k(s)$ converges to some $\pi_t(s)\in\mathcal{P}(\mathcal{A})$. As a result,  $\pmb{\pi}^n\rightarrow\pmb{\pi}\in\{\Pi\}_{t=0}^{\infty}$ under metric $D$, which shows that $(\{\Pi\}_{t=0}^{\infty},D)$ is complete. 

The completeness of $(\{\mathcal{P}(\mathcal{S}\times\mathcal{A})\}_{t=0}^{\infty}, \mathcal{W}_1)$ can be proved similarly.
\qed
\end{proof}

The same argument for Lemma \ref{metrics} shows that   both $D$ and $W_1$ are distance functions and are finite for any input distribution pairs, with both $(\Pi, D)$ and $(\mathcal{P}(\mathcal{S}\times\mathcal{A}),W_1)$  again complete metric spaces. 

{\color{black}
\section{Bounds for GMF-V-Q using asynchronous Q-learning}\label{asyn-q}
In the main text, we have shown the results by using synchronous Q-learning algorithm. Here for the completeness, we also show the corresponding results for asynchronous Q-learning algorithm.

For asynchronous Q-learning algorithm, at each step $l$ with the state $s$ and an action $a$, the system reaches state $s'$ according to the
controlled dynamics and the Q-function approximation $Q_l$ is updated according to
\begin{equation}\label{asyn-q-update}
\hat{Q}^{l+1}(s,a)= (1-\beta_l(s,a))\hat{Q}^l(s,a)+\beta_l(s,a)\left[r(s,a)+\gamma\max_{\bar{a}}\hat{Q}^l(s',\bar{a})\right],
\end{equation}
where $\hat{Q}^0(s,a)=C$ for some constant $C\in\mathbb{R}$ for any $s\in\mathcal{S}$ and $a\in\mathcal{A}$, and the step size $\beta_l(s,a)$ can  be chosen as (\citet{Q-rate})
\begin{equation}\label{asyn-q-step-size}
\beta_l(s,a)=
\begin{cases}
|\#(s,a,l)+1|^{-h}, & (s,a)=(s_l,a_l),\\
0, & \text{otherwise}.
\end{cases}
\end{equation}
with $h\in(1/2,1)$. Here $\#(s,a,l)$ is the number of times up to time $l$ that one visits the state-action pair $(s,a)$. The algorithm then proceeds to choose action $a'$ based on $\hat{Q}^{l+1}$ with appropriate exploration strategies, including the $\epsilon$-greedy strategy.

\begin{lemma}[\citet{Q-rate}: sample complexity of asynchronous Q-learning]\label{asyn-Q-finite-bd} For an MDP, say $\mathcal{M}=(\mathcal{S},\mathcal{A},P,r,\gamma)$, suppose that the Q-learning algorithm takes step-sizes \eqref{q-step-size}.
Also suppose that the covering time of the state-action pairs is bounded by $L$ with probability at least $1-p$ for some $p\in(0,1)$. Then $\|{\color{black}\hat{Q}^{T_{\mathcal{M}}(\delta,\epsilon)}}-Q_{\mathcal{M}}^\star\|_{\infty}\leq \epsilon$  with probability at least $1-2\delta$. Here {\color{black}$\hat{Q}^T$} is the $T$-th update in the  Q-learning updates \eqref{q-update}, $Q_{\mathcal{M}}^\star$ is the (optimal) Q-function, and 
\[
\begin{split}
T_{\mathcal{M}}(\delta,&\epsilon)=\Omega\left(\left(\dfrac{L\log_p(\delta)}{\beta}\log\dfrac{V_{\max}}{\epsilon}\right)^{\frac{1}{1-h}}+\left(\dfrac{\left(L\log_p(\delta)\right)^{1+3h}V_{\max}^2\log\left(\frac{|\mathcal{S}||\mathcal{A}|V_{\max}}{\delta\beta\epsilon}\right)}{\beta^2\epsilon^2}\right)^{\frac{1}{h}}\right),
\end{split}
\]

where $\beta=(1-\gamma)/2$, $V_{\max}=R_{\max}/(1-\gamma)$, and $R_{\max}$ {\color{black} is such that a.s.
$0\leq {\color{black}r(s,a)}\leq R_{\max}$. 
}
\end{lemma} 
Here the covering time $L$ of a state-action pair sequence is defined to be the number of steps needed to visit all state-action pairs starting from any arbitrary state-action pair. 
Also notice that the $l_{\infty}$ norm above is defined in an element-wise sense, \textit{i.e.}, for $M\in\mathbb{R}^{|\mathcal{S}|\times|\mathcal{A}|}$, we have $\|M\|_\infty=\max_{s\in\mathcal{S},a\in\mathcal{A}}|M(s,a)|$.

\begin{corollary}[Value-based guarantee of asynchronous Q-learning algorithm]\label{guarantee_q}
The asynchronous Q-learning algorithm with appropriate choices of step-sizes (\textit{cf}. \eqref{q-step-size}) satisfies the following value-based guarantee, where $C_{\mathcal{M}}^{(i)}(i=1,2,3)$ are constants depending on  $|\mathcal{S}|,|\mathcal{A}|,V_{\max},\beta$ and $h$, and we have:
\begin{eqnarray*}
\alpha_2^{(1)} = \alpha_4^{(1)} = \frac{1}{1-h},\,\,\alpha_1^{(1)} = \alpha_3^{(2)}=0,\\
\alpha_1^{(2)}=\frac{2}{h}, \,\,\alpha_4^{(2)}=\frac{2+3h}{h},\,\,\alpha_j^{(2)}=0 \,\,\mbox{for} \,\,j=2,3,\\
\alpha_1^{(3)}=\frac{2}{h}, \,\,\alpha_2^{(3)}=\frac{1}{h}, \,\,\alpha_4^{(3)}=\frac{1+3h}{h}, \,\,\alpha_3^{(3)}=0.
\end{eqnarray*}
In addition, assume the same assumptions as Theorem \ref{thm1_stat}, 
then for Algorithm \ref{AIQL_MFG} with asynchronous Q-learning method, with probability at least $1-2\delta$, $W_1(\mathcal{L}_{K_{\epsilon,\eta}},\mathcal{L}^\star)\leq C_0\epsilon$, where $K_{\epsilon,\eta}$ is defined as in Theorem \ref{conv_AIQL}.
And the total number of samples $T=\sum_{k=0}^{K_{\epsilon,\eta} -1}T_{\mathcal{M}_{\mathcal{L}_k}}(\delta_k,\epsilon_k)$ is bounded by
\[
T\leq O\left(K_{\epsilon,\eta}^{\frac{4}{h}+1}\left(\log\frac{K_{\epsilon,\eta}}{\delta}\right)^{3+\frac{2}{h}}+\left(\log \frac{K_{\epsilon,\eta}}{\delta}\right)^{\frac{2}{1-h}}\right).
\]
\end{corollary}

}

\section{Weak simulator}\label{sec:weak}
In this section, we state the counterpart of Theorems \ref{conv_AIQL} and \ref{conv_AIQL-II} for Algorithms \ref{GMF-VW} and \ref{GMF-PW}, respectively. Notice that here the major difference is the additional $O(1/\sqrt{N})$ term.

{\color{black}
We first (re)state the relation between $\textbf{Emp}_{N}$ (which serves as a $1/N$-net) and action gaps:

\textit{For any positive integer $N$,
there exist a positive constant $\phi_N>0$, with the property that 
$\max_{a'\in\mathcal{A}}Q^\star_{\mathcal{L}}(s,a')-Q^\star_{\mathcal{L}}(s,a)\geq \phi_N$
 for any $\mathcal{L}\in\textbf{Emp}_N$,  $s\in\mathcal{S}$, and any $a\notin \text{argmax}_{a\in\mathcal{A}}Q^\star_{\mathcal{L}}(s,a)$. 
}

Now we are ready to state the convergence results.
}

\begin{theorem}[Convergence and complexity of GMF-VW]\label{conv_AIQL-VW}
Assume the same assumptions as Theorem \ref{thm1_stat}. Suppose that \textit{Alg} has a value-based guarantee with parameters 
\[
\{C_{\mathcal{M}}^{(i)},\alpha_1^{(i)},\alpha_2^{(i)},\alpha_3^{(i)},\alpha_4^{(i)}\}_{i=1}^m.
\] 

For any $\epsilon,~\delta>0$, set $\delta_k=\delta/ K_{\epsilon,\eta}$, $\epsilon_k=(k+1)^{-(1+\eta)}$ for some $\eta\in(0,1]$ $(k=0,\dots,K_{\epsilon,\eta}-1)$, and  $c'=c=\frac{\log(1/\epsilon)}{{\color{black}\phi_N}}$.\footnote{\color{black}Here we actually only need $c'=\Omega(\frac{\log(1/\epsilon)}{\phi_N})$ and $c=O(\frac{\log(1/\epsilon)}{\phi_N})$, and the corresponding result will differ only in some absolute constants.} Then with probability at least $1-4\delta$, $$W_1(\mathcal{L}_{K_{\epsilon,\eta}},\mathcal{L}^\star)\leq C\epsilon+\frac{\text{diam}(\mathcal{S})\text{diam}(\mathcal{A})|\mathcal{S}||\mathcal{A}|}{2(1-d)}\sqrt{\dfrac{1}{2N}\log(|\mathcal{S}||\mathcal{A}|K_{\epsilon,\eta}/\delta)}.$$   Here $K_{\epsilon,\eta}:=\left\lceil 2\max\left\{(\eta\epsilon/c)^{-1/\eta},\log_d(\epsilon/\max\{\text{diam}(\mathcal{S})\text{diam}(\mathcal{A}),c\})+1\right\}\right\rceil$ is the number of outer iterations, and the constant $C$ is independent of $\delta$, $\epsilon$ and $\eta$. 
 
    Moreover,  the total number of {\color{black}samples} $T=\sum_{k=0}^{K_{\epsilon,\eta} -1}T_{\mathcal{M}_{\mathcal{L}_k}}(\delta_k,\epsilon_k)$ is bounded by
\begin{equation}\label{Tbound-I-VW}
T\leq \sum_{i=1}^m\dfrac{2^{\alpha_2^{(i)}}}{2\alpha_1^{(i)}+1}C_{\mathcal{M}}^{(i)}K_{\epsilon,\eta}^{2\alpha_1^{(i)}+1}(K_{\epsilon,\eta}/\delta)^{\alpha_3^{(i)}}\left(\log(K_{\epsilon,\eta}/\delta)\right)^{\alpha_2^{(i)}+\alpha_4^{(i)}}.
\end{equation}
\end{theorem}


\begin{theorem}[Convergence and complexity of GMF-PW]\label{conv_AIQL-II-PW}
Assume the same assumptions as in Theorem \ref{thm1_stat}. Suppose that \textit{Alg} has a policy-based guarantee with parameters 
\[
\{C_{\mathcal{M}}^{(i)},\alpha_1^{(i)},\alpha_2^{(i)},\alpha_3^{(i)},\alpha_4^{(i)}\}_{i=1}^m.
\] 

Then for any $\epsilon,~\delta>0$, set $\delta_k=\delta/ K_{\epsilon,\eta}$, $\epsilon_k=(k+1)^{-(1+\eta)}$ for some $\eta\in(0,1]$ $(k=0,\dots,K_{\epsilon,\eta}-1)$, and  $c'=c=\frac{\log(1/\epsilon)}{{\color{black}\phi_N}}$,\footnote{\color{black}Here again we actually only need $c'=\Omega(\frac{\log(1/\epsilon)}{\phi_N})$ and $c=O(\frac{\log(1/\epsilon)}{\phi_N})$, and the corresponding result will differ only in some absolute constants.} with probability at least $1-4\delta$, $$W_1(\mathcal{L}_{K_{\epsilon,\eta}},\mathcal{L}^\star)\leq C\epsilon+\frac{\text{diam}(\mathcal{S})\text{diam}(\mathcal{A})|\mathcal{S}||\mathcal{A}|}{2(1-d)}\sqrt{\dfrac{1}{2N}\log(|\mathcal{S}||\mathcal{A}|K_{\epsilon,\eta}/\delta)}.$$   Here $K_{\epsilon,\eta}:=\left\lceil 2\max\left\{(\eta\epsilon/c)^{-1/\eta},\log_d(\epsilon/\max\{\text{diam}(\mathcal{S})\text{diam}(\mathcal{A}),c\})+1\right\}\right\rceil$ is the number of outer iterations, and the constant $C$ is independent of $\delta$, $\epsilon$ and $\eta$. 

    Moreover,  the total number of {\color{black}samples} $T=\sum_{k=0}^{K_{\epsilon,\eta} -1}T_{\mathcal{M}_{\mathcal{L}_k}}(\delta_k,\epsilon_k)$ is bounded by
\begin{equation}\label{Tbound-II-PW}
T\leq \sum_{i=1}^{m+1}\dfrac{2^{\alpha_2^{(i)}}}{2\alpha_1^{(i)}+1}\tilde{C}_{\mathcal{M}}^{(i)}K_{\epsilon,\eta}^{2\alpha_1^{(i)}+1}(K_{\epsilon,\eta}/\delta)^{\alpha_3^{(i)}}\left(\log(K_{\epsilon,\eta}/\delta)\right)^{\alpha_2^{(i)}+\alpha_4^{(i)}},
\end{equation}
 where the parameters $\{\tilde{C}_{\mathcal{M}}^{(i)},\alpha_1^{(i)},\alpha_2^{(i)},\alpha_3^{(i)},\alpha_4^{(i)}\}_{i=1}^{m+1}$ are defined in Lemma \ref{conv_2to1}.
\end{theorem}

The key to the proof of Theorems \ref{conv_AIQL-VW} and \ref{conv_AIQL-II-PW} is the following lemma, which follows from the Hoeffding inequality.
\begin{lemma}\label{Hoeffding_weak_simul}
The expectation $\mathbb{E}\left[{\mathcal{L}}_{k+1}(s',a')\right]=\pi_k(s',a')\sum_{s\in\mathcal{S}}\sum_{a\in\mathcal{A}}\mu_k(s)P(s'|s,a,\mathcal{L}_k)\pi_k(s,a)=\Gamma_2(\pi_k,\mathcal{L}_k)$. In addition, we have that for any $t> 0$, $s'\in\mathcal{S}$ and $a'\in\mathcal{A}$,  
\begin{equation}\label{conc_Lk}
\mathbb{P}\left(\left|{\mathcal{L}}_{k+1}(s',a')-\mathbb{E}\left[{\mathcal{L}}_{k+1}(s',a')\right]\right|\geq t\right)\leq 2\exp\left(-2Nt^2\right).
\end{equation}
\end{lemma}
The above lemma essentially states that the iterates $\mathcal{L}_{k+1}$ of Algorithms \ref{GMF-VW} and \ref{GMF-PW}  are very close to the  ``$\tilde{\mathcal{L}}_{k+1}$'' obtained from the (strong) simulator $\mathcal{G}(s,\pi_k,\mathcal{L}_k)$ with $s\sim \mu_k$ following line 5 in Algorithm \ref{GMF-V} and line 6 in Algorithm \ref{GMF-P}. This bridges the gap between the weak and the strong simulators. In particular, by noticing that 
\[
\begin{split}
 W_1({\mathcal{L}}_{k+1},\mathcal{L}^\star)&\leq W_1(\mathcal{L}_{k+1},\Gamma_2(\pi_k,\mathcal{L}_k))+W_1(\Gamma_2(\pi_k,\mathcal{L}_k),\Gamma_2(\Gamma_1(\mathcal{L}^\star),\mathcal{L}^\star))\\
 &=W_1(\mathcal{L}_{k+1},\mathbb{E}\left[{\mathcal{L}}_{k+1}\right])+W_1(\Gamma_2(\pi_k,\mathcal{L}_k),\Gamma_2(\Gamma_1(\mathcal{L}^\star),\mathcal{L}^\star))\\
 &\leq \frac{\text{diam}(\mathcal{S})\text{diam}(\mathcal{A})|\mathcal{S}||\mathcal{A}|}{2}\left\|\mathcal{L}_{k+1}-\mathbb{E}\left[{\mathcal{L}}_{k+1}\right]\right\|_{\infty}+W_1(\Gamma_2(\pi_k,\mathcal{L}_k),\Gamma_2(\Gamma_1(\mathcal{L}^\star),\mathcal{L}^\star)),
 \end{split}
\]
one can bound the first term with high probability via \eqref{conc_Lk}. The second term is then bounded in exactly the same way as the proofs for Theorems \ref{conv_AIQL} and \ref{conv_AIQL-II}, and hence we omit the details. 
 \end{APPENDICES}





\end{document}